\documentclass{article}

\usepackage[final, nonatbib]{neurips_2021}

\usepackage[utf8]{inputenc} 
\usepackage[T1]{fontenc}    
\usepackage{url}            
\usepackage{booktabs}       
\usepackage{nicefrac}       
\usepackage{microtype}      
\usepackage{algorithm}
\usepackage{algorithmic}
\usepackage{amsmath,amssymb,amsfonts,amstext,amsthm,mathrsfs}
\usepackage{multirow}
\usepackage{graphicx}
\usepackage{threeparttable}
\usepackage[table,dvipsnames]{xcolor}
\usepackage{mathtools}
\usepackage{cleveref}
\usepackage{subfigure}

\allowdisplaybreaks[4]

\DeclareMathOperator*{\argmin}{arg\,min}

\usepackage{xcolor,colortbl}
\definecolor{DarkBlue}{RGB}{22,54,93}

\newtheorem{theorem}{Theorem}
\newtheorem{coro}{Corollary}
\newtheorem{lemma}{Lemma}
\newtheorem{definition}{Definition}
\newtheorem{proposition}{Proposition}
\newtheorem{Assumption}{Assumption}
\newtheorem{property}{Property}

\title{Provably Faster Algorithms for Bilevel Optimization}

\author{Junjie Yang \\
	Department of ECE\\
	The Ohio State University \\
	\texttt{yang.4972@osu.edu}\\ 
\And
Kaiyi Ji\\
Department of EECS\\
University of Michigan\\
\texttt{kaiyiji@umich.edu}\\
\And
Yingbin Liang\\
Department of ECE\\
The Ohio State University \\
\texttt{liang.889@osu.edu}\\ 
}

\begin{document}

\maketitle
\begin{abstract}
    Bilevel optimization has been widely applied in many important machine learning applications such as hyperparameter optimization and meta-learning. Recently, several momentum-based algorithms have been proposed to solve bilevel optimization problems faster. However, those momentum-based algorithms do not achieve provably better computational complexity than $\mathcal{\widetilde O}(\epsilon^{-2})$ of the SGD-based algorithm. In this paper, we propose two new algorithms for bilevel optimization, where the first algorithm adopts momentum-based recursive iterations, and the second algorithm adopts recursive gradient estimations in nested loops to decrease the variance. We show that both algorithms achieve the complexity of $\mathcal{\widetilde O}(\epsilon^{-1.5})$, which outperforms all existing algorithms by the order of magnitude. Our experiments validate our theoretical results and demonstrate the superior empirical performance of our algorithms in hyperparameter applications.

\end{abstract}

\section{Introduction}
Bilevel optimization has become a timely and important topic recently due to its great effectiveness in a wide range of applications including hyperparameter optimization \cite{franceschi2018bilevel,feurer2019hyperparameter}, meta-learning \cite{rajeswaran2019meta,ji2020convergence,bertinetto2018meta}, reinforcement learning \cite{hong2020two,konda2000actor}. Bilevel optimization can be generally formulated as the following minimization problem:    
\begin{align}
\label{eq:main}
\underset{x\in \mathbb{R}^p}{\min} \Phi(x) := f(x,y^{\ast}(x))\quad \text{ s.t. }\; y^{\ast}(x) = \underset{y \in \mathbb{R}^q}{\argmin} \; g(x,y).
 \end{align}
Since the outer function $\Phi(x):=f(x,y^*(x))$ depends on the variable $x$ also via the optimizer $y^*(x)$ of the inner-loop function $g(x,y)$, the algorithm design for bilevel optimization is much more complicated and challenging than minimization and minimax optimization. For example, if the gradient-based approach is applied, then the gradient of the outer-loop function (also called {\em hypergradient}) will necessarily involve Jacobian and Hessian matrices of the inner-loop function $g(x,y)$, which require more careful design to avoid high computational complexity.

This paper focuses on the nonconvex-strongly-convex setting, where the outer function $f(x,y^*(x))$ is nonconvex with respect to (w.r.t.)~$x$ and the inner function $g(x,y)$ is strongly convex w.r.t.~$y$ for any $x$. Such a case often occurs in practical applications. For example, in hyperparameter optimization~\cite{franceschi2018bilevel}, $f(x,y^*(x))$ is often nonconvex with $x$ representing neural network hyperparameters, but the inner function $g(x,\cdot)$ can be strongly convex w.r.t.~$y$ by including a strongly-convex regularizer on $y$. In few-shot  meta-learning~\cite{bertinetto2018meta}, the inner function $g(x,\cdot)$ often takes a quadratic form together with a strongly-convex regularizer. 
To efficiently solve the deterministic problem in \cref{eq:main}, various bilevel optimization algorithms have been proposed, which include two popular classes of deterministic gradient-based methods respectively based on  
approximate implicit differentiation (AID)~\cite{pedregosa2016hyperparameter, gould2016differentiating,ghadimi2018approximation} and iterative differentiation (ITD)~\cite{maclaurin2015gradient, franceschi2017forward, franceschi2018bilevel}. 

Recently, stochastic bilevel opitimizers~\cite{ghadimi2018approximation,ji2021bilevel} have been proposed, in order to achieve better efficiency than deterministic methods for large-scale scenarios where the data size is large or vast fresh data needs to be sampled as the algorithm runs. 

 
 In particular, such a class of problems adopt functions by:
 \begin{align}
 \label{eq:expobj}
      \Phi(x):=f(x,y^{\ast}(x)):=\mathbb{E}_{\xi}[F(x,y^{\ast}(x);\xi)],  \nonumber
      \quad g(x,y) := \mathbb{E}_{\zeta}[G(x,y;\zeta)]
 \end{align}
where the outer and inner functions take the expected values w.r.t.~samples $\xi$ and $\zeta$, respectively. 

Along this direction, \cite{ji2021bilevel} proposed a stochastic gradient descent (SGD) type optimizer~(stocBiO), and showed that stocBiO attains a computational complexity of $\mathcal{\widetilde O}(\epsilon^{-2})$ in order to reach an $\epsilon$-accurate stationary point. More recently, several studies~\cite{chen2021single,guo2021stochastic,khanduri2021momentum} have tried to accelerate SGD-type bilevel optimizers via momentum-based techniques, e.g., by introducing a momentum (historical information) term into the gradient estimation. All of these optimizers follow a {\bf single-loop} design, i.e., updating $x$ and $y$ simultaneously. Specifically, \cite{khanduri2021momentum} proposed an algorithm MSTSA by updating $x$ via a momentum-based recursive technique introduced by~\cite{cutkosky2019momentum,tran2019hybrid}. \cite{guo2021stochastic} proposed an optimizer SEMA similarly to MSTSA but using the momentum recursive technique for updating both $x$ and $y$. \cite{chen2021single} proposed an algorithm STABLE, which applies the momentum strategy for updating the Hessian matrix, but the algorithm involves expensive Hessian inverse computation rather than hypergradient approximation loop. However, as shown in \Cref{tb:order}, SEMA, MSTSA and STABLE achieve the same complexity order of $\mathcal{\widetilde O}(\epsilon^{-2})$ as the SGD-type stocBiO algorithm, where the momentum technique in these algorithms does not exhibit the theoretical advantage. Such a comparison is not consistent with those in minimization~\cite{cutkosky2019momentum} and minimax optimization~\cite{huang2020accelerated}, where the single-loop momentum-based recursive technique achieves provable performance improvements over SGD-type methods. 
This motivates the following natural but important question:
\begin{list}{$\bullet$}{\topsep=0.01in \leftmargin=0.2in \rightmargin=0.1in \itemsep =0.01in}
\item Can we design a faster single-loop momentum-based recursive bilevel optimizer, which achieves order-wisely lower computational complexity than SGD-type stocBiO (and all other momentum-based algorithms), and is also easy to implement with efficient matrix-vector products?    
\end{list}

Although the existing theoretical efforts on accelerating bilevel optimization algorithms have been exclusively focused on single-loop design\footnote{In the literature of bilevel optimization, although many hypergradient-based algorithms include an iteration loop of Hessian inverse estimation, such a loop is typically not counted when these algorithms are classified by the number of loops. This paper follows such a convention to be consistent with the existing literature. Namely, the single- and double-loop algorithms mentioned here can include an additional loop of Hessian inverse estimation in the hypergradient approximation.}, empirical studies in~\cite{ji2021bilevel} suggested that {\bf double-loop} bilevel algorithms such as  BSA~\cite{ghadimi2018approximation} and stocBiO~\cite{ji2021bilevel} achieve much better performances than {\bf single-loop} algorithms such as TTSA~\cite{hong2020two}. A good candidate suitable for accelerating double-loop algorithms can be the popular variance reduction method, such as SVRG~\cite{johnson2013accelerating}, SARAH~\cite{nguyen2017sarah} and SPIDER~\cite{fang2018spider}, which typically yield provably lower complexity. The basic idea is to construct low-variance gradient estimators using periodic high-accurate large-batch gradient evaluations. So far, there has not been any study on using variance reduction to accelerate double-loop bilevel optimization algorithms. This motivates the second question that we address in this paper:

\begin{list}{$\bullet$}{\topsep=0.01in \leftmargin=0.2in \rightmargin=0.1in \itemsep =0.01in}
\item Can we develop a double-loop variance-reduced  bilevel optimizer with improved computational complexity over SGD-type stocBiO (and all other existing algorithms)? If so, when such a {\bf double-loop} algorithm holds advantage over the {\bf single-loop} algorithms in bilevel optimization?
\end{list}

 \subsection{Main Contributions}
 This paper proposes two algorithms for bilevel optimization, both outperforming all existing algorithms in terms of complexity order.
 
 We first propose a single-loop momentum-based recursive bilevel optimizer (MRBO). MRBO updates variables $x$ and $y$ simultaneously, and uses the momentum recursive technique for constructing low-variance {\bf mini-batch} estimators for both the gradient $\nabla g(x,\cdot)$ and the hypergradient $\nabla\Phi(\cdot)$; in contrast to previous momentum-based  algorithms that accelerate only one gradient or neither. Further, MRBO is easy to implement, and allows efficient computations of Jacobian- and Hessian-vector products via automatic differentiation. Theoretically, we show that MRBO achieves a computational complexity (w.r.t.~computations of gradient, Jacobian- and Hessian-vector product) of $\mathcal{\widetilde O}(\epsilon^{-1.5})$, which outperforms all existing algorithms by an order of $\epsilon^{-0.5}$. Technically, our analysis needs to first characterize the estimation property for the momentum-based recursive estimator for the {\bf Hessian-vector type} hypergradient and then uses such a property to further bound the per-iteration error due to momentum updates for both inner and outer loops.
 
We then propose a double-loop variance-reduced bilevel optimizer (VRBO), which is the first algorithm that adopts the recursive variance reduction for bilevel optimization. In VRBO, each inner loop constructs a variance-reduced gradient (w.r.t.~$y$) and {\bf hypergradient} (w.r.t.~$x$) estimators through the use of large-batch gradient estimations computed periodically at each outer loop. Similarly to MRBO, VRBO involves the computations of Jacobian- and Hessian-vector products rather than Hessians or Hessian inverse. Theoretically, we show that VRBO achieves the same near-optimal complexity of $\mathcal{\widetilde O}(\epsilon^{-1.5})$ as MRBO and outperforms all existing algorithms. Technically, differently from the use of variance reduction in minimization and minimax optimization, our analysis for VRBO needs to characterize the variance reduction property for the {\bf Hessian-vector type} of hypergradient estimators, which only involves Hessian vector computation rather than Hessian. Such estimator introduces additional errors to handle in the telescoping and convergence analysis.

Our experiments\footnote{Our codes are available online at https://github.com/JunjieYang97/MRVRBO} show that VRBO achieves the highest accuracy among all comparison algorithms, and MRBO converges fastest among its same type of single-loop momentum-based algorithms. 
In particular, we find that our double-loop VRBO algorithm converges much faster than other singlr-loop algorithms including our MRBO, which is in contrast to the existing efforts exclusively on accelerating the single-loop algorithms \cite{chen2021single, guo2021stochastic, khanduri2021momentum}. Such a result also differs from those phenomenons observed in minimization and minimax optimization, where single-loop algorithms often outperform double-loop algorithms. 

\begin{table*}[!t]
 \centering
 \caption{Comparison of stochastic algorithms for bilevel  optimization.}
 \small
 \begin{threeparttable}
  \begin{tabular}{|c|c|c|c|c|c|}
   \hline
Algorithm & Gc($F,\epsilon$) & Gc($G,\epsilon$) & JV($G,\epsilon$) &  HV($G,\epsilon$) & $\text{Hyy}^{\text{inv}}$($G,\epsilon$)  
\\ \hline \hline
MSTSA \cite{khanduri2021momentum}  &$ \mathcal{O}(\epsilon^{-2})$ & $\mathcal{O}(\epsilon^{-2})$& $\mathcal{O}(\epsilon^{-2})$& $\mathcal{\widetilde O}(\epsilon^{-2})$ &  {/} 
\\ \hline
SEMA \cite{guo2021stochastic}  &$ \mathcal{\widetilde O}(\epsilon^{-2})$ & $\mathcal{\widetilde O}(\epsilon^{-2})$& $\mathcal{\widetilde O}(\epsilon^{-2})$& $\mathcal{\widetilde O}(\epsilon^{-2})$ &  {/} 
\\ \hline
STABLE \cite{chen2021single}  & $\mathcal{O}(\epsilon^{-2})$ & $\mathcal{O}(\epsilon^{-2})$ &  {/}& {/} & $\mathcal{O}(\epsilon^{-2})$ 
\\ \hline
stocBiO \cite{ji2021bilevel}  & $\mathcal{O}(\epsilon^{-2})$ & $\mathcal{O}(\epsilon^{-2})$ &  $\mathcal{O}\left(\epsilon^{-2}\right)$& $\mathcal{\widetilde O}\left(\epsilon^{-2}\right)$ &{/}  
\\ \hline
RSVRB \cite{guo2021randomized} (Concurrent)  & $\mathcal{O}(\epsilon^{-1.5})$ & $\mathcal{O}(\epsilon^{-1.5})$ &  $\mathcal{O}\left(\epsilon^{-1.5}\right)$& {/} &  $\mathcal{O}\left(\epsilon^{-1.5}\right)$
\\ \hline
SUSTAIN \cite{khanduri2021near} (Concurrent) & $\mathcal{O}(\epsilon^{-1.5})$ & $\mathcal{O}(\epsilon^{-1.5})$ &  $\mathcal{O}\left(\epsilon^{-1.5}\right)$& $\mathcal{\widetilde O}\left(\epsilon^{-1.5}\right)$ &{/}  
\\ \hline
   \cellcolor{blue!15}{MRBO (ours)} & \cellcolor{blue!15}{$\mathcal{O}(\epsilon^{-1.5})$} & \cellcolor{blue!15}{$\mathcal{O}(\epsilon^{-1.5})$ } & \cellcolor{blue!15}{$\mathcal{O}\left(\epsilon^{-1.5}\right)$} & \cellcolor{blue!15}{$\mathcal{\widetilde O}\left(\epsilon^{-1.5}\right)$}&
   \cellcolor{blue!15}{/} 
\\ \hline
   \cellcolor{blue!15}{VRBO (ours)} & \cellcolor{blue!15}{$\mathcal{\widetilde O}(\epsilon^{-1.5})$} & \cellcolor{blue!15}{$\mathcal{\widetilde O}(\epsilon^{-1.5})$ } & \cellcolor{blue!15}{$\mathcal{\widetilde O}\left(\epsilon^{-1.5}\right)$} & \cellcolor{blue!15}{$\mathcal{\widetilde O}\left(\epsilon^{-1.5}\right)$} &
   \cellcolor{blue!15}{/}
   \\ \hline 
  \end{tabular}\label{tb:order}
   \begin{tablenotes}
  \item $\mbox{\normalfont Gc}(F,\epsilon)$ and $\mbox{\normalfont Gc}(G,\epsilon)$: 
 number of gradient evaluations w.r.t. $F$ and $G$.
 \item $\mbox{\normalfont Jv}(G,\epsilon)$: number of Jacobian-vector products  $\nabla_x\nabla_y G (\cdot)v$. $\mathcal{\widetilde O}(\cdot)$: omit $\log\frac{1}{\epsilon}$ terms.
  \item  $\mbox{\normalfont Hv}(G,\epsilon)$: number of Hessian-vector products $\nabla_y^2G(\cdot)v$.
  \item $\mbox{\normalfont Hyy}^{\text{inv}}(G,\epsilon)$: number of  evaluations of Hessian inverse $[\nabla_y^2G]^{-1}$.  
 \end{tablenotes}
 \end{threeparttable}
\end{table*}

\subsection{Related Works}
\noindent {\bf Bilevel optimization approaches:}
At the early stage of bilevel optimization studies, a class of constraint-based algorithms~\cite{hansen1992new, shi2005extended, moore2010bilevel} were proposed, which tried to penalize the outer function with the optimality conditions of the inner problem. To further simplify the implementation of constraint-based bilevel methods,  
gradient-based bilevel algorithms were then proposed, which include but not limited to AID-based~\cite{rajeswaran2019meta,franceschi2018bilevel,shaban2019truncated,ji2021lower}, ITD-based~\cite{gould2016differentiating,pedregosa2016hyperparameter,ghadimi2018approximation,ji2020convergence,ji2020multi} methods, and stochastic bilevel optimizers such as BSA~\cite{ghadimi2018approximation}, stocBiO~\cite{ji2021bilevel}, and TTSA~\cite{hong2020two}. The finite-time (i.e., non-asymptotic) convergence analysis for bilevel optimization has been recently studied in several works~\cite{ghadimi2018approximation,ji2021bilevel,hong2020two}. In this paper, we propose two novel stochastic bilevel algorithms using momentum recursive and variance reduction techniques, and show that they order-wise improve the computational complexity over existing stochastic bilevel optimizers.


\noindent{\bf Momentum-based recursive approaches:}
The momentum recursive technique was first introduced by~\cite{cutkosky2019momentum,tran2019hybrid} for minimization problems, and has been shown to achieve improved computational complexity over SGD-based updates in theory and in practice. Several works~\cite{khanduri2021momentum,chen2021single,guo2021stochastic} applied the similar single-loop momentum-based strategy to bilevel optimization to accelerate the SGD-based bilevel algorithms such as BSA~\cite{ghadimi2018approximation} and stocBiO~\cite{ji2021bilevel}. However, the computational complexities of these momentum-based algorithms are not shown to outperform that of stocBiO. In this paper, we propose a new single-loop momentum-based recursive bilevel optimizer (MRBO), which we show achieves order-wisely lower complexity than existing stochastic bilevel optimizers.  


\noindent{\bf Variance reduction approaches:}
Variance reduction has been studied extensively for conventional minimization problems, and many algorithms have been designed along this line, including but not limited to SVRG~\cite{johnson2013accelerating,li2018simple}, SARAH~\cite{nguyen2017sarah}, SPIDER~\cite{fang2018spider}, SpiderBoost~\cite{wang2018spiderboost,wang2019spiderboost,ji2019improved} and SNVRG~\cite{zhou2018stochastic}. Several works~\cite{luo2020stochastic,xu2020gradient,yang2020global,rafique2021weakly} recently employed such techniques for minimax optimization to achieve better complexities. In this paper, we propose the first-known variance reduction-based bilevel optimizer (VRBO), which achieves a near-optimal computational complexity and outperforms existing stochastic bilevel algorithms.  

\noindent{\bf Two concurrent works:}
As we were finalizing this submission, two concurrent studies were posted on arXiv recently (\cite{khanduri2021near} was posted on May 8 and \cite{guo2021randomized} was posted on May 5). Both studies overlap only with our MRBO algorithm, nothing similar to our VRBO. Specifically,
\cite{khanduri2021near} and \cite{guo2021randomized} respectively proposed the SUSTAIN and RSVRB algorithms for bilevel optimization, both using momentum-based design as our MRBO. Although SUSTAIN and RSVRB have been shown to achieve the same theoretical complexity of $\mathcal{O}(\epsilon^{-1.5})$ as our MRBO (and VRBO), both algorithms have major drawbacks in their design, so that their empirical performance (as we demonstrate in our experiments) is much worse that our MRBO (and even worse than our VRBO). SUSTAIN adopts only single-sample for each update (whereas MRBO uses minibatch for stability); and RSVRB requires to compute Hessian inverse at each iteration (whereas MRBO uses Hessian-vector products for fast computation). As an additional note, our experiments demonstrate that our VRBO significantly outperforms all these single-loop momentum-based algorithms SUSTAIN and RSVRB as well as our MRBO. 

\section{Two New Algorithms}

In this section, we propose two new algorithms for bilevel optimization.
Firstly, we introduce the hypergradient of the objective function $\Phi(x_k)$, which is useful for designing stochastic algorithms.
\begin{property}
\label{prop:hypergradori}
The (hyper)gradient of $\Phi(x) = f(x,y^{\ast}(x))$ in \cref{eq:main}
takes a form of 
	\begin{equation}\label{eq:hypergradient}
	\nabla \Phi(x) = \nabla_x f(x, y^{\ast}(x)) - \nabla_x\nabla_y g(x, y^{\ast}(x)) [\nabla^2_y g(x, y^{\ast}(x))]^{-1} \nabla_y f(x, y^{\ast}(x)).
	\end{equation}
\end{property}
However, it is not necessary to compute $y^{\ast}$ for updating $x$ at every iteration, and it is not time and memory efficient to compute Hessian inverse matrix in \cref{eq:hypergradient} explicitly. Here,  we estimate the hypergradient similarly to~\cite{ji2021bilevel,ghadimi2018approximation}, which takes a form of
\begin{align}\label{eq:hypergradientest}
 \overline{\nabla} \Phi(x) = \nabla_xf(x, y) - \nabla_x \nabla_y g(x, y) \eta \sum_{q=-1}^{Q-1} (I-\eta \nabla_y^2g(x, y))^{q+1}\nabla_yf(x,y),
\end{align}
where the Neumann series $\eta\sum_{i=0}^{\infty}(I-\eta G)^i=G^{-1}$ is applied to approximate the Hessian inverse.
\subsection{Momentum-based Recursive Bilevel Optimizer (MRBO)}
As shown in \Cref{alg:mrbio}, we propose a {\bf M}omentum-based {\bf R}ecursive {\bf B}ilevel {\bf O}ptimizer (MRBO) for solving the bilevel problem in \cref{eq:main}.


\begin{algorithm}
\caption{Momentum-based Recursive Bilevel Optimizer (MRBO)}
\label{alg:mrbio}
\begin{algorithmic}[1]
\STATE \textbf{Input:} Stepsize $ \lambda, \gamma >0$, Coefficients $\alpha_0, \beta_0$, Initializers $x_0,  y_0$, Hessian Estimation Number $Q$, Batch Size $S$, Constant $c_1, c_2, m, d>0$
\FOR {$k=0,1,\ldots,K$}
\STATE Draw Samples $\mathcal{B}_y, \mathcal{B}_{x}=\{\mathcal{B}_j(j=1,\ldots,Q), \mathcal{B}_F,\mathcal{B}_G\}$ with batch size $S$ for each component
\IF {$k=0$:}
\STATE $v_k=\widehat{\nabla}\Phi(x_k; \mathcal{B}_x), u_k = \nabla_yG(x_k, y_k; \mathcal{B}_y)$
\ELSE 
\STATE  $v_k= \widehat{\nabla}\Phi(x_k;\mathcal{B}_x) + (1-\alpha_k)(v_{k-1}-\widehat{\nabla} \Phi(x_{k-1};\mathcal{B}_x))$ 
\STATE  $u_k = \nabla_yG(x_k, y_k; \mathcal{B}_y) + (1-\beta_k)(u_{k-1} - \nabla_yG(x_{k-1}, y_{k-1}; \mathcal{B}_y))$
\ENDIF
\STATE \textbf{update:} $\eta_k=\frac{d}{\sqrt[3]{m+k}},\quad \alpha_{k+1} = c_1\eta_k^2, \quad \beta_{k+1}=c_2\eta_k^2$
\STATE $x_{k+1} = x_k - \gamma\eta_k v_k, \quad y_{k+1} = y_k - \lambda\eta_k u_k$

\ENDFOR
\end{algorithmic}
\end{algorithm}
	
MRBO updates in a single-loop manner, where the momentum recursive technique STORM \cite{cutkosky2019momentum} is employed for updating both $x$ and $y$ at each iteration simultaneously. To update $y$, at step $k$, MRBO first constructs the momentum-based gradient estimator $u_k$ based on the current $\nabla_y G(x_k,y_k;\mathcal{B}_y)$ and the previous $\nabla_y G(x_{k-1},y_{k-1};\mathcal{B}_y)$ using a minibatch $\mathcal{B}_y$ of samples (see line 8 in \Cref{alg:mrbio}). Note that the hyperparameter $\beta_k$ decreases at each iteration, so that the gradient estimator $u_k$ is more determined by the previous $u_{k-1}$, which improves the stability of gradient estimation, especially when $y_k$ is close to the optimal point. Then MRBO uses the gradient estimator for updating $y_k$ (see line 11). The stepsize $\eta_k$ decreases at each iteration to reduce the convergence error. 

To update $x$, at step $k$, MRBO first constructs the momentum-based recursive hypergradient estimator $v_k$ based on the current $\widehat{\nabla}\Phi(x_k;\mathcal{B}_x)$ and the previous $\widehat{\nabla}\Phi(x_{k-1};\mathcal{B}_x)$ computed using several independent minibatches of samples $\mathcal{B}_x=\{\mathcal{B}_j(j=1,\ldots,Q), \mathcal{B}_F,\mathcal{B}_G\}$ (see line 7 in \Cref{alg:mrbio}). The hyperparameter $\alpha_k$ decreases at each iteration, so that the new gradient estimation $v_k$ is more determined by the previous $v_{k-1}$, which improves the stability of gradient estimation, especially when $x_k$ is around the optimal point.  Specifically, the hypergradient estimator $\widehat{\nabla}\Phi(x_k;\mathcal{B}_x)$ is designed based on the expected form in \cref{eq:hypergradientest}, and takes a form of:
\begin{align}
\label{eq:xgest_mrbio}
	 & \widehat{\nabla}\Phi(x_k;\mathcal{B}_x) = \nabla_xF(x_k, y_k; \mathcal{B}_{F}) \nonumber \\ 
	 & \quad - \nabla_x\nabla_yG(x_k, y_k; \mathcal{B}_{G}) \eta \sum_{q=-1}^{Q-1} \prod_{j=Q-q}^{Q} (I-\eta \nabla_y^2G(x_k, y_k; \mathcal{B}_j))\nabla_yF(x_k,y_k;\mathcal{B}_{F}),
\end{align}
Note that MRBO computes the above estimator recursively using only {\bf Hessian vectors} rather than {\bf Hessians} (see \Cref{sec:hessianvector}) in order to reduce the memory and computational cost. Then MRBO uses the estimated gradient $v_k$ for updating $x_k$ (see line 11). The stepsize $\eta_k$ decreases at each iteration to facilitate the convergence. 

\subsection{Variance Reduction Bilevel Optimizer (VRBO)}
Although all of the existing momentum algorithms ~\cite{chen2021single,khanduri2021momentum,guo2021stochastic} (and two current studies \cite{khanduri2021near,guo2021randomized}) for bilevel optimization follow the single-loop design, empirical results in~\cite{ji2021bilevel} suggest that {\bf double-loop} bilevel algorithms can achieve much better performances than {\bf single-loop} algorithms. Thus, 
as shown in \Cref{alg:vrbio}, we propose a double-loop algorithm called {\bf V}ariance {\bf R}eduction {\bf B}ilevel {\bf O}ptimizer (VRBO). VRBO adopts the variance reduction technique in SARAH~\cite{nguyen2017sarah}/SPIDER~\cite{fang2018spider} for bilevel optimization, which is suitable for designing double-loop algorithms. Specifically, VRBO constructs the recursive variance-reduced gradient estimators for updating both $x$ and $y$, where each update of $x$ in the outer-loop is followed by $(m+1)$ inner-loop updates of $y$. VRBO divides the outer-loop iterations into epochs, and at the beginning of each epoch computes the hypergradient estimator $\widehat{\nabla}\Phi(x_k, y_k;\mathcal{S}_1)$ and the gradient $\nabla_yG(x_{k}, y_{k};\mathcal{S}_1)$ based on a relatively large batch $\mathcal{S}_1$ of samples for variance reduction, 
where $\widehat{\nabla}\Phi(x_k, y_k;\mathcal{S}_1)$ takes a form of
\begin{align}
\label{eq:xgest_vrbio}
	\widehat{\nabla}& \Phi(x_k ,y_k;\mathcal{S}_1) = \frac{1}{S_1}\sum_{i=1}^{S_1}\Big( \nabla_xF(x_k, y_k; \xi_i) \nonumber \\ 
	&  \quad - \nabla_x\nabla_yG(x_k, y_k; \zeta_i) \eta \sum_{q=-1}^{Q-1} \prod_{j=Q-q}^{Q} (I-\eta \nabla_y^2G(x_k, y_k; \zeta_{i}^j))\nabla_yF(x_k,y_k;\xi_i)\Big),
\end{align}
where all samples in $\mathcal{S}_1=\{\zeta_i^j(j=1,\ldots,Q), \xi_i, \zeta_i,i=1,\ldots,S_1\}$ are independent.
Note that \cref{eq:xgest_vrbio} takes a different form from MRBO in \cref{eq:xgest_mrbio}, but the Hessian-vector computation method for MRBO is still applicable here.
Then, VRBO recursively updates the gradient estimators for $\nabla_yG(\widetilde{x}_{k,t}, \widetilde{y}_{k,t};\mathcal{S}_2)$ and $\widehat{\nabla}\Phi(\widetilde{x}_{k,t}, \widetilde{y}_{k,t};\mathcal{S}_2)$ (which takes the same form as \cref{eq:xgest_vrbio}) with a small sample batch $\mathcal{S}_2$ (see lines 11 to 16) during  inner-loop iterations. 


We remark that VRBO is the first algorithm that adopts the recursive variance reduction method for bilevel optimization. As we will shown in \Cref{sec:mainresults}, VRBO achieves the same nearly-optimal computational complexity as MRBO (and  outperforms all other existing algorithms). More interestingly, as a double-loop algorithm, VRBO empirically significantly outperforms all existing single-loop momentum algorithms including MRBO. More details and explanation are provided in \Cref{sec:experiment}.



\begin{algorithm}
	\caption{Variance Reduction Bilevel Optimizer (VRBO)}
	\label{alg:vrbio}
	\begin{algorithmic}[1]
		\STATE \textbf{Input:} Stepsize $\beta, \alpha >0$, Initializer $x_0,  y_0$, Hessian $Q$, Sample Size $S_1, S_2$, Periods $q$
		\FOR {$k=0,1,\ldots,K$}
		\IF {mod$(k,q)=0$:}
		\STATE Draw a batch $\mathcal{S}_1$ of i.i.d.~samples
		\STATE $u_k = \nabla_yG(x_k, y_k;\mathcal{S}_1)$, $v_k = \widehat{\nabla}\Phi(x_k,y_k;\mathcal{S}_1)$
		\ELSE
		\STATE $u_k =\widetilde{u}_{k-1,m+1}$, $v_k = \widetilde{v}_{k-1,m+1}$
		 \ENDIF
		 \STATE $x_{k+1} = x_k-\alpha v_k$
		\STATE Set \,$\widetilde{x}_{k,-1}=x_k, \widetilde{y}_{k,-1}=y_k, \widetilde{x}_{k,0}=x_{k+1}, \widetilde{y}_{k,0}=y_k, \widetilde{v}_{k,-1}=v_k, \widetilde{u}_{k,-1}=u_k$
	    \FOR {$t=0,1,\ldots,m+1$}
	    \STATE Draw a batch $\mathcal{S}_2$ of i.i.d~samples
	    \STATE $\widetilde{v}_{k,t}=\widetilde{v}_{k,t-1}+\widehat{\nabla}\Phi(\widetilde{x}_{k,t}, \widetilde{y}_{k,t};\mathcal{S}_2)-\widehat{\nabla}\Phi(\widetilde{x}_{k,t-1}, \widetilde{y}_{k,t-1};\mathcal{S}_2)$
	    \STATE $\widetilde{u}_{k,t} = \widetilde{u}_{k,t-1}+\nabla_yG(\widetilde{x}_{k,t},\widetilde{y}_{k,t};\mathcal{S}_2) - \nabla_yG(\widetilde{x}_{k,t-1},\widetilde{y}_{k,t-1};\mathcal{S}_2)$\
	    \STATE $\widetilde{x}_{k,t+1} = \widetilde{x}_{k,t}, \widetilde{y}_{k,t+1}=\widetilde{y}_{k,t}-\beta \widetilde{u}_{k,t}$
	    \ENDFOR
	    \STATE $y_{k+1} = \widetilde{y}_{k,m+1}$
		\ENDFOR
	\end{algorithmic}
\end{algorithm}


\section{Main Results}\label{sec:mainresults}
In this section, we first introduce several standard assumptions for the analysis, and then present the convergence results for the proposed MRBO and VRBO algorithms. 

\subsection{Technical Assumptions and Definitions} \label{subsec:assumptions}
\begin{Assumption}
\label{ass:conv}
Assume that the inner function $G(x,y;\zeta)$ is $\mu$-strongly-convex w.r.t.\ $y$ for any $\zeta$.
\end{Assumption}

We then make the following assumptions on the Lipschitzness and bounded variance, as adopted by the existing studies~\cite{ghadimi2018approximation,ji2021bilevel,hong2020two} on stochastic bilevel optimization.
\begin{Assumption}\label{ass:lip} 
Let $z:=(x,y)$. Assume the functions $F(z;\xi)$ and $G(z;\zeta)$ satisfy, for any $\xi$ and $\zeta$,
\begin{list}{}{\topsep=0.ex \leftmargin=0.15in \rightmargin=0.in \itemsep =0.02in}
\item[a)] $F(z;\xi)$ is $M$-Lipschitz, i.e., for any $z, z^\prime$, $|F(z;\xi)-F(z^\prime;\xi)|\leq M\|z-z^\prime\|.$
\item[b)] $\nabla F(z;\xi)$ and $\nabla G(z;\zeta)$ are $L$-Lipschitz, i.e., for any $z, z^\prime$, 
\begin{align*}
\|\nabla F(z;\xi)-\nabla F(z^\prime;\xi)\| \leq L\|z-z^{\prime}\|,\quad \|\nabla G(z;\zeta)-\nabla G(z^{\prime};\zeta)\|\leq L\|z-z^{\prime}\|.
\end{align*}
\item[c)] $\nabla_x\nabla_yG(z;\zeta)$ is $\tau$-Lipschitz, i.e., for any $z, z^\prime$, $\|\nabla_x \nabla_y G(z;\zeta)-\nabla_x \nabla_y G(z^\prime;\zeta)\|\leq \tau \|z-z^\prime\|.$
\item[d)] $\nabla_y^2 G(z;\zeta)$ is $\rho$-Lipschitz, i.e., for any $z, z^\prime$, $\|\nabla_y^2 G(z;\zeta) - \nabla_y^2 G(z^\prime;\zeta)\| \leq \rho \|z-z^\prime\|$.
\end{list}
\end{Assumption}
Note that~\Cref{ass:lip} also implies that $\mathbb{E}_\xi \|\nabla F(z;\xi)-\nabla f(z)\|^2 \leq M^2$, $\mathbb{E}_\zeta \|\nabla_x \nabla_y G(z;\zeta)-\nabla_x\nabla_y g(z)\|^2 \leq L^2$ and 
$\mathbb{E}_\zeta \|\nabla_y^2 G(z; \zeta)-\nabla_y^2g(z) \|^2 \leq L^2$.
\begin{Assumption}
\label{ass:var}
Assume that $\nabla G(z;\xi)$ has 
bounded variance, i.e., $\mathbb{E}_\xi \|\nabla G(z;\xi)-\nabla g(z)\|^2 \leq \sigma^2$. 
\end{Assumption}

Assumptions \ref{ass:lip} and \ref{ass:var} require the Lipschitzness conditions to hold for the gradients and second-order derivatives of the inner and outer objective functions, which further imply the gradient of the outer objective function is bounded. Such assumptions have also been adopted by the existing studies~\cite{ji2021bilevel,hong2020two,khanduri2021near, khanduri2021momentum, guo2021randomized} for stochastic bilevel optimization. Furthermore, these assumptions are mild in practice as long as the iterates along practical training paths are bounded. All our experiments indicate that these iterates are well located in a bounded regime. It is also possible to consider a bilevel problem over a convex compact set, which relaxes the boundedness assumption. By introducing a projection of iterative updates into such a set, our analysis for the unconstrained setting can be extended easily to such a constrained problem.

We next define the $\epsilon$-stationary point for a nonconvex function as the convergence criterion. 

\begin{definition}
We call $\bar{x}$ an $\epsilon$-stationary point for a function $\Phi(x)$ if $\|\nabla\Phi(\bar{x})\|^2\leq \epsilon.$
\end{definition}


\subsection{Convergence Analysis of MRBO Algorithm}\label{sec:convergence_mrbo}

To analyze the convergence of MRBO, bilevel optimization presents two major challenges due to the momentum recursive method in MRBO, beyond the previous studies of momentum in conventional minimization and minimax optimization. (a) Outer-loop updates of bilevel optimization use hypergradients, which involve both the first-order gradient and the Hessian-vector product. Thus, the analysis of the momentum recursive estimator for such a hypergradient is much more complicated than that for the vanilla gradient. (b) Since MRBO applies the momentum-based recursive method to both inner- and outer-loop iterations, the analysis needs to capture the interaction between the inner-loop gradient estimator and the outer-loop hypergradient estimator. 
Below, we will provide two major properties for MRBO, which develop new analysis for handling the above two challenges.

In the following proposition, we characterize the variance bound for the hypergradient estimator in bilevel optimization, and further use such a bound to characterize the variance of the momentum recursive estimator of the hypergradient. 
\begin{proposition}
\label{prop:mrbio_epi}
Suppose Assumptions \ref{ass:conv}, \ref{ass:lip} and \ref{ass:var} hold and $\eta < \frac{1}{L}$, the hypergradient estimator $\widehat{\nabla}\Phi(x_k;\mathcal{B}_x)$ w.r.t.\ x based on a minibatch $\mathcal{B}_x$ has bounded variance
    \begin{align}
    \label{eq:mrbio_estim}
        \mathbb{E} \|\widehat{\nabla}\Phi(x_k;\mathcal{B}_x)-\overline{\nabla}\Phi(x_k)\|^2 \leq G^2,
    \end{align}
    where $G^2 = \frac{2M^2}{S}+\frac{12M^2L^2\eta^2(Q+1)^2}{S}+\frac{4M^2L^2(Q+2)(Q+1)^2\eta^4\sigma^2}{S}$.
    Further, let $\bar{\epsilon}_k=v_k-\overline{\nabla}\Phi(x_k)$, where $v_k$ denotes the momentum recursive estimator for the hypergradient. Then the per-iteratioon variance bound of $v_k$ satisfies
    \begin{align}
    \label{eq:mrbio_update}
	\mathbb{E} \| \bar{\epsilon}_k \|^2 \leq \mathbb{E} [2&\alpha_k^2 G^2 + 2(1-\alpha_k)^2L_Q^2 \| x_k-x_{k-1} \|^2 \nonumber
	\\&+ 2(1-\alpha_k)^2L_Q^2 \| y_k-y_{k-1} \|^2 + (1-\alpha_k)^2 \|\bar{\epsilon}_{k-1} \|^2],
	\end{align}
	where $L_Q^2=2L^2+4\tau^2\eta^2M^2(Q+1)^2+8L^4\eta^2(Q+1)^2+2L^2\eta^4M^2\rho^2Q^2(Q+1)^2.$
\end{proposition}
    
The variance bound $G$ of the hypergradient in \cref{eq:mrbio_estim} scales with the number $Q$ of Neumann series terms (i.e., the number of Hessian vectors) and can be reduced by that minibatch size $S$.
    
Then the bound \cref{eq:mrbio_update} further captures how the variance $\| \bar{\epsilon}_k\|$ of momentum recursive hypergradient estimator changes after one iteration. Clearly, the term $(1-\alpha_k)^2 \|\bar{\epsilon}_{k-1} \|^2$ indicates a variance reduction per iteration, and the remain three terms capture the impact of the randomness due to the update in step $k$, including the variance of the stochastic hypergradient estimator $G^2$ (as captured in \cref{eq:mrbio_estim}) and the stochastic update of both variables $x$ and $y$. In particular, the variance reduction term plays a key role in the performance improvement for MRBO over other existing algorithms.


\begin{proposition}
\label{prop:mrbio_phi}
Suppose Assumptions \ref{ass:conv}, \ref{ass:lip}, \ref{ass:var} hold. Let $\eta < \frac{1}{L}$ and $\gamma \leq \frac{1}{4L_{\Phi}\eta_k}$, where $L_{\Phi}=L+\frac{2L^2+\tau M^2}{\mu}+\frac{\rho LM+L^3+\tau ML}{\mu^2}+\frac{\rho L^2 M}{\mu^3}$. Then, we have 
\begin{align*}
 \mathbb{E}[\Phi(x_{k+1})] 
 \leq \mathbb{E}[\Phi(x_k)]+2\eta_k \gamma( {L^{\prime}}^2 \|y_k-y^{\ast}(x_k)\|^2 + \|\bar{\epsilon}_k\|^2 + C_Q^2) - \frac{1}{2\gamma\eta_k} \|x_{k+1}-x_k\|^2,
	\end{align*}
where $C_Q=\frac{(1-\eta\mu)^{Q+1}ML}{\mu},  {L^{\prime}}^2=\max \{(L+\frac{L^2}{\mu}+\frac{M\tau}{\mu}+\frac{LM\rho}{\mu^2})^2,L_Q^2\}$.
\end{proposition}
\Cref{prop:mrbio_phi} characterizes how the objective function value decreases (i.e., captured by $\mathbb{E}[\Phi(x_{k+1})]-\mathbb{E}[\Phi(x_k)]$) due to one-iteration update $\|x_{k+1}-x_k\|^2$ of variable $x$ (last term in the bound). Such a value reduction is also affected by the tracking error $\|y_k-y^{\ast}(x_k)\|^2$ of the variable $y$ (i.e., $y_k$ does not equal the desirable $y^{\ast}(x_k)$), the variance $\|\bar{\epsilon}_k\|^2$ of momentum recursive hypergradient estimator, and the Hessian inverse approximation error $C_Q$ w.r.t.\ hypergradient.

Based on Propositions \ref{prop:mrbio_epi} and \ref{prop:mrbio_phi}, we next characterize the convergence of MRBO.

\begin{theorem}
	\label{thm:mrbio}
	Apply MRBO to solve the problem \cref{eq:main}. Suppose Assumptions \ref{ass:conv}, \ref{ass:lip}, and \ref{ass:var} hold. Let hyperparameters $c_1\geq \frac{2}{3d^3}+\frac{9\lambda \mu}{4}, c_2\geq \frac{2}{3d^3}+\frac{75{L'}^2\lambda}{2\mu}, m\geq \max\{2, d^3, (c_1d)^3, (c_2d)^3\}, y_1=y^{\ast}(x_1),\eta < \frac{1}{L}, 0\leq \lambda \leq \frac{1}{6L}, 0\leq \gamma \leq \min \{\frac{1}{4L_\Phi \eta_K},\frac{\lambda \mu}{\sqrt{150{L^{\prime}}^2L^2/\mu^2+8\lambda\mu(L_Q^2+L^2)}}  \}$. Then, we have 
	\begin{equation}
	\label{eq:mrbio_conv}
	 \frac{1}{K} \sum_{k=1}^{K}\left(\frac{{L^{\prime}}^2}{4} \|y^{\ast}(x_k)-y_k\|^2+\frac{1}{4}\|\bar{\epsilon}_k\|^2+\frac{1}{4\gamma^2 \eta_k^2}\|x_{k+1}-x_k\|^2\right) \leq \frac{M'}{K}(m+K)^{1/3},
	\end{equation}
	where ${L^{\prime}}^2$ is defined in \Cref{prop:mrbio_phi}, and $M' =\frac{\Phi(x_1)-\Phi^{\ast}}{\gamma d} +\left(\frac{2G^2(c_1^1+c_2^2)d^2}{\lambda \mu}+\frac{2C_Q^2d^2}{\eta_K^2}\right)\log(m+K)+ \frac{2G^2}{S \lambda \mu d \eta_0}$.
\end{theorem}



\Cref{thm:mrbio} captures the simultaneous convergence of the variables $x_k$, $y_k$ and $\|\bar{\epsilon}_k\|$: the tracking error $\|y^{\ast}(x_k)-y_k\|$ converges to zero, and the variance $\|\bar{\epsilon}_k\|$ of the momentum recursive hypergradient estimator reduces to zero, both of which further facilitate the convergence of $x_k$ and the algorithm.

By properly choosing the hyperparameters in \Cref{alg:mrbio} to satisfy the conditions in \Cref{thm:mrbio}, we obtain the following computational complexity for MRBO. 
\begin{coro} 
    \label{coro:mrbio}
Under the same conditions of Theorem \ref{thm:mrbio} and choosing  $K=\mathcal{O}(\epsilon^{-1.5}), Q=\mathcal{O}(\log(\frac{1}{\epsilon}))$, 
MRBO in \Cref{alg:mrbio} finds an $\epsilon$-stationary point with the gradient complexity of $\mathcal{O}(\epsilon^{-1.5})$ and the (Jacobian-) Hessian-vector complexity of $\mathcal{\widetilde O}(\epsilon^{-1.5})$.
\end{coro}
As shown in~\Cref{coro:mrbio}, MRBO achieves the computational complexity of $\mathcal{\widetilde O}(\epsilon^{-1.5})$, which outperforms all existing stochastic bilevel algorithms by a factor of $\mathcal{\widetilde O}(\epsilon^{-0.5})$ (see \Cref{tb:order}). Further, this also achieves the best known complexity of $\mathcal{\widetilde  O}(\epsilon^{-1.5})$ for vanilla nonconvex optimization via first-order stochastic algorithms. As far as we know, this is the first result to demonstrate the improved performance of single-loop recursive momentum over SGD-type updates for bilevel optimization. 




\subsection{Convergence Analysis of VRBO Algorithm}

To analyze the convergence of VRBO, we need to first characterize the statistical properties of the hypergradient estimator, in which all the gradient, Jacobian-vector, and Hessian-vector have recursive variance reduction forms. We then need to characterize how the inner-loop tracking error affects the outer-loop hypergradient estimation error in order to establish the overall convergence. The complication in the analysis is mainly due to the hypergradient in bilevel optimization, which does not exist in the previous studies of variance reduction in conventional minimization and minimax optimization.
Below, we provide two properties of VRBO for handling the aforementioned challenges.


In the following proposition, we characterize the variance of the hypergradient estimator, and further use such a bound to characterize the cumulative variances of both the hypergradient and inner-loop gradient estimators based on the recursive variance reduction technique over all iterations.
\begin{proposition}
\label{prop:vrbo_est}
Suppose Assumptions $\ref{ass:conv}$, $\ref{ass:lip}$, $\ref{ass:var}$ hold. Let $\eta < \frac{1}{L}$. Then the hypergradient estimator $\widehat{\nabla}\Phi(x_k,y_k;\mathcal{S}_1)$ defined in \cref{eq:xgest_vrbio} w.r.t.\ $x$ has bounded variance as
\begin{align}
\label{eq:vrbio_est}
 \mathbb{E}\|\widehat{\nabla}\Phi(x_k,y_k;\mathcal{S}_1)-\overline{\nabla}\Phi({x_k})\|^2 \leq \frac{\sigma'^2}{S_1},
\end{align}
where $\sigma'^2 = 2M^2+28L^2M^2\eta^2(Q+1)^2$. 
Let $\Delta_k =\mathbb{E}(\|v_k-\overline{\nabla}\Phi(x_k)\|^2+\|u_k-\nabla_yg(x_k,y_k)\|^2)$, where $v_k$ and $u_k$ denote the recursive variance reduction estimators for hypergradient and inner-loop gradient respectively. Then, the cumulative variance of $v_k$ and $u_k$ is bounded by
\begin{align}
\label{eq:vrbio_update}
	\sum_{k=0}^{K-1}\Delta_k \leq \frac{4\sigma'^2K}{S_1} + 22\alpha^2L_Q^{2}\sum_{k=0}^{K-2}\mathbb{E}\|v_k\|^2 + \frac{4}{3} \mathbb{E}\|\nabla_yg(x_0,y_0)\|^2.
	\end{align}
\end{proposition}
As shown in~\cref{eq:vrbio_est}, the variance bound of the hypergradient estimator increases with the number $Q$ of Hessian-vector products for approximating the Hessian inverse and can be reduced by the batch size $S_1$.  
Then \cref{eq:vrbio_update} further provides an upper bound on the cumulative variance $\sum_{k=0}^{K-1} \Delta_k$ of the recursive hypergradient estimator and inner-loop gradient estimator. 


\begin{proposition}
\label{prop:vrbo_phi}
Suppose Assumptions $\ref{ass:conv}$, $\ref{ass:lip}$, $\ref{ass:var}$ hold. Let $\eta < \frac{1}{L}$. Then, we have
\begin{align*}
\textstyle \mathbb{E}[\Phi(x_{k+1})]\leq \mathbb{E}[\Phi(x_k)] + \frac{\alpha{L'}^2}{\mu^2}\mathbb{E}\|\nabla_yg(x_k,y_k)\|^2 + \alpha \mathbb{E} \|\widetilde{\nabla} \Phi(x_k)-v_k\|^2 - (\frac{\alpha}{2}-\frac{\alpha^2}{2}L_\Phi)\mathbb{E}\|v_k\|^2,
\end{align*}
where ${L^{\prime}}^2=(L+\frac{L^2}{\mu}+\frac{M\tau}{\mu}+\frac{LM\rho}{\mu^2})^2$ and $\widetilde{\nabla} \Phi(x_k)$ takes a form of
\begin{align}\label{eq:tildephi}
    \widetilde{\nabla}\Phi(x_k) = \nabla_xf(x_k,y_k)-\nabla_x\nabla_yg(x_k,y_k)[\nabla_y^2g(x_k,y_k)]^{-1}\nabla_yf(x_k,y_k).
\end{align}
\end{proposition}
\Cref{prop:vrbo_phi} characterizes how the objective function value decreases (i.e., captured by $\mathbb{E}[\Phi(x_{k+1})] - \mathbb{E}[\Phi(x_k)]$) due to one iteration update $\|v_k\|^2$ of variable $x$ (last term in the bound). Such a value reduction is also affected by the moments of gradient w.r.t.\ $y$ and the variance of recursive hypergradient estimator.



Based on Propositions \ref{prop:vrbo_est} and \ref{prop:vrbo_phi}, we next characterize the convergence of VRBO.
\begin{theorem}
\label{thm:vrbio}

Apply VRBO to solve the problem \cref{eq:main}. Suppose Assumptions \ref{ass:conv}, \ref{ass:lip}, \ref{ass:var} hold. Let $\alpha=\frac{1}{20L_m^3}, \beta=\frac{2}{13L_Q}, \eta < \frac{1}{L}, S_2\geq 2(\frac{L}{\mu}+1)L\beta, m=\frac{16}{\mu \beta}-1, q=\frac{\mu L \beta S_2}{\mu+L}$ where $L_m = \max \{L_Q, L_\Phi\}$. Then, we have
\begin{align}
\label{eq:vrbio_conv}
  \frac{1}{K}\sum_{k=0}^{K-1}\mathbb{E}\|\nabla\Phi(x_k)\|^2 
    \leq \mathcal{O}(\frac{Q^4}{K}+\frac{Q^6}{S_1}+Q^4(1-\eta\mu)^{2Q}).
\end{align}
\end{theorem}

\Cref{thm:vrbio} shows that VRBO converges sublinearly w.r.t.~the number $K$ of iterations with the convergence error consisting of two terms. The first error term $\frac{Q^6}{S_1}$ is caused by the minibatch gradient and hypergradient estimation at outer loops and can be reduced by increasing the batch size $S_1$ (in fact, $Q$ scales only logarithmically with $S_1$). The second error term $Q^4(1-\eta\mu)^{2Q}$ is due to the approximation error of the Hessian-vector type of hypergradient estimation, which decreases exponentially fast w.r.t.~$Q$.  
By properly choosing the hyperparameters in \Cref{alg:vrbio}, we obtain the following complexity result for VRBO.

\begin{coro}
\label{coro:vrbio}
    Under the same conditions of \Cref{thm:vrbio}, choose $S_1 =\mathcal{O}( \epsilon^{-1}), S_2=\mathcal{O}(\epsilon^{-0.5}), Q=\mathcal{O}(\log (\frac{1}{\epsilon^{0.5}})), K = \mathcal{O}(\epsilon^{-1})$. Then, VRBO finds an $\epsilon$-stationary point with the gradient complexity of $\mathcal{\widetilde O}(\epsilon^{-1.5})$ and Hessian-vector complexity of $\mathcal{\widetilde O}(\epsilon^{-1.5})$.
\end{coro}

Similarly to MRBO, \Cref{coro:vrbio} indicates that VRBO also outperforms all existing stochastic algorithms for bilevel optimization by a factor of $\mathcal{\widetilde O}(\epsilon^{-0.5})$ (see \Cref{tb:order}). 
Further, although MRBO and VRBO achieve the same theoretical computational complexity, VRBO empirically performs much better than MRBO (as well as other single-loop momentum-based algorithms MSTSA~\cite{khanduri2021momentum}, STABLE~\cite{chen2021single}, and SEMA \cite{guo2021stochastic}), as will be shown in \Cref{sec:experiment}. 



We note that although our theory requires $Q$ to scale as $\mathcal{O}(\log (\frac{1}{\epsilon^{0.5}}))$, a very small $Q$ is sufficient to attain a fast convergence speed in experiments. For example, we choose $Q=3$ in our hyper-cleaning experiments as other benchmark algorithms such as AID-FP, reverse, and stocBiO. 

\section{Experiments}\label{sec:experiment}
In this section, we compare the performances of our proposed VRBO and MRBO algorithms with the following bilevel optimization algorithms: AID-FP~\cite{grazzi2020iteration}, reverse~\cite{franceschi2017forward} (both are double-loop deterministic algorithms), BSA~\cite{ghadimi2018approximation} (double-loop stochastic algorithm),  MSTSA~\cite{khanduri2021momentum} and SUSTAIN~\cite{khanduri2021near} (single-loop stochastic algorithms), 
STABLE~\cite{chen2021single} (single-loop stochastic algorithm with Hessian inverse computations), and stocBiO~\cite{ji2021bilevel} (double-loop stochastic algorithm). SEMA~\cite{guo2021stochastic} is not included in the list because it performs similarly to SUSTAIN. RSVRB~\cite{guo2021randomized} is not included since it performs similarly to STABLE. Our experiments are run over a hyper-cleaning application on MNIST\footnote{The experiments on CIFAR10 are still ongoing.}. We provide the detailed experiment specifications in \Cref{sec:extraexperiments}. 

As shown in~\Cref{fig:mainresults} (a) and (b), the convergence rate (w.r.t.~running time) of our VRBO and the SGD-type stocBiO converge much faster than other algorithms in comparison. Between VRBO and stocBiO, they have comparable performance, but our VRBO achieves a lower training loss as well as a more stable convergence. Further, our VRBO converges significantly faster than all single-loop momentum-based methods. This provides some evidence on the advantage of double-loop algorithms over single-loop algorithms for bilevel optimization. Moreover, our MRBO achieves the fastest convergence rate among all single-loop momentum-based algorithms, which is in consistent with our theoretical results. In \Cref{fig:mainresults} (c), we compare our algorithms MRBO and VRBO with three momentum-based algorithms, i.e., MSTSA, STABLE, and SUSTAIN, where SUSTAIN (proposed in the concurrent work~\cite{khanduri2021near}) achieves the same theoretical complexity as our MRBO and VRBO. However, it can be seen that MRBO and VRBO are significantly faster than the other three algorithms.



All three plots suggest an interesting observation that {\bf double-loop} algorithms tend to converge faster than {\bf single-loop} algorithms as demonstrated by (i) double-loop VRBO performs the best among all algorithms; and (ii) double-loop SGD-type stocBiO, GD-type reverse and AID-FP perform even better than single-loop momentum-accelerated stochastic algorithm MRBO; and (iii) double-loop SGD-type BSA (with single-sample updates) converges faster than single-loop momentum-accelerated stochastic MSTSA, STABLE and SUSTAIN (with single-sample updates). Such a phenomenon has been observed only in bilevel optimization (to our best knowledge), and occurs oppositely in minimization and minimax problems, where single-loop algorithms substantially outperform double-loop algorithms. 
The reason for this can be that the hypergradient estimation at the outer-loop in bilevel optimization is very sensitive to the inner-loop output. Thus, for each outer-loop iteration, sufficient inner-loop iterations in the double-loop structure provides a much more accurate output close to $y^*(x)$ than a single inner-loop iteration, and thus helps to estimate a more accurate hypergradient in the outer loop. This further facilitates better outer-loop iterations and yields faster overall convergence.




\begin{figure}[h]
	\centering    
	\subfigure[Noise rate $p=0.1$]{\includegraphics[width=45.0mm]{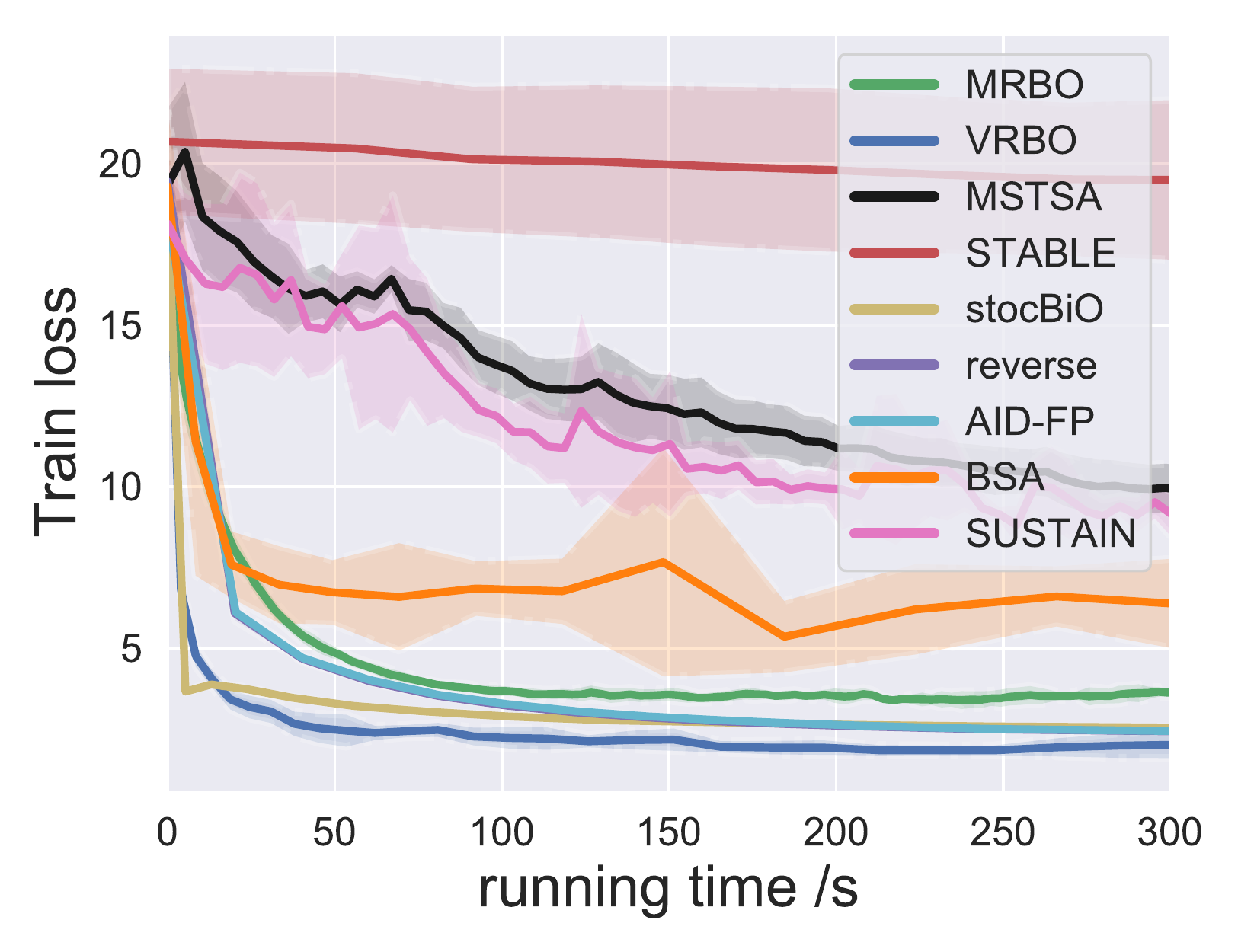}}
    \subfigure[Noise rate $p=0.15$]{\includegraphics[width=45.0mm]{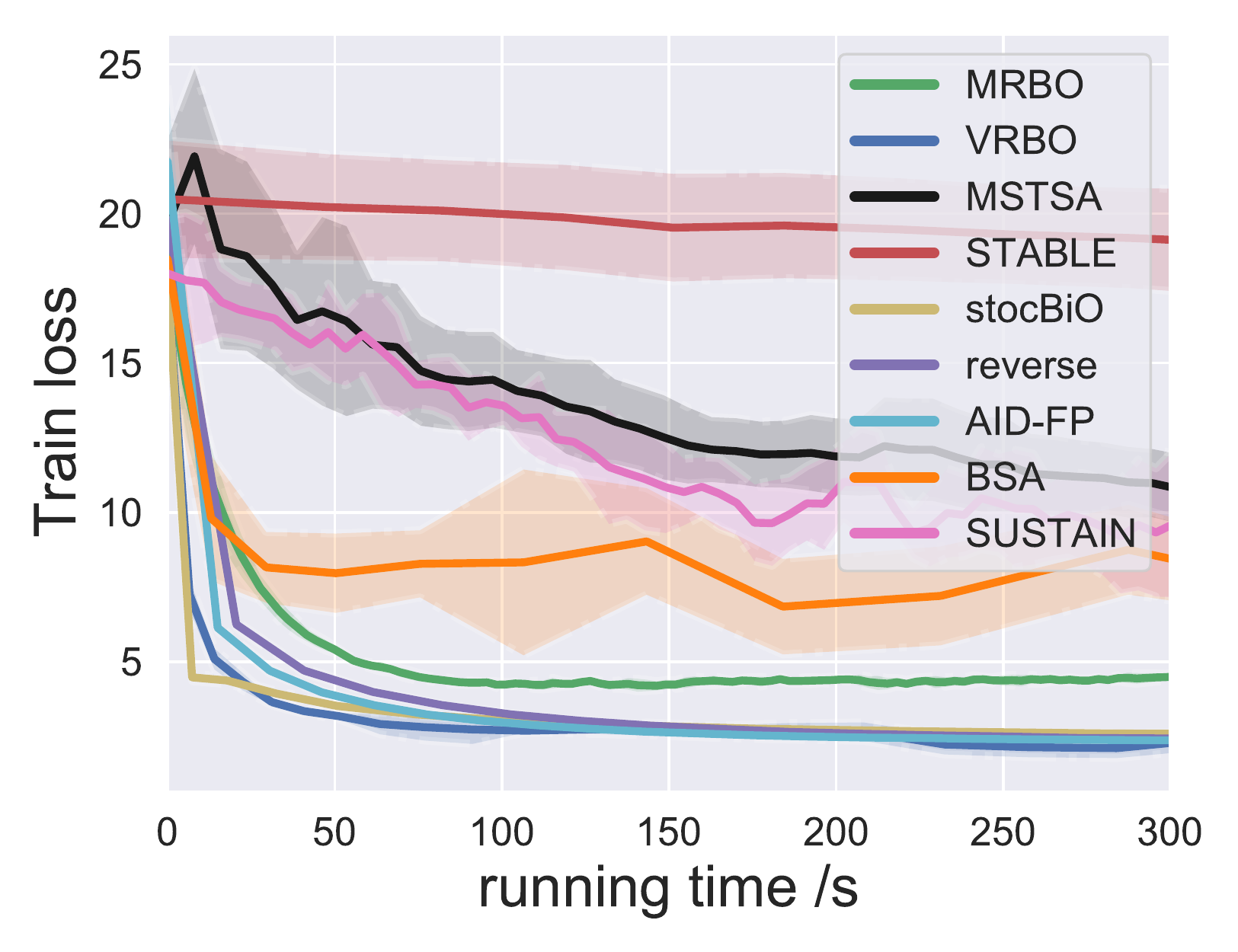}}
    \subfigure[Noise rate $p=0.1$]{ \includegraphics[width=45.0mm]{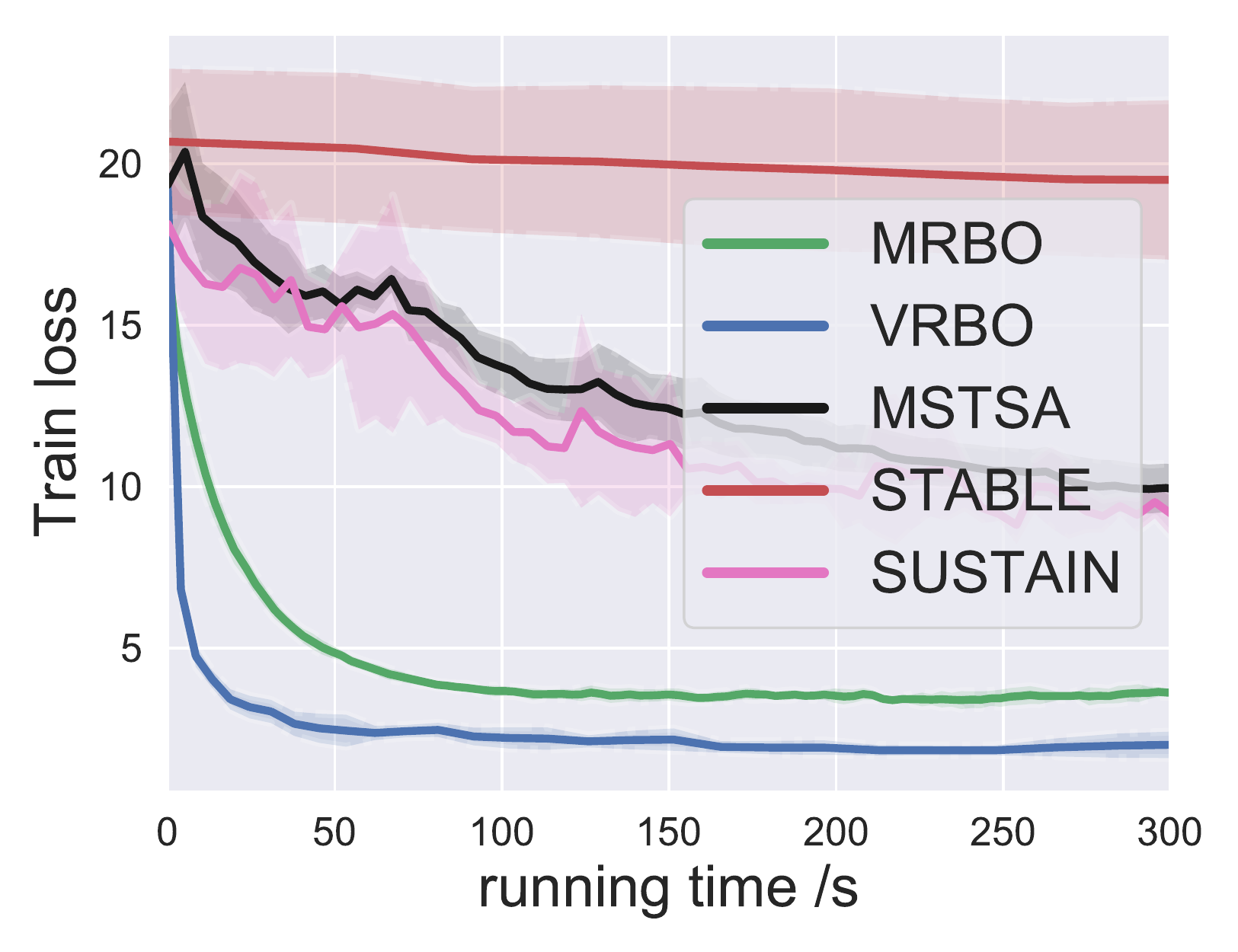}}
	\caption{training loss v.s. running time.}\label{fig:mainresults}
\end{figure}

\section{Conclusion}
In this paper, we proposed two novel algorithms MRBO and VRBO for the nonconvex-strongly-convex bilevel stochastic optimization problem, and showed that their computational complexities outperform all existing algorithms order-wise. In particular, MRBO is the first momentum algorithm that exhibits the order-wise improvement over SGD-type algorithms for bilevel optimization, and VRBO is the first that adopts the recursive variance reduction technique to accelerate bilevel optimization. 
Our experiments demonstrate the superior performance of these algorithms, and further suggest that the double-loop design may be more suitable for bilevel optimization than the single-loop structure. We anticipate that our analysis can be applied to studying bilevel problems under various other loss geometries. 
We also hope that our study can motivate further comparison between double-loop and single-loop algorithms in bilevel optimization.

\section*{Acknowledgements}
The work was supported in part by the U.S. National Science Foundation under the grants ECCS-2113860, DMS-2134145 and CNS-2112471.

\bibliographystyle{abbrv}
\bibliography{neurips2021}

\newpage
\appendix

{\Large{\bf Supplementary Materials}}

\section{Hessian Vector Implementation} \label{sec:hessianvector}

In this section, we provide an algorithm (see \Cref{alg:hessianvector}) for computing the hypergradient estimator in \cref{eq:xgest_mrbio} in MRBO by {\bf Hessian vectors} rather than {\bf Hessians}, in order to reduce the memory and computational cost. 

\begin{algorithm}
\caption{Hessian Vector Implementation for Computing Hypergradient Estimator in \cref{eq:xgest_mrbio}}
\label{alg:hessianvector}
\begin{algorithmic}[1]
\STATE \textbf{Input:} Hessian Estimation Number $Q$, Samples $\mathcal{B}_x$, Hyperparameter $\eta$,
\STATE Compute $\nabla_xF(x,y;\mathcal{B}_F)$, $r_0=\nabla_yF(x,y;\mathcal{B}_F)$, $\nabla_yG(x,y;\mathcal{B}_G)$
\FOR {$q=0,1,\ldots,Q-1$}
\STATE $G_{q+1} = (y-\eta \nabla_yG(x,y;\mathcal{B}_{Q-q}))r_{q}$
\STATE $r_{q+1} = \partial(G_{q+1})/\partial y$ \qquad \text{note: } $\partial(G_{q+1})/\partial y=r_q-\eta \nabla_y^2G(x,y;\mathcal{B}_{Q-q})r_{q}$
\ENDFOR
\STATE $M_Q=\eta \sum_{q=0}^{Q}r_q$
\STATE Return $\nabla_xF(x,y;\mathcal{B}_F)-\partial(\nabla_yG(x,y;\mathcal{B}_G)M_Q)/\partial x$

\end{algorithmic}
\end{algorithm}

As shown in line 5 of \Cref{alg:hessianvector}, instead of updating $r_{q+1}=r_q-\eta \nabla_y^2G(x,y;\mathcal{B}_{Q-q})r_{q}$ by directly computing Hessian $\nabla_y^2G(x,y;\mathcal{B}_{Q-q})$, we choose to compute the Hessian-vector product via $r_{q+1} = \partial(G_{q+1})/\partial y$. A similar implementation is applied to compute the Jacobian vector $\partial(\nabla_yG(x,y;\mathcal{B}_G)M_Q)/\partial x$ in line 8. Note that both lines 5 and 8 can apply automatic differentiation function {\em torch.grad()} for easy implementation. In this way, we compute the hypergradient estimator in \cref{eq:xgest_mrbio} recursively (see lines 3-6 in \Cref{alg:hessianvector}) via Hessian-vector products without computing Hessian explicitly.

\section{Specifications of Experiments}\label{sec:extraexperiments}
We compare our proposed algorithms MRBO and VRBO with other benchmarks including stocBiO~\cite{ji2021bilevel}, reverse~\cite{franceschi2017forward}, AID-FP~\cite{grazzi2020iteration}, BSA~\cite{ghadimi2018approximation}, MSTSA~\cite{khanduri2021momentum}, STABLE~\cite{chen2021single} and SUSTAIN~\cite{khanduri2021near} on the hyper-cleaning problem~\cite{shaban2019truncated} with MNIST dataset~\cite{lecun1998gradient}. The formulation of data hyper-cleaning is given below:
\begin{align*}
   &\min_\lambda \mathbb{E} [\mathcal{L}_{\mathcal{V}}(\lambda, w^{\ast})]=\frac{1}{|S_{\mathcal{V}}|} \sum_{(x_i,y_i)\in S_{\mathcal{V}}} L_{CE}((w^{\ast})^Tx_i,y_i) \\
   &\text{s.t.\quad } w^{\ast} =\argmin_w \mathcal{L}(\lambda, w):=\frac{1}{|S_{\mathcal{T}}|}\sum_{(x_i,y_i)\in S_{\mathcal{T}}}\sigma(\lambda_i)L_{CE}(w^Tx_i,y_i)+C\|w\|^2,
\end{align*}
where $L_{CE}$ denotes the cross-entropy loss, $S_{\mathcal{T}}$ and $S_{\mathcal{V}}$ denote the training data and the validation data, respectively, $\lambda=\{\lambda_i\}_{i\in \mathcal{S}_\mathcal{T}}$ and $C$ are the regularization parameters, and $\sigma(\cdot)$ denotes the sigmoid function. In experiment, we set $C=0.001$ and fix the size of the training data $\mathcal{S}_{\mathcal{V}}$ and validation data $S_{\mathcal{T}}$ as 20000 and 5000, respectively. Furthermore, we use 10000 images for testing, which follows the setting in \cite{ji2021bilevel}. We use the Hessian-vector based algorithm (\Cref{alg:hessianvector}) for computing the hypergradient estimator, where we set $Q=3$ and $\eta=0.5$. For stochastic algorithms including MRBO, VRBO, stocBiO, we set the batchsize to be 1000 for both training and validation procedures. For VRBO, we set the inner batchsize to be 500 and the period $q$ to be 3. For the double-loop algorithms, we fine tune the number of inner-loop steps and set it to be 200 for the stocBiO, AID-FP, BSA and reverse algorithms for the best performance, and set it to be 20 for VRBO for the best performance. To set the outer-loop and inner-loop stepsizes, we use the training loss as the metric and apply the standard grid search with the stepsizes $\lambda$, $\gamma$, $\alpha$ and $\beta$ all chosen from the interval [1e-3,1]. We then select those that yield the best convergence performance. Thus, we set 0.1 as the stepsize for all algorithms except SUSTAIN and STABLE. For SUSTAIN, the inner-loop stepsize is set to be 0.03 and outer-loop stepsize is set to be 0.1, and for STABLE, inner-loop and outer-loop stepsizes are set to be 0.01 and 1e-10, respectively, because these algorithms are not stable with larger stepsizes. Our experimental implementations are based on the implementation of stocBiO in \cite{ji2021bilevel}, which is under MIT License. Futhermore, all results are repeated with 5 random seeds and we use iMac with 3.8GHz Quad-Core Intel Core i5 CPU and 32 GB 2400 MHz DDR4 for training without the requirement of GPU. However, our code supports GPU cluster training.

\subsection{Additional Experiments of Hyper-cleaning}
In this subsection, we include extra experiments to further validate our theoretical results and understand the VRBO algorithm.

\begin{figure}[h]
	\centering    
	\subfigure[Noise rate $p=0.15$]{\includegraphics[width=60.0mm]{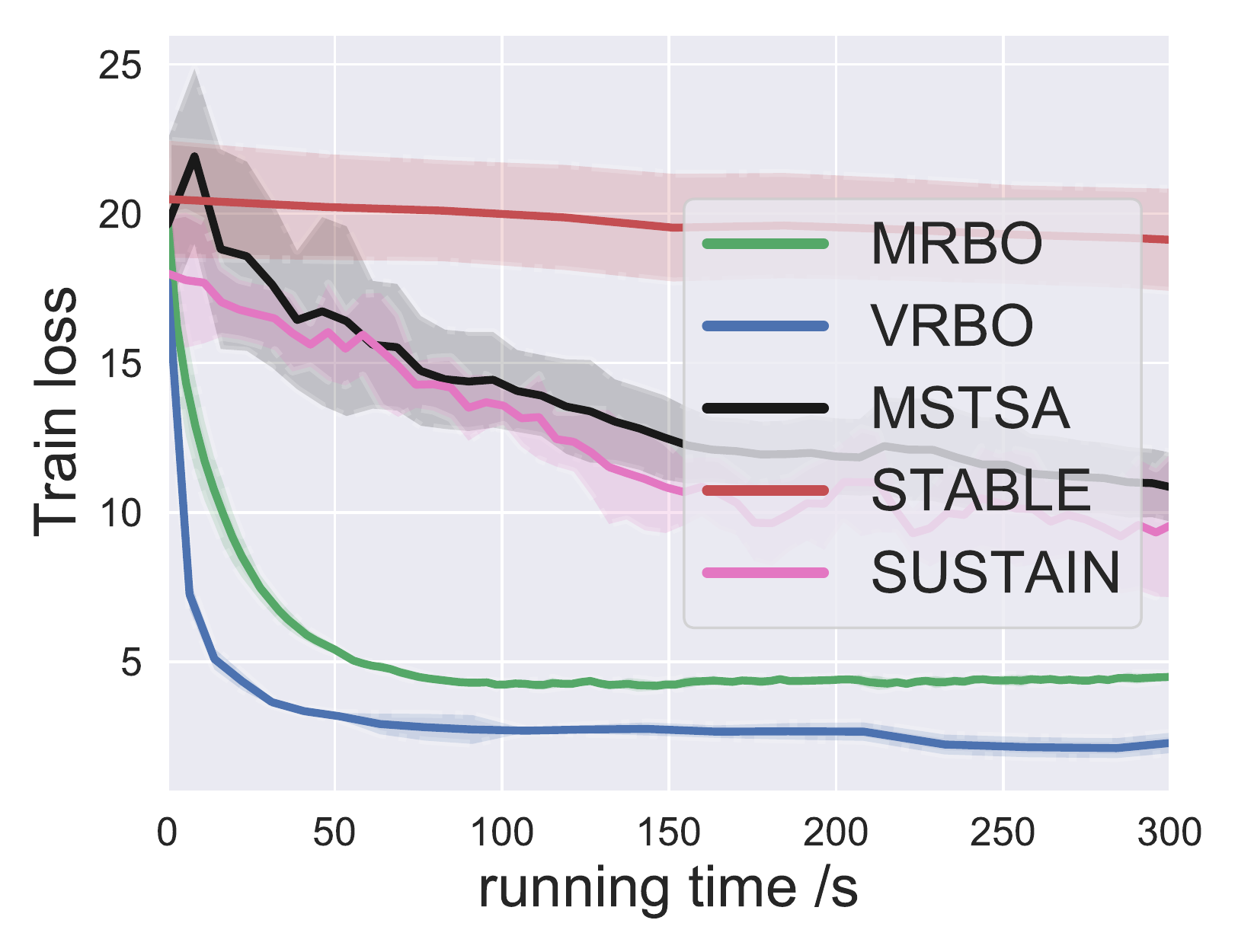}}
    	\vspace{-1mm}
	\caption{training loss v.s. running time.}\label{fig:extraresults}
\end{figure}

In~\Cref{fig:extraresults}, we compare our algorithms MRBO and VRBO with three momentum-based algorithms, i.e., MSTAS, STABLE, and SUSTAIN, under the noise rate $p=0.15$, which is a scenario in addition to the experiment provided in \Cref{fig:mainresults} (c) of the main part under the noise rate $p=0.1$. It is clear that our algorithms MRBO and VRBO achieve the lowest training loss and converge fastest among all momentum-based algorithms.

\begin{figure}[ht]
    \centering
     \subfigure[Noise rate $p=0.1$]{\includegraphics[width=60.0mm]{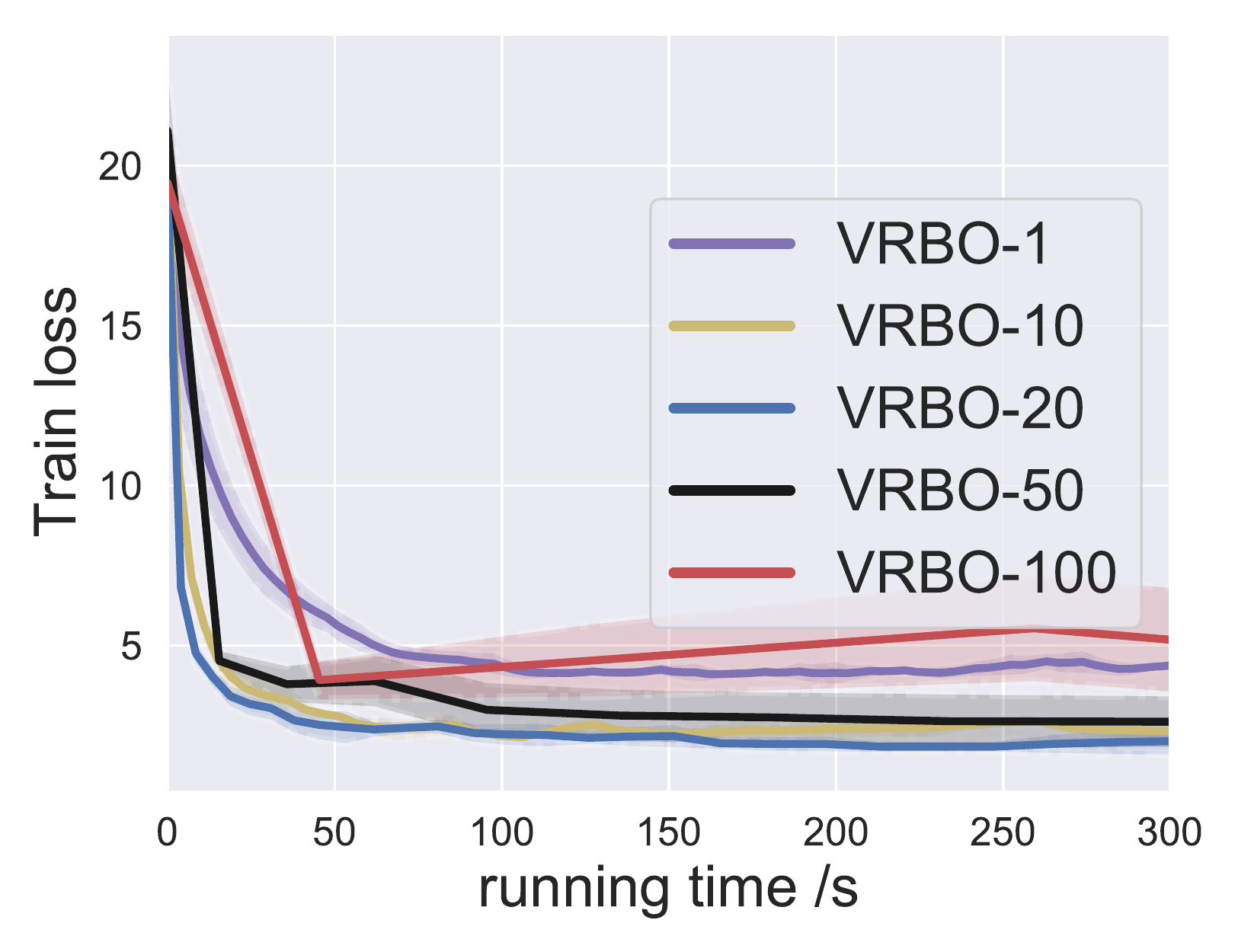}}
    \subfigure[Noise rate $p=0.15$]{ \includegraphics[width=60.0mm]{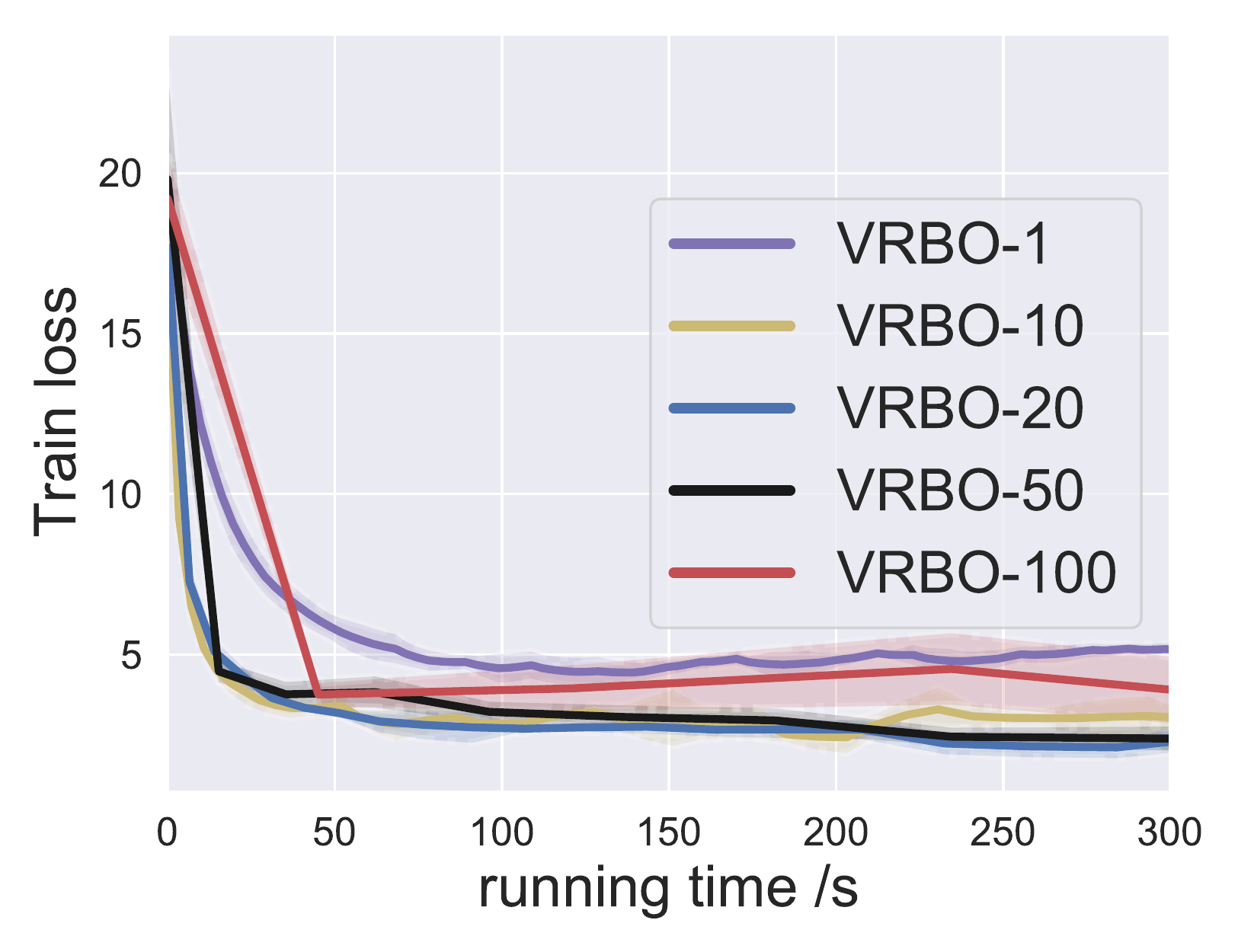}}
    \caption{training loss v.s. running time.}
    \label{fig:vrbo}
\end{figure}

The next experiment focuses on the double-loop algorithm VRBO and studies how the number $m$ of inner-loop steps affects its performance. In~\Cref{fig:vrbo} (a) and (b), we compare VRBO among five choices of $m\in \{1,10,20,50,100\}$, where VRBO-$m$ in the legend indicates that the inner-loop of VRBO takes $m$ steps. It can be observed that as $m$ increases from $1$, VRBO becomes more stable and achieves lower training loss until $m=20$. Beyond this point, as $m$ further increases, the performance of VRBO becomes worse with higher final training loss and lower stability. This can be explained by two reasons: (i) the accuracy of the inner-loop output and (ii) the accuracy of the variance-reduced gradient estimator. By the formulation of bilevel optimization, at each outer-loop step $k$, it is desirable that the inner loop obtains $y_k$ as close as possible to the optimal point $y^*(x_k)=\argmin_{y} g(x_k,y)$. Hence, taking more inner-loop steps (i.e., as $m$ increases) helps to obtain more accurate $y_k$. Further, increasing $m$ allows the large-batch gradient estimator to benefit more steps of gradient estimators in the inner loop via variance reduction, and hence improves the computational efficiency. Both reasons explain that the overall performance of VRBO gets better as $m$ increases from $m=1$ to $m=20$. On the other hand, when $m$ is large enough (i.e., $m=20$ in our plots), the inner-loop can already provide a sufficiently accurate $y_k$. Then further increasing $m$ will cause unnecessary inner-loop iterations and hurt the computational efficiency. Moreover, larger $m$ causes the variance-reduced gradient estimators in the later stage of the inner loop becomes less accurate. Thus, the overall convergence of VRBO becomes slower and less stable.

\begin{figure}[ht]
	\centering
	\subfigure[Noise rate $p=0.1$]{\includegraphics[width=60.0mm]{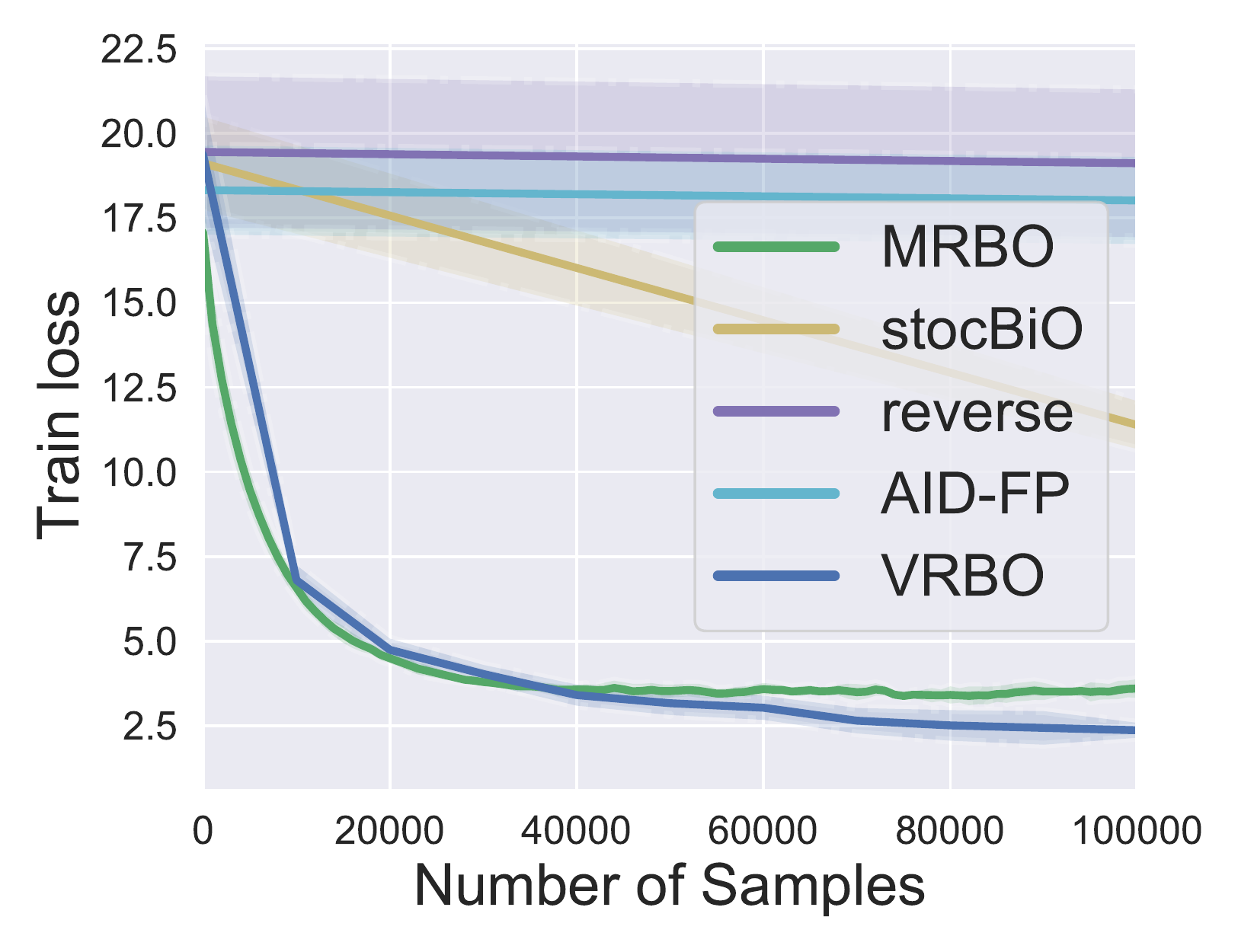}}
	\subfigure[Noise rate $p=0.15$]{ \includegraphics[width=60.0mm]{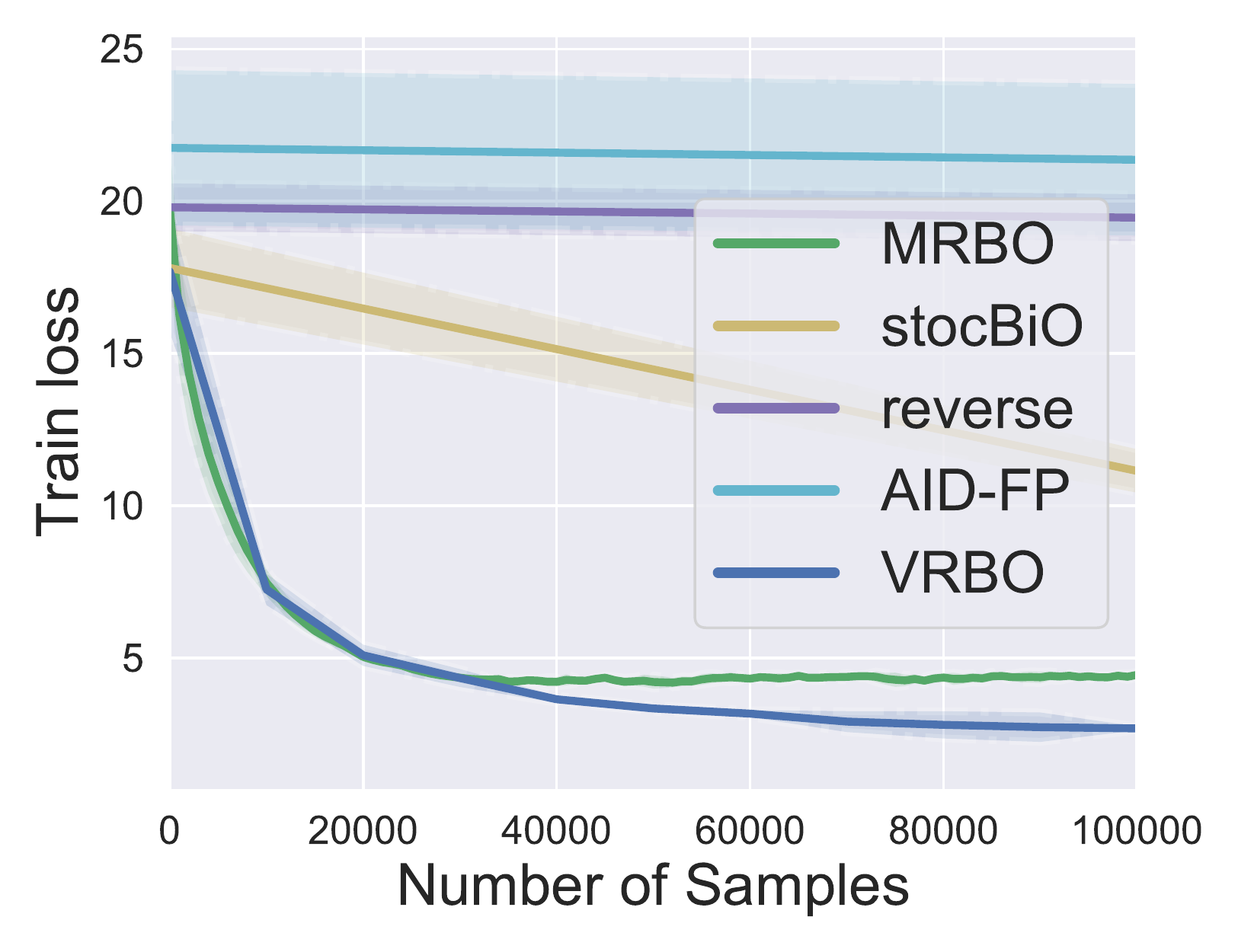}}
	\caption{training loss v.s. number of samples.}
	\label{fig:xdata}
\end{figure}

In \Cref{fig:xdata}, we further compare our algorithms with other \textbf{batch-sample} based algorithms in terms of the training performance versus the number of samples required. It can be seen that MRBO and VRBO are much more sample efficient in training compared with stocBiO and the GD-based algorithms reverse and AID-FP. 

\begin{figure}[ht]
	\centering
	\subfigure[Noise rate $p=0.1$]{\includegraphics[width=60.0mm]{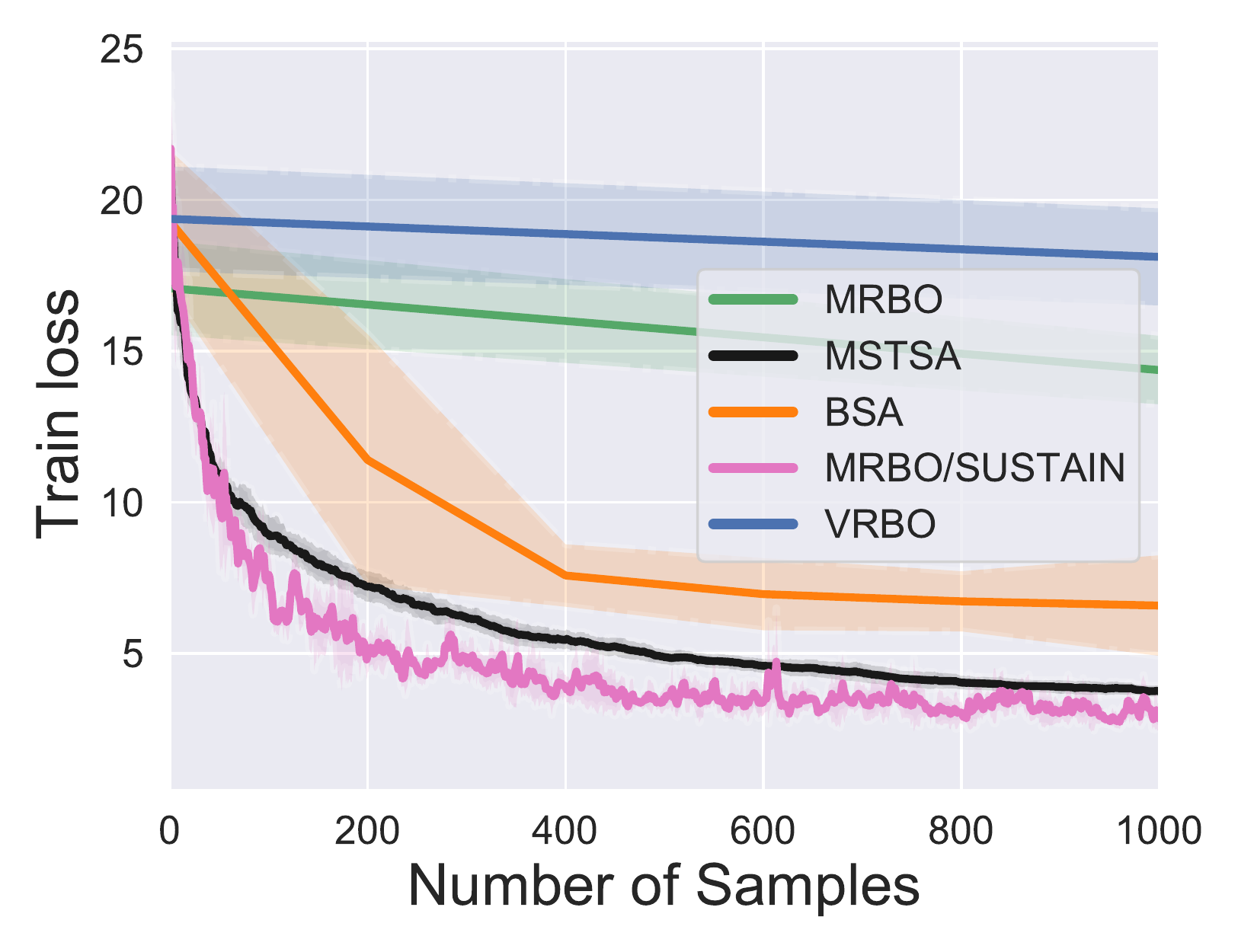}}
	\subfigure[Noise rate $p=0.15$]{ \includegraphics[width=60.0mm]{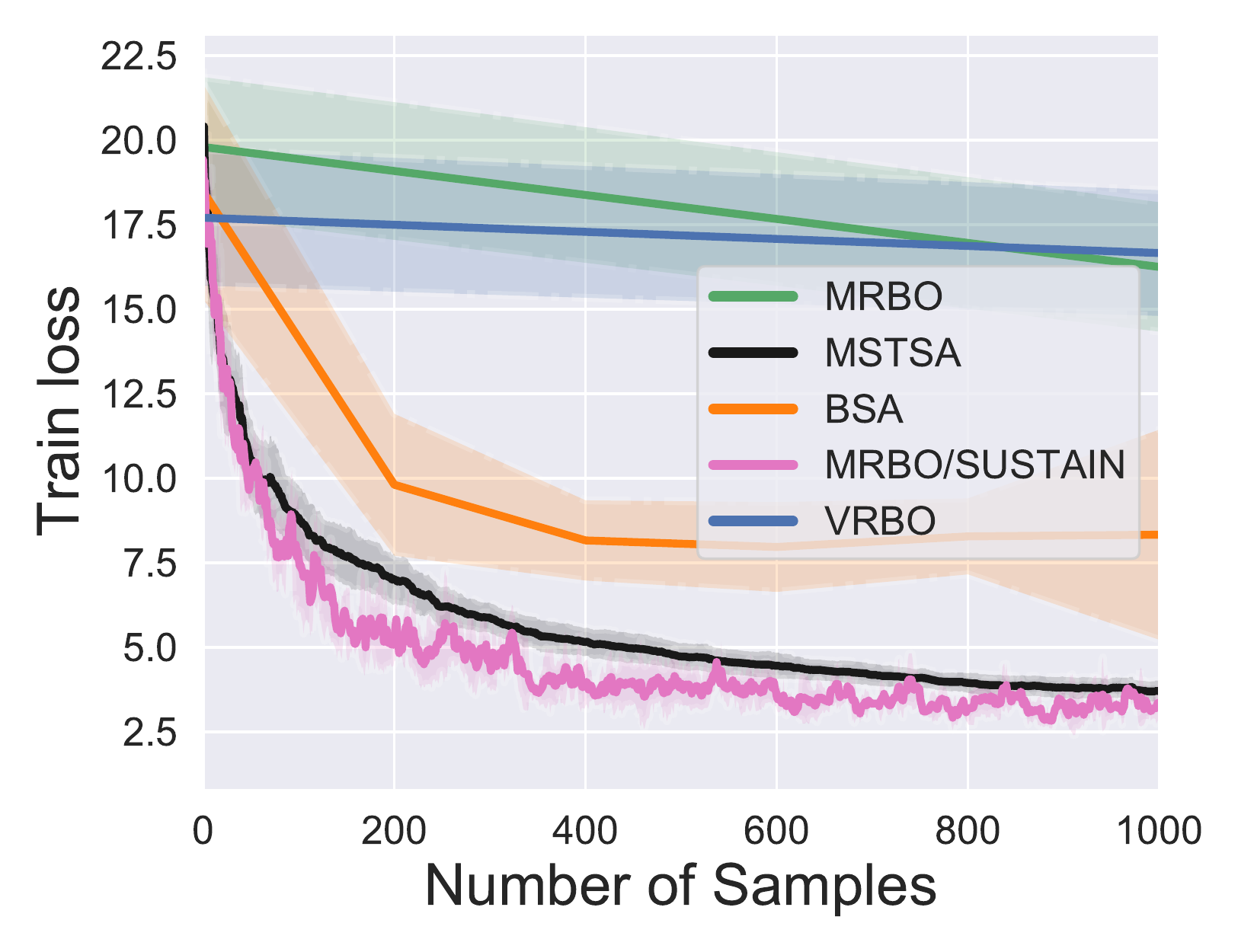}}
	\caption{training loss v.s. number of samples.}
	\label{fig:xdata_single}
\end{figure}

We also compare our algorithms with other \textbf{single-sample} based algorithms w.r.t.~the number of samples in \Cref{fig:xdata_single}. It can be seen that single-sample based algorithms are more sample efficient than MRBO and VRBO. This is because single-sample based algorithms update each parameter using a single sample, whereas batch-sample based algorithms update each parameter using a batch of samples. As a result, single-sample based algorithms enable a larger parameter update per sample, and hence achieve a higher sample efficiency. It is worthy to mention that our MRBO can be implemented in a single-sample fashion, which then becomes the same as the concurrently proposed algorithm SUSTAIN. However, compared to the sample efficiency, we believe that the execution time (under the same computing resource) is a more reasonable measure of the computational efficiency of bilevel algorithms. This is because the minibatch computation are more preferred and efficient than the single-sample computation in existing deep learning platforms such as PyTorch. Thus, as demonstrated in our \Cref{fig:mainresults}, batch-sample based algorithms converge much faster than single-sample based algorithms w.r.t.~running time. 

\subsection{Experiments of Logistic Regression}
We further conduct the experiment on the logistic regression problem over the 20 Newsgroup dataset~\cite{grazzi2020iteration}. The objective function is given by:
\begin{figure}[ht]
	\centering
	\subfigure[Test Accuracy v.s. Running Time]{\includegraphics[width=60.0mm]{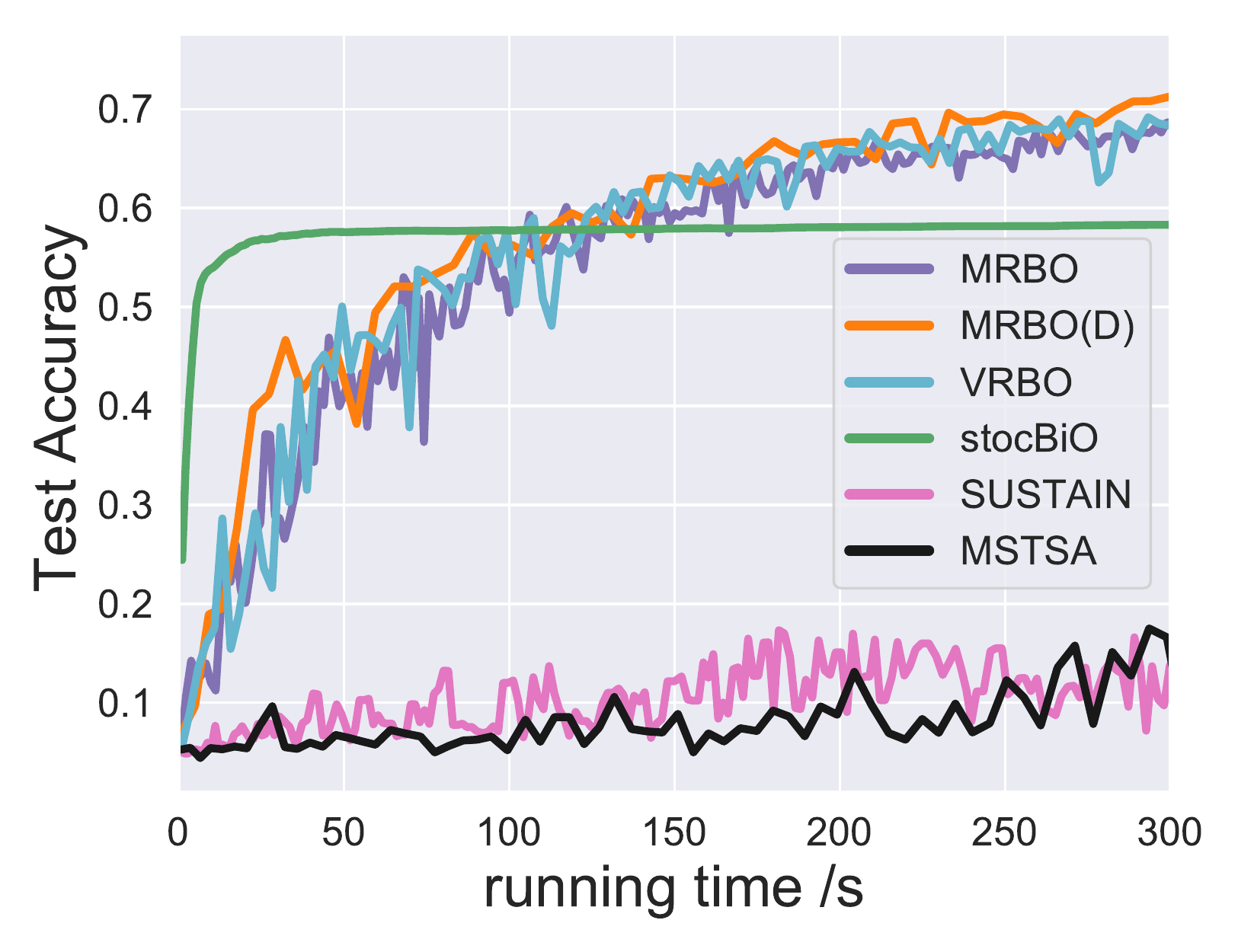}}
	\subfigure[Test Loss v.s. Running Time]{ \includegraphics[width=60.0mm]{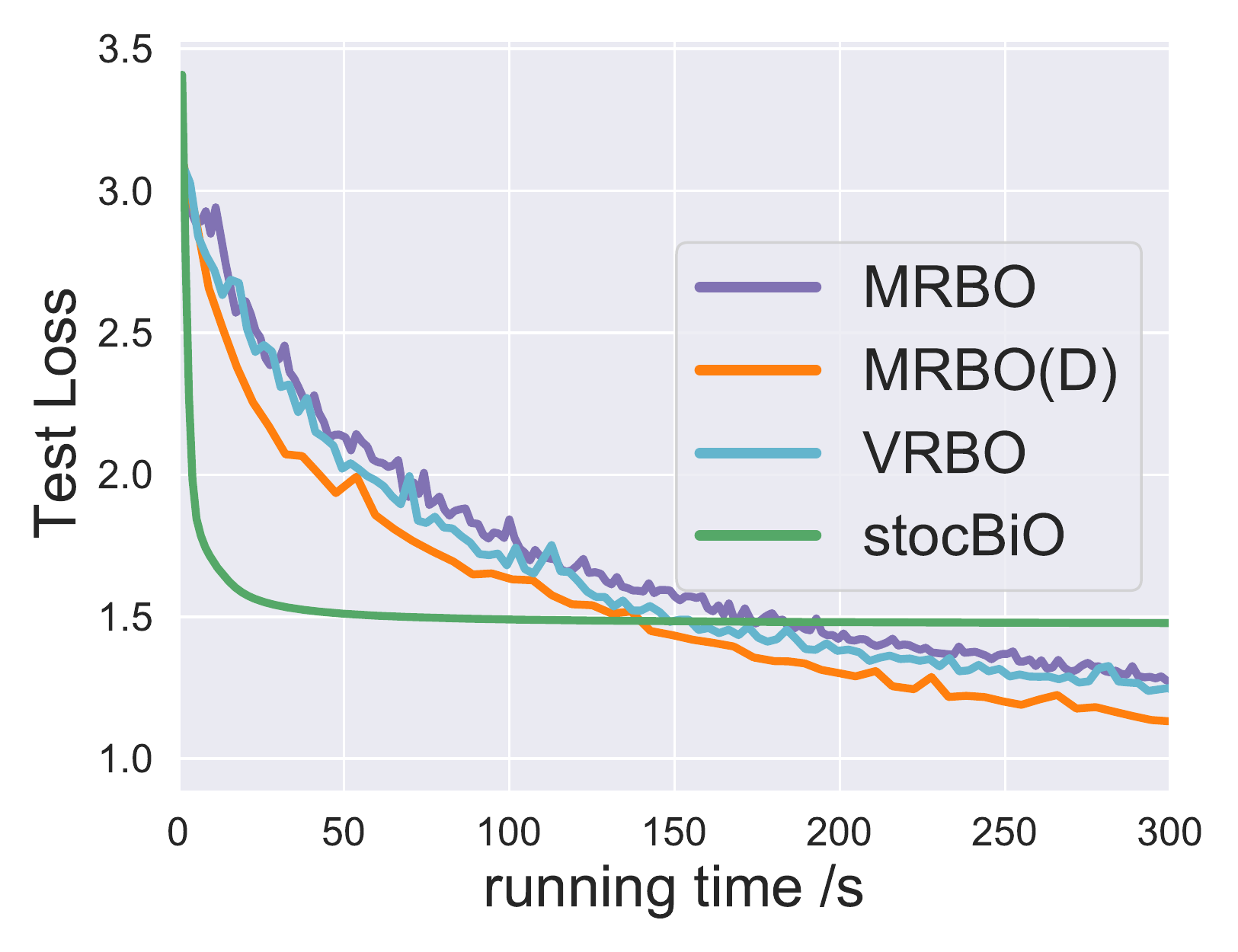}}
	\caption{test accuracy or test loss v.s. running time (Batchsize=100)}
	\label{fig:logistric}
\end{figure}

\begin{figure}[ht]
	\centering
	\subfigure[Test Accuracy v.s. Running Time]{\includegraphics[width=60.0mm]{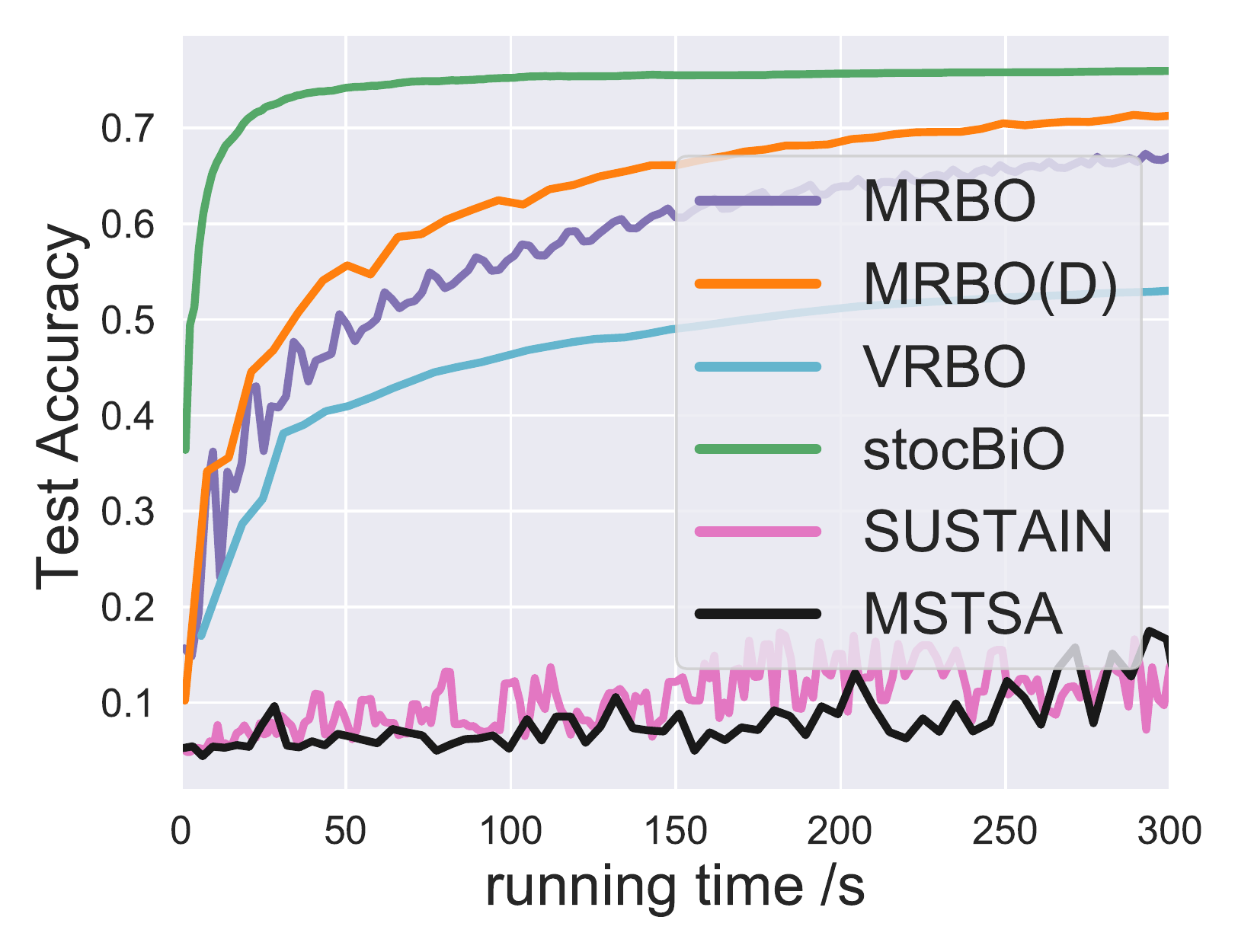}}
	\subfigure[Test Loss v.s. Running Time]{ \includegraphics[width=60.0mm]{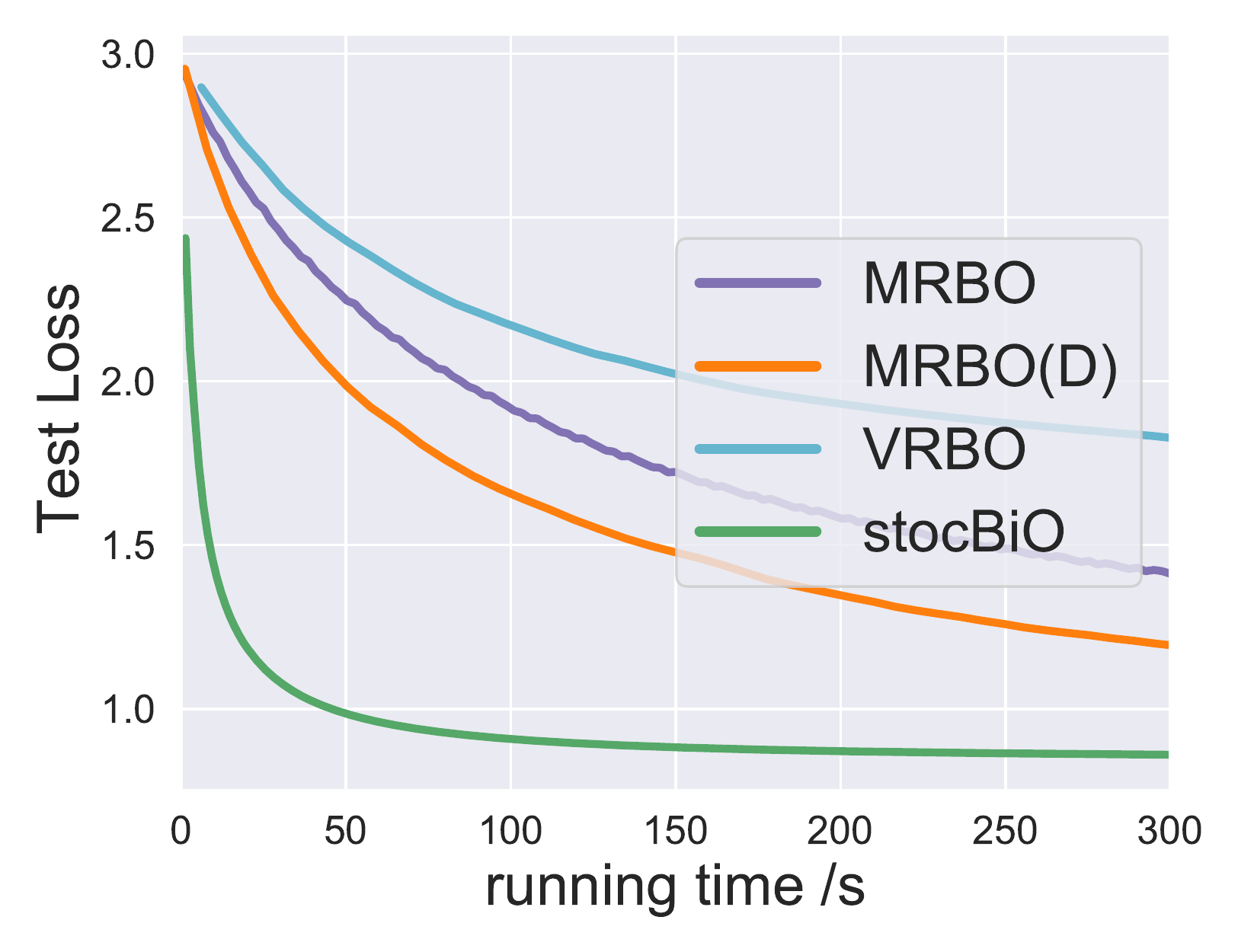}}
	\caption{test accuracy or test loss v.s. running time (Batchsize=1000)}
	\label{fig:logistric_1000}
\end{figure}

\begin{align*}
   &\min_\lambda \mathbb{E} [\mathcal{L}_{\mathcal{V}}(\lambda, w^{\ast})]=\frac{1}{|S_{\mathcal{V}}|} \sum_{(x_i,y_i)\in S_{\mathcal{V}}} L_{CE}((w^{\ast})^Tx_i,y_i) \\
   &\text{s.t.\quad } w^{\ast} =\argmin_w \mathcal{L}(\lambda, w):=\frac{1}{|S_{\mathcal{T}}|}\sum_{(x_i,y_i)\in S_{\mathcal{T}}}L_{CE}(w^Tx_i,y_i)+\frac{1}{cp}\sum_{i=1}^{c}\sum_{j=1}^{p}\exp(\lambda_j)w_{ij}^2,
\end{align*}
where $L_{CE}$ denotes the cross-entropy loss, $S_{\mathcal{T}}$ and $S_{\mathcal{V}}$ denote the training and validation datasets, respectively. In the experiment, we follow the setting for stocBiO in~\cite{ji2021bilevel} and set $\eta=0.5$ and $Q=10$ for the hypergradient estimation. Besides, we apply the standard grid search for the inner- and outer-loop stepsizes for all algorithms. Thus, we set inner- and outer-loop stepsizes as 100 for stocBiO, inner- and outer-loop stepsizes as 30 for MRBO, VRBO, SUSTAIN and MSTSA. Following the setting in stocBiO, we set inner-loop steps as 10 for stocBiO. For VRBO, we set the period $q$ as 2 and inner-loop steps as 3 for the best performance. We also conduct MRBO in a double loop fashion and call it as MRBO(D), where we apply the inner update procedure 10 times per epoch.

In \Cref{fig:logistric}, we set the batchsize of all stochastic algorithms to 100. It can be seen that although stocBiO achieves the fastest initial convergence rate, both MRBO and VRBO reach a higher accuracy than stocBiO due to more accurate hypergradient estimation. It can be also seen that our  double-loop MRBO(D) achieves the highest accuracy, whereas single-loop SUSTAIN and MSTSA algorithms do not converge well. This demonstrates the advantage of double-loop updates over single-loop updates. In \Cref{fig:logistric_1000}, we choose a larger batchsize of 1000 for all algorithms. We note that double-loop algorithms stocBiO and MRBO(D) still outperform other single-loop algorithms significantly, and stocBiO achieves the best test accuracy due to a more accurate gradient estimation.

\section{Proof of \Cref{thm:mrbio}} \label{sec:proofmrbo}
\subsection{Proof of Supporting Lemmas (Propositions \ref{prop:mrbio_epi} and \ref{prop:mrbio_phi} Correspond to Lemmas \ref{lem:G} and \ref{lem:ephi})}
For notation simplification, we define the following:
\begin{align}\label{eq:Vqk}
    V_{Qk} = \eta \sum_{q=-1}^{Q-1} \prod_{j=Q-q}^{Q} (I-\eta \nabla_y^2G(x_k, y_k; \mathcal{B}_j))\nabla_yF(x_k,y_k;\mathcal{B}_{F}).
\end{align}
Firstly, we characterize the variance of $V_{Qk}$ in the following lemma.
\begin{lemma} \label{lem:Vqk}
Suppose Assumptions \ref{ass:lip}, \ref{ass:var} hold. Let $\eta < \frac{1}{L}$. Then, we have
\begin{align*}
    \mathbb{E}\|V_{Qk}&-\mathbb{E}[V_{Qk}]\|^2 \leq  \frac{2\eta^2M^2(Q+1)^2}{S} + \frac{M^2(Q+2)(Q+1)^2\eta^2\sigma^2}{2S},
\end{align*}
where $V_{Qk}$ is defined in \cref{eq:Vqk}.

\begin{proof}
Based on the form of $V_{Qk}$, we have
      \begin{align*}
        \mathbb{E}\|&V_{Qk}-\mathbb{E}[V_{Qk}]\|^2\\
        \overset{(i)}=& \eta^2 \mathbb{E} \bigg\|\sum_{q=0}^{Q}(I-\eta \nabla_y^2g(x_k,y_k))^q \nabla_yf(x_k,y_k)
        -\sum_{q=0}^{Q}(I-\eta \nabla_y^2g(x_k,y_k))^q \nabla_yF(x_k,y_k;\mathcal{B}_{F})\\
        &+\sum_{q=0}^{Q}(I-\eta \nabla_y^2g(x_k,y_k))^q \nabla_yF(x_k,y_k;\mathcal{B}_{F})\\
        &-\sum_{q=-1}^{Q-1} \prod_{j=Q-q}^{Q} (I-\eta \nabla_y^2G(x_k, y_k; \mathcal{B}_j))\nabla_yF(x_k,y_k;\mathcal{B}_{F})\bigg\|^2 \\
         \overset{(ii)}\leq & 2\eta^2M^2\mathbb{E}\bigg\|\sum_{q=0}^{Q}(I-\eta \nabla_y^2g(x_k,y_k))^q-\sum_{q=-1}^{Q-1} \prod_{j=Q-q}^{Q} (I-\eta \nabla_y^2G(x_k, y_k; \mathcal{B}_j))\bigg\|^2\\
         &+\frac{2\eta^2M^2(Q+1)^2}{S}\\
         \leq & 2\eta^2M^2(Q+1)\mathbb{E}\sum_{q=0}^Q\bigg\|(I-\eta \nabla_y^2g(x_k,y_k))^q-\prod_{j=Q+1-q}^Q(I-\eta \nabla_y^2G(x_k, y_k; \mathcal{B}_j))\bigg\|^2 \\
         &+\frac{2\eta^2M^2(Q+1)^2}{S} \\
        \leq & \frac{2\eta^2M^2(Q+1)^2}{S}+2\eta^2M^2(Q+1)\mathbb{E}\sum_{q=0}^Q\bigg\|(I-\eta \nabla_y^2g(x_k,y_k))^q\\
       &-(I-\eta \nabla_y^2g(x_k,y_k))^{q-1}(I-\eta \nabla_y^2G(x_k, y_k; \mathcal{B}_Q))\\
       &+(I-\eta \nabla_y^2g(x_k,y_k))^{q-1}(I-\eta \nabla_y^2G(x_k, y_k; \mathcal{B}_Q))-\prod_{j=Q+1-q}^Q(I-\eta \nabla_y^2G(x_k, y_k; \mathcal{B}_j))\bigg\|^2 \\
         \overset{(iii)}\leq & 2\eta^2M^2(Q+1)\mathbb{E}\sum_{q=0}^Q (q+1)\bigg\|(I-\eta \nabla_y^2G(x_k,y_k;\mathcal{B}_q)-(I-\eta\nabla_y^2g(x_k,y_k))\bigg\|^2  \\
         &+\frac{2\eta^2M^2(Q+1)^2}{S}\\
         \overset{(iv)}\leq & \frac{2\eta^2M^2(Q+1)^2}{S} + \frac{M^2(Q+2)(Q+1)^2\eta^4\sigma^2}{S},
    \end{align*}
    where $(i)$ follows from the fact that $\mathbb{E}[V_{Qk}]=\eta \sum_{q=0}^{Q}(I-\eta \nabla_y^2g(x_k,y_k))^q\nabla_yf(x_k,y_k)$, $(ii)$ follows from \Cref{ass:lip} and the fact that $\|I-\eta \nabla_y^2g(x_k,y_k)\|\leq 1$, $(iii)$ follows from the facts that $\|I-\eta \nabla_y^2G(x_k,y_k;\mathcal{B}_j)\|\leq 1$ and $\|I-\eta \nabla_y^2g(x_k,y_k)\|\leq 1$, and $(iv)$ follows from Assumptions \ref{ass:lip} and \ref{ass:var}. Then, the proof is complete.
\end{proof}
\end{lemma}

Futhermore, we characterize the Lipschitz property of $V_{Qk}$ in the following lemma.
\begin{lemma} Suppose Assumption \ref{ass:lip} holds. Let $\eta < \frac{1}{L}$. Then, we have
    \begin{align}
        \|V_{Qk}-V_{Q(k-1)}\|^2 \leq \left(\frac{M^2 Q^2(Q+1)^2\eta^4 \rho^2}{2}+2\eta^2L^2(Q+1)^2\right) \|z_k-z_{k-1}\|^2,
    \end{align}
    where $V_{Qk}$ is defined in \cref{eq:Vqk}.
    \begin{proof}
    Based on the form of $V_{Qk}$, we have
    \begin{align*}
          \|V_{Qk}&-V_{Q(k-1)}\|^2 \\
          \overset{(i)}\leq& \eta^2 \bigg\|\sum_{q=-1}^{Q-1} \prod_{j=Q-q}^{Q} (I-\eta \nabla_y^2G(x_k, y_k; \mathcal{B}_j))\nabla_yF(x_k,y_k;\mathcal{B}_{F})\\
          &-\sum_{q=-1}^{Q-1} \prod_{j=Q-q}^{Q} (I-\eta \nabla_y^2G(x_{k-1}, y_{k-1}; \mathcal{B}_j))\nabla_yF(x_{k-1},y_{k-1};\mathcal{B}_{F})\bigg\|^2 \\
          \leq & \eta^2 \bigg( 2\|\nabla_yF(x_k,y_k;\mathcal{B}_{F})\|^2\bigg\|\sum_{q=-1}^{Q-1} \prod_{j=Q-q}^{Q} (I-\eta \nabla_y^2G(x_k, y_k; \mathcal{B}_j))\\
          &-\sum_{q=-1}^{Q-1} \prod_{j=Q-q}^{Q} (I-\eta \nabla_y^2G(x_{k-1}, y_{k-1}; \mathcal{B}_j))\bigg\|^2 \bigg) \\
          &+2\|\nabla_yF(x_k,y_k;\mathcal{B}_{F})-\nabla_yF(x_{k-1},y_{k-1};\mathcal{B}_{F})\|^2\bigg\|\sum_{q=-1}^{Q-1} \prod_{j=Q-q}^{Q} (I-\eta \nabla_y^2G(x_{k-1}, y_{k-1}; \mathcal{B}_j))\bigg\|^2 \\
          \overset{(ii)}\leq & \eta^2\bigg(2M^2\bigg\|\sum_{q=-1}^{Q-1} \prod_{j=Q-q}^{Q} (I-\eta \nabla_y^2G(x_k, y_k; \mathcal{B}_j))-\sum_{q=-1}^{Q-1} \prod_{j=Q-q}^{Q} (I-\eta \nabla_y^2G(x_{k-1}, y_{k-1}; \mathcal{B}_j))\bigg\|^2 \\
          &+2L^2(Q+1)^2\|z_k-z_{k-1}\|^2\bigg) \\
          \overset{(iii)}\leq & 2\eta^2M^2 \left(\sum_{q=0}^{Q} \bigg\|\prod_{j=Q+1-q}^{Q} (I-\eta \nabla_y^2G(x_{k}, y_{k}; \mathcal{B}_j))-\prod_{j=Q+1-q}^{Q} (I-\eta \nabla_y^2G(x_{k-1}, y_{k-1}; \mathcal{B}_j)) \bigg\|\right)^2\\
          &+2\eta^2L^2(Q+1)^2\|z_k-z_{k-1}\|^2 \\
          \overset{(iv)}\leq & 2\eta^2M^2\left(\sum_{q=0}^{Q} q\eta \rho \|z_k-z_{k-1}\|\right)^2+2\eta^2L^2(Q+1)^2\|z_k-z_{k-1}\|^2 \\
          \leq & \left(2\eta^2M^2 (\frac{Q(Q+1)}{2})^2\eta^2 \rho^2+2\eta^2L^2(Q+1)^2\right) \|z_k-z_{k-1}\|^2,
    \end{align*}
    where $(i)$ follows from the definition of $V_{Qk}$,  $(ii)$ follows from \Cref{ass:lip} and the fact that $\|I-\eta \nabla_y^2G(x_k,y_k;\mathcal{B}_j)\|\leq 1$, $(iii)$ follows from Jensen's inequality and $(iv)$ follows because $\|I-\eta \nabla_y^2G(x_k,y_k;\mathcal{B}_j)\|\leq 1$ and from \Cref{ass:lip}. Then, the proof is complete.
    \end{proof}
\end{lemma}

Then, we characterize the Lipschtiz property of $\widehat{\nabla}\Phi(x_k;\mathcal{B}_x)$ defined in \cref{eq:xgest_mrbio} in the following lemma.
\begin{lemma}\label{lem:LQ}
Suppose Assumptions \ref{ass:lip} holds. Let $\eta < \frac{1}{L}$ and $z=(x,y)$. Then, for $\widehat{\nabla}\Phi(x_k;\mathcal{B}_x)$ defined in \cref{eq:xgest_mrbio}, we have
\begin{align}
    \|\widehat{\nabla}\Phi(x_k;\mathcal{B}_x)-\widehat{\nabla}\Phi(x_{k-1};\mathcal{B}_x)\|^2 \leq L^2_Q\|z_k-z_{k-1}\|^2,
\end{align}
where $L_Q^2=2L^2+4\tau^2\eta^2M^2(Q+1)^2+8L^4\eta^2(Q+1)^2+2L^2\eta^4M^2\rho^2Q^2(Q+1)^2.$
    \begin{proof}
    Based on the form of $\widehat{\nabla}\Phi(x_k;\mathcal{B}_x)$, we have
    \begin{align*}
        \|\widehat{\nabla}&\Phi(x_k;\mathcal{B}_x)-\widehat{\nabla}\Phi(x_{k-1};\mathcal{B}_x)\|^2\\
        \leq& 2\|\nabla_xF(z_k;\mathcal{B}_F)-\nabla_xF(z_{k-1};\mathcal{B}_F)\|^2+2\|\nabla_x\nabla_yG(x_k,y_k;\mathcal{B}_G)V_{Qk}\\
        &-\nabla_x\nabla_yG(x_{k-1},y_{k-1};\mathcal{B}_G)V_{Q(k-1)}\|^2 \\
        \overset{(i)}\leq &2L^2\|z_k-z_{k-1}\|^2+4\|\nabla_x\nabla_yG(x_k,y_k;\mathcal{B}_G)(V_{Qk}-V_{Q(k-1)})\|^2\\
        &+4\|(\nabla_x\nabla_yG(x_k,y_k;\mathcal{B}_G)-\nabla_x\nabla_yG(x_{k-1},y_{k-1};\mathcal{B}_G))V_{Q(k-1)}\|^2 \\
        \overset{(ii)}\leq & 2L^2\|z_k-z_{k-1}\|^2+4L^2\|V_{Qk}-V_{Q(k-1)}\|^2+4\tau^2\|z_k-z_{k-1}\|^2\|V_{Q(k-1)}\|^2 \\
        \overset{(iii)}\leq & (2L^2+4\tau^2\eta^2M^2(Q+1)^2 )\|z_k-z_{k-1}\|^2 +4L^2\|V_{Qk}-V_{Q(k-1)}\|^2\\
        \overset{(iv)}\leq & (2L^2+4\tau^2\eta^2M^2(Q+1)^2+8L^4\eta^2(Q+1)^2+2L^2\eta^4M^2\rho^2Q^2(Q+1)^2)\|z_k-z_{k-1}\|^2,
    \end{align*}
    where $(i)$ and $(ii)$ follow from \Cref{ass:lip}, $(iii)$ follows from the fact $\|V_{Qk}\|\leq\eta M(Q+1)$, and $(iv)$ follows from \Cref{lem:Vqk}. Then, the proof is complete.
    \end{proof}
\end{lemma}

\begin{lemma}[{\bf Restatement of~\Cref{prop:mrbio_epi}}]\label{lem:G}
Suppose Assumptions \ref{ass:conv}, \ref{ass:lip} and \ref{ass:var} hold. Let $\eta < \frac{1}{L}$. Then, we have
\begin{align}
    \label{eq:G}
        \mathbb{E} \|\widehat{\nabla}\Phi(x_k;\mathcal{B}_x)-\overline{\nabla}\Phi(x_k)\|^2 \leq G^2,
    \end{align}
    where $G=\frac{2M^2}{S}+\frac{12M^2L^2\eta^2(Q+1)^2}{S}+\frac{4M^2L^2(Q+2)(Q+1)^2\eta^4\sigma^2}{S}, \widehat{\nabla}\Phi(x_k;\mathcal{B}_x)$ is defined in \cref{eq:xgest_mrbio} and $\overline{\nabla}\Phi(x_k)$ is defined in \cref{eq:hypergradientest}. Further, for the iterative update of line 8 in \Cref{alg:mrbio},  we let $\bar{\epsilon}_k=v_k-\overline{\nabla}\Phi(x_k)$. Then, we have  
	\begin{align}
	\label{eq:lossx}
	\mathbb{E} \| \bar{\epsilon}_k \|^2 \leq \mathbb{E} [&2\alpha_k^2 G^2 + 2(1-\alpha_k)^2L_Q^2 \| x_k-x_{k-1} \|^2\nonumber \\
	 &+ 2(1-\alpha_k)^2L_Q^2 \| y_k-y_{k-1} \|^2 + (1-\alpha_k)^2 \|\bar{\epsilon}_{k-1} \|^2],
	\end{align}
where $L_Q^2$ is defined in \Cref{lem:LQ}.

\begin{proof}
We first prove \cref{eq:G}. Based on the forms of $\widehat{\nabla}\Phi(x_k;\mathcal{B}_x)$ and $\overline{\nabla}\Phi(x_k)$, we have
    \begin{align*}
        \mathbb{E} \|\widehat{\nabla}&\Phi(x_k;\mathcal{B}_x)-\overline{\nabla}\Phi(x_k)\|^2 \\
        \overset{(i)}\leq& 2\mathbb{E}\|\nabla_xF(x_k,y_k;\mathcal{B}_F)-\nabla_xf(x_k,y_k)\|^2\\
        &+2\mathbb{E}\|\nabla_x\nabla_yG(x_k,y_k;\mathcal{B}_G)V_{Qk}-\nabla_x\nabla_yg(x_k,y_k)\mathbb{E}[V_{Qk}]\|^2 \\
        \overset{(ii)}\leq & \frac{2M^2}{S}+2\mathbb{E}\|\nabla_x\nabla_yG(x_k,y_k;\mathcal{B}_G)V_{Qk}-\nabla_x\nabla_yG(x_k,y_k;\mathcal{B}_G)\mathbb{E}[V_{Qk}] \\
        &+\nabla_x\nabla_yG(x_k,y_k;\mathcal{B}_G)\mathbb{E}[V_{Qk}]-\nabla_x\nabla_yg(x_k,y_k)\mathbb{E}[V_{Qk}]\|^2\\
        \leq& \frac{2M^2}{S}+4\mathbb{E}\|\nabla_x\nabla_yG(x_k,y_k;\mathcal{B}_G)\|^2\mathbb{E}\|V_{Qk}-\mathbb{E}[V_{Qk}]\|^2\\
        &+4\mathbb{E}\|\nabla_x\nabla_yG(x_k,y_k;\mathcal{B}_G)-\nabla_x\nabla_yg(x_k,y_k)\|^2\|\mathbb{E}[V_{Qk}]\|^2\\
        \overset{(iii)} \leq &  \frac{2M^2}{S} +4L^2\mathbb{E}\|V_{Qk}-\mathbb{E}[V_{Qk}]\|^2+\frac{4L^2}{S}\|\mathbb{E}[V_{Qk}]\|^2 \\
        \overset{(iv)} \leq & \frac{2M^2}{S}+\frac{12M^2L^2\eta^2(Q+1)^2}{S}+\frac{4M^2L^2(Q+2)(Q+1)^2\eta^4\sigma^2}{S},
    \end{align*}
    where $(i)$ follows from the definitions of $\widehat{\nabla}\Phi(x_k;\mathcal{B}_x)$ and $\overline{\nabla}\Phi(x_k)$, 
    $(ii)$ and $(iii)$ follow from \Cref{ass:lip}, and $(iv)$ follows from \Cref{lem:Vqk} and the bound that 
    \begin{align}\label{eq:evqk}
         \|\mathbb{E}V_{Qk}\|^2\leq &\eta^2M^2\|\sum_{q=0}^{Q}(I-\eta \nabla_y^2g(x_k,y_k))^q\|^2 \leq \eta^2 M^2(Q+1)\sum_{q=0}^{Q}\|(I-\eta \nabla_y^2g(x_k,y_k))^q\|^2 \nonumber \\
         \leq& \eta^2M^2(Q+1)^2.
    \end{align}
    
    Then, we present the proof of \cref{eq:lossx}. Based on the forms of $v_k$ and $\overline{\nabla}\Phi(x_k)$, we have
    \begin{align*}
\mathbb{E} \| \bar{\epsilon}_k \| \overset{(i)}=& \mathbb{E} \| \widehat{\nabla}\Phi(x_k;\mathcal{B}_x) + (1-\alpha_k)(v_{k-1}-\widehat{\nabla} \Phi(x_{k-1};\mathcal{B}_x)) - \overline{\nabla}\Phi(x_k) \|^2 \\
=& \mathbb{E} \| \alpha_k(\widehat{\nabla}\Phi(x_k;\mathcal{B}_x) - \overline{\nabla}\Phi(x_k)) + (1-\alpha_k)((\widehat{\nabla}\Phi(x_k;\mathcal{B}_x)-\widehat{\nabla}\Phi(x_{k-1};\mathcal{B}_x))\\
&- (\overline{\nabla} \Phi(x_k)-\overline{\nabla}\Phi(x_{k-1}))
+ (1-\alpha_k)(\widehat{\nabla} \Phi(x_{k-1})-\overline{\nabla}\Phi(x_{k-1}) )\|^2 \\
\overset{(ii)}\leq& \mathbb{E}[2\alpha_k^2 \| \widehat{\nabla}\Phi(x_k;\mathcal{B}_x) - \overline{\nabla}\Phi(x_k) \|^2 + 2(1-\alpha_k)^2 \|\widehat{\nabla}\Phi(x_k;\mathcal{B}_x)-\widehat{\nabla}\Phi(x_{k-1};\mathcal{B}_x)\\
&-\overline{\nabla} \Phi(x_k) + \overline{\nabla} \Phi(x_{k-1}) \|^2 + (1-\alpha_k)^2 \| \overline{\epsilon}_{k-1} \|^2] \\
\overset{(iii)}\leq& \mathbb{E} [2\alpha_k^2 G^2 + 2(1-\alpha_k)^2 \| \widehat{\nabla}\Phi(x_k;\mathcal{B}_x)-\widehat{\nabla}\Phi(x_{k-1};\mathcal{B}_x) \|^2 + (1-\alpha_k)^2 \|\bar{\epsilon}_{k-1} \|^2] \\
 \overset{(iv)}\leq& \mathbb{E} [2\alpha_k^2 G^2 + 2(1-\alpha_k)^2L_Q^2 \| z_k-z_{k-1} \|^2 + (1-\alpha_k)^2 \|\bar{\epsilon}_{k-1} \|^2] \\
 \overset{(v)}\leq& \mathbb{E} [2\alpha_k^2 G^2 + 2(1-\alpha_k)^2L_Q^2 \| x_k-x_{k-1} \|^2 + 2(1-\alpha_k)^2L_Q^2 \| y_k-y_{k-1} \|^2\\
 &+ (1-\alpha_k)^2 \|\bar{\epsilon}_{k-1} \|^2],
\end{align*}
where $(i)$ follows from the definition of $v_k$, $(ii)$ follows because $\widehat{\nabla}\Phi(x_k;\mathcal{B}_x)$ and $\widehat{\nabla}\Phi(x_k;\mathcal{B}_x)-\widehat{\nabla}\Phi(x_{k-1};\mathcal{B}_x)$ are unbiased estimator of $\overline{\nabla}\Phi(x_k)$ and $\overline{\nabla} \Phi(x_k) - \overline{\nabla} \Phi(x_{k-1})$, respectively, $(iii)$ follows from \Cref{lem:G}, $(iv)$ follows from \Cref{lem:LQ}, and $(v)$ follows from the fact that $z_k=(x_k,y_k)$. 
   Then, the proof is complete.
\end{proof}
\end{lemma}

\begin{lemma}\label{lem:lossy}
	 Suppose Assumptions \ref{ass:conv}, \ref{ass:lip} and \ref{ass:var} hold. Let $\eta < \frac{1}{L}$. Then, we have 
	\begin{align*}
	\mathbb{E} \| \nabla_yg(x_k, y_k) - u_k \|^2 \leq &\mathbb{E} [2 \beta_k^2 G^2 + 2(1-\beta_k)^2 L^2 (\|x_k-x_{k-1}\|^2+\| y_k-y_{k-1}\|^2)\\
	&+(1-\beta_k)^2\| \nabla_yg(x_{k-1}, y_{k-1})-u_{k-1} \|^2],
	\end{align*}
	where $G$ is defined in \Cref{lem:G}.
	\begin{proof}
		This proof follow from the steps similar to the proof of \cref{eq:lossx} in \Cref{lem:G}.
	\end{proof}
\end{lemma}

Then, we characterize how the variance of the hypergradient and the inner-loop gradient change between iterations.
\begin{lemma}\label{lem:episdiff}
	Suppose Assumptions \ref{ass:conv}, \ref{ass:lip} and \ref{ass:var} hold. Let $\eta < \frac{1}{L}$, $c_1\geq \frac{2}{3d^3}+\frac{9\lambda \mu}{4}, c_2\geq \frac{2}{3d^3}+\frac{75{L'}^2\lambda}{2\mu}, \eta_k=\frac{d}{(m+k)^{1/3}}, m\geq \max \{2,(c_1d)^3,(c_2d)^3,d^3\},$ $\widetilde{x}_{k+1}=x_k-\gamma v_k, \; \widetilde{y}_{k+1}=y_k-\lambda u_k$, where ${L^{\prime}}^2=\max \{(L+\frac{L^2}{\mu}+\frac{M\tau}{\mu}+\frac{LM\rho}{\mu^2})^2,L_Q^2\}$. Then, we have
	\begin{align}\label{eq:epidiff}
	\frac{1}{\eta_k}\mathbb{E} \|\bar{\epsilon}_{k+1}\|^2 - \frac{1}{\eta_{k-1}}\mathbb{E} \| \bar{\epsilon}_k \|^2 \leq& -\frac{9\lambda \mu \eta_k}{4} \mathbb{E} \|\bar{\epsilon}_k\|^2 + 2 L_Q^2 \eta_k(\|\widetilde{x}_k-x_{k-1}\|^2+\| \widetilde{y}_k-y_{k-1}\|^2)\nonumber\\
	&+ \frac{2\alpha_{k+1}^2G^2}{\eta_k},
	\end{align}
	where $L_Q$ is defined in \Cref{lem:LQ}, $G$ and $\bar{\epsilon}_{k}$ are defined in  \Cref{lem:G}. Further, we characterize the relationship of the variance of the inner-loop gradient between iterations in the following inequality.
	\begin{align}\label{eq:gdiff}
	\frac{1}{\eta_k}& \mathbb{E} \|\nabla_yg(x_{k+1}, y_{k+1})-u_{k+1}\|^2 - \frac{1}{\eta_{k-1}} \mathbb{E} \|\nabla_yg(x_k, y_k)-u_k \|^2 \\
	&\leq -\frac{75{L^{\prime}}^2\lambda\eta_k}{2\mu} \mathbb{E} \| \nabla_yg(x_k, y_k) - u_k\|^2 
	+ 2L^2 \eta_k(\|\widetilde{x}_{k+1}-x_k\|^2+\|\widetilde{y}_{k+1}-y_k\|^2)+\frac{2\beta_{k+1}^2G^2}{\eta_k}.
	\end{align}
\end{lemma}

\begin{proof}
We first prove the \cref{eq:epidiff}. Based on the forms of $\bar{\epsilon}_{k}$, we have
\begin{align*}
\frac{1}{\eta_k}\mathbb{E}& \|\bar{\epsilon}_{k+1}\|^2 - \frac{1}{\eta_{k-1}}\mathbb{E} \| \bar{\epsilon}_k \|^2 \\
\overset{(i)}\leq& \left(\frac{(1-\alpha_{k+1})^2}{\eta_k}-\frac{1}{\eta_{k-1}}\right) \mathbb{E} \| \bar{\epsilon}_k\|^2 + 2(1-\alpha_{k+1})^2L_Q^2\eta_k(\|\widetilde{x}_k-x_{k-1}\|^2+\|\widetilde{y}_k-y_{k-1}\|^2)\\
&+ \frac{2\alpha_{k+1}^2G^2}{\eta_k}\\
 \overset{(ii)}\leq& \left(\frac{1}{\eta_k}-\frac{1}{\eta_{k-1}}-c_1\eta_k\right) \mathbb{E} \|\bar{\epsilon}_k\|^2 + 2 L_Q^2 \eta_k(\|\widetilde{x}_k-x_{k-1}\|^2+\| \widetilde{y}_k-y_{k-1}\|^2) + \frac{2\alpha_{k+1}^2G^2}{\eta_k} \\
\overset{(iii)}\leq& -\frac{9\lambda \mu \eta_k}{4} \mathbb{E} \|\bar{\epsilon}_k\|^2 + 2 L_Q^2 \eta_k(\|\widetilde{x}_k-x_{k-1}\|^2+\| \widetilde{y}_k-y_{k-1}\|^2) + \frac{2\alpha_{k+1}^2G^2}{\eta_k},
\end{align*}
where $(i)$ follows from \cref{eq:lossx}, $(ii)$ follows because $\alpha_{k+1}=c_1\eta_k^2\leq c_1\eta_0^2\leq 1$, and $(iii)$ follows from $c_1\geq \frac{2}{3d^3}+\frac{9\lambda \mu}{4}$. 

Then, we present the proof of \cref{eq:gdiff}. In particular, we have

\begin{align*}
\frac{1}{\eta_k}& \mathbb{E} \|\nabla_yg(x_{k+1}, y_{k+1})-u_{k+1}\|^2 - \frac{1}{\eta_{k-1}} \mathbb{E} \|\nabla_yg(x_k, y_k)-u_k \|^2 \\
 \overset{(i)}\leq& \left(\frac{1}{\eta_k}-\frac{1}{\eta_{k-1}}-c_2\eta_k\right) \mathbb{E} \| \nabla_yg(x_k, y_k) - u_k\|^2 + 2L^2 \eta_k(\|\widetilde{x}_{k+1}-x_k\|^2+\|\widetilde{y}_{k+1}-y_k\|^2)\\
&+\frac{2\beta_{k+1}^2G^2}{\eta_k}\\
\overset{(ii)}\leq& -\frac{75{L^{\prime}}^2\lambda\eta_k}{2\mu} \mathbb{E} \| \nabla_yg(x_k, y_k) - u_k\|^2 + 2L^2 \eta_k(\|\widetilde{x}_{k+1}-x_k\|^2+\|\widetilde{y}_{k+1}-y_k\|^2)+\frac{2\beta_{k+1}^2G^2}{\eta_k},
\end{align*}
where $(i)$ follows from \Cref{lem:lossy} and because $\beta_{k+1}=c_2\eta_k^2\leq c_2\eta_0^2 \leq 1 $, and $(ii)$ follows because $c_2\geq \frac{2}{3d^3}+\frac{75{L'}^2\lambda}{2\mu}$. Then, the proof is complete.
\end{proof}

Next, we characterize the approximation bound $C_Q$ on the Hessian inverse.
\begin{lemma}\label{lem:CQ}
Suppose Assumptions \ref{ass:conv}, \ref{ass:lip} and \ref{ass:var} hold. Let $\eta < \frac{1}{L}$. Then, we have
\begin{align*}
    \|\widetilde{\nabla}\Phi(x_k)-\overline{\nabla}\Phi(x_k)\| \leq C_Q,
\end{align*}
where $\widetilde{\nabla}\Phi(x_k)$ is defined in \cref{eq:tildephi}, $\overline{\nabla}\Phi(x_k)$ is defined in \cref{eq:hypergradientest},
and $C_Q=\frac{(1-\eta\mu)^{Q+1}ML}{\mu}$.
\begin{proof}
Following from the proof of Proposition 3 in \cite{ji2021bilevel}, we have $\|\mathbb{E}[V_{Qk}]-[\nabla_y^2g(x_k,y_k)]^{-1}\nabla_yf(x_k,y_k)\|\leq \frac{(1-\eta\mu)^{Q+1}M}{\mu}$. Then, we obtain
\begin{align*}
    \|\widetilde{\nabla}\Phi(x_k)-\overline{\nabla}\Phi(x_k)\|
    \leq& \|\nabla_x\nabla_yg(x_k,y_k)\|\|\mathbb{E}[V_{Qk}]-[\nabla_y^2g(x_k,y_k)]^{-1}\nabla_yf(x_k,y_k)\| \\
    \overset{(i)}\leq & \frac{(1-\eta\mu)^{Q+1}ML}{\mu},
\end{align*}
where $(i)$ follows from \Cref{ass:lip}. Then, the proof is complete.
\end{proof}
\end{lemma}

\begin{lemma}[{\bf Restatement of \Cref{prop:mrbio_phi}}]\label{lem:ephi}
	Suppose Assumptions \ref{ass:conv}, \ref{ass:lip} and \ref{ass:var} hold. Let $\eta < \frac{1}{L}$, and $ \gamma \leq \frac{1}{4L_{\Phi}\eta_k}$, where $L_\Phi = L+\frac{2L^2+\tau M^2}{\mu}+\frac{\rho LM+L^3+\tau ML}{\mu^2}+\frac{\rho L^2M}{\mu^3}$. Then, we have 
	\begin{align*}
	\mathbb{E}[\Phi(x_{k+1})] \leq \mathbb{E}[\Phi(x_k)]+2\eta_k \gamma {L^{\prime}}^2 \|y_k-y^{\ast}(x_k)\|^2 + 2\eta_k \gamma\|\bar{\epsilon}_k\|^2 + 2\eta_k \gamma C_Q^2 - \frac{1}{2\gamma\eta_k} \|x_{k+1}-x_k\|^2,
	\end{align*}
where $C_Q$ is defined in \Cref{lem:CQ} and $L'$ is defined in \Cref{prop:mrbio_phi}.
\begin{proof}
Based on the Lipschitz property of $\Phi(x_k)$, we have
\begin{align*}
\mathbb{E}[\Phi(x_{k+1})]
\overset{(i)}\leq& \mathbb{E}[\Phi(x_k) + \langle\nabla \Phi(x_k), x_{k+1}-x_k \rangle + \frac{L_{\Phi}}{2} \| x_{k+1}-x_k\|^2] \\
\overset{(ii)} =&\mathbb{E}[\Phi(x_k)+\eta_k \langle \nabla \Phi(x_k), \widetilde{x}_{k+1}-x_k \rangle + \frac{L_{\Phi}}{2} \eta_k^2 \|\widetilde{x}_{k+1}-x_k\|^2] \\
=&\mathbb{E}[\Phi({x_k})+\eta_k \langle \nabla \Phi(x_k)-v_k, \widetilde{x}_{k+1}-x_k\rangle + \eta_k \langle v_k, \widetilde{x}_{k+1}-x_k \rangle\\
&+\frac{L_{\Phi}}{2} \eta_k^2 \| \widetilde{x}_{k+1}-x_k\|^2],
\end{align*}
where $(i)$ follows from the smoothness of the function $\Phi(x)$ proved by Lemma 2 in \cite{ji2021bilevel}, and $(ii)$ follows because $\eta_k (\widetilde{x}_{k+1}-x_k) =x_{k+1}-x_k$, where $\widetilde{x}_{k+1}$ is defined in \Cref{lem:episdiff}.

Based on Lemma 25 in \cite{huang2020accelerated}, we have
$\langle v_k, \widetilde{x}_{k+1}-x_k \rangle \leq -\frac{1}{\gamma} \| \widetilde{x}_{k+1}-x_k\|^2$, which yields
\begin{align*}
\langle \nabla \Phi(x_k)& - v_k, \widetilde{x}_{k+1}-x_k \rangle \\
 =& \langle \nabla \Phi(x_k) - \widetilde{\nabla} \Phi(x_k) + \widetilde{\nabla} \Phi(x_k) -\overline{\nabla}\Phi(x_k) + \overline{\nabla}\Phi(x_k)-v_k, \widetilde{x}_{k+1} - x_k\rangle \\
 \leq & \|\nabla\Phi(x_k)-\widetilde{\nabla} \Phi(x_k)\|\|\widetilde{x}_{k+1} - x_k\|+\|\widetilde{\nabla} \Phi(x_k) -\overline{\nabla}\Phi(x_k)\|\|\widetilde{x}_{k+1} - x_k\|\\
 &+\|\overline{\nabla}\Phi(x_k)-v_k\|\|\widetilde{x}_{k+1} - x_k\|\\
 \overset{(i)}\leq& 2\gamma {L^{\prime}}^2 \|y_k-y^{\ast}(x_k)\|^2 + \frac{1}{8\gamma} \|\widetilde{x}_{k+1}-x_k\|^2+C_Q\|\widetilde{x}_{k+1}-x_k\|+2\gamma \|\overline{\nabla} \Phi(x_k) - v_k\|^2\\
&+ \frac{1}{8\gamma} \|\widetilde{x}_{k+1}-x_k \|^2 \\
 \overset{(ii)}\leq& 2\gamma {L^{\prime}}^2 \|y_k-y^{\ast}(x_k)\|^2 + \frac{1}{8\gamma} \|\widetilde{x}_{k+1}-x_k\|^2+2\gamma C_Q^2+\frac{1}{8\gamma}\|\widetilde{x}_{k+1}-x_k\|^2\\
&+2\gamma \|\overline{\nabla} \Phi(x_k) - v_k\|^2 
+ \frac{1}{8\gamma} \| \widetilde{x}_{k+1}-x_k\|^2,
\end{align*}
where $(i)$ follows from \cite[Lemma 7]{ji2021bilevel}, \Cref{lem:CQ} and Young's inequality, and $(ii)$ follows from Young's inequality.

Combining the above inequalities and applying $\gamma \leq \frac{1}{4L_{\Phi}\eta_k}$, we have
\begin{align*}
\mathbb{E}[\Phi(x_{k+1})] \leq& \mathbb{E}[\Phi(x_k)]+2\eta_k \gamma {L^{\prime}}^2 \|y_k-y^{\ast}(x_k)\|^2 + 2\eta_k \gamma\|\bar{\epsilon}_k\|^2 + 2\eta_k \gamma C_Q^2 - \frac{\eta_k}{2\gamma} \|\widetilde{x}_{k+1}-x_k\|^2 \\
= &\mathbb{E}[\Phi(x_k)]+2\eta_k \gamma {L^{\prime}}^2 \|y_k-y^{\ast}(x_k)\|^2 + 2\eta_k \gamma\|\bar{\epsilon}_k\|^2 + 2\eta_k \gamma C_Q^2 - \frac{1}{2\gamma\eta_k} \|x_{k+1}-x_k\|^2.
\end{align*}
Then, the proof is complete.
\end{proof}
\end{lemma}

\begin{lemma}\label{lem:ystar}
Suppose Assumptions \ref{ass:conv}, \ref{ass:lip} and \ref{ass:var} hold. Let $ \eta_k < 1$ and $0<\lambda \leq \frac{1}{6L}$. Then, we have
	\begin{align*}
	\| y_{k+1}-y^{\ast}(x_{k+1}) \|^2 \leq& \left(1-\frac{\eta_k \mu \lambda}{4}\right)\|y_k-y^{\ast}(x_k)\|^2 - \frac{3\eta_k}{4} \|\widetilde{y}_{k+1} - y_k\|^2 \\
	&+ \frac{25\eta_k \lambda}{6\mu} \|\nabla_yg(x_k, y_k) - u_k\|^2 + \frac{25 L^2 \eta_k}{6\mu^3 \lambda} \|x_k-\widetilde{x}_{k+1}\|^2.
	\end{align*}
	\begin{proof}
	Based on Lemma 18 in \cite{huang2020accelerated} first version, we obtain
	\begin{align*}
	    \|y_{t+1}-y^\ast(x_{t+1})\|^2 \leq& (1-\frac{\eta_t\tau \lambda}{4})\|y_t-y^\ast(x_t)\|^2 - \frac{3\eta_t}{4}\|\widetilde{y}_{t+1}-y_t\|^2 \\
	    & + \frac{25\eta_t \lambda}{6\tau}\|\nabla_yf(x_t,y_t)-w_t\|^2+\frac{25\kappa_y^2\eta_t}{6\tau \lambda}\|x_t-\widetilde{x}_{t+1}\|^2, 
	\end{align*}
	where $\kappa_y=L_f/\tau$.
	The proof is finished by replacing $f(x_t,y_t)$ with $-g(x_k,y_k)$.
	\end{proof}
\end{lemma}

\subsection{Proof of \Cref{thm:mrbio}}
Based on the above lemmas, we develop the proof of \Cref{thm:mrbio} in the following.
\begin{theorem}[{\bf Restatement of \Cref{thm:mrbio}}]
	\label{thm:mrbio_resta}
	Apply MRBO to solve the problem in \cref{eq:main}. Suppose Assumptions \ref{ass:conv}, \ref{ass:lip}, and \ref{ass:var} hold. Let the hyperparameters $c_1\geq \frac{2}{3d^3}+\frac{9\lambda \mu}{4}, c_2\geq \frac{2}{3d^3}+\frac{75{L'}^2\lambda}{2\mu}, m\geq \max\{2, d^3, (c_1d)^3, (c_2d)^3\}, y_1=y^{\ast}(x_1), \eta<\frac{1}{L}, 0\leq \lambda \leq \frac{1}{6L}, 0\leq \gamma \leq \min \{ \frac{1}{4L_\Phi \eta_K},\frac{\lambda \mu}{\sqrt{150{L^{\prime}}^2L^2/\mu^2+8\lambda\mu(L_Q^2+L^2)}} \}$. Then, we have 
	\begin{align}
	\label{eq:mrbio_conv_restate}
	\textstyle \frac{1}{K} \sum_{k=1}^{K}\left(\frac{{L^{\prime}}^2}{4} \|y^{\ast}(x_k)-y_k\|^2+\frac{1}{4}\|\bar{\epsilon}_k\|^2+\frac{1}{4\gamma^2 \eta_k^2}\|x_{k+1}-x_k\|^2\right) \leq \frac{M'}{K}(m+K)^{1/3},
	\end{align}
	where ${L^{\prime}}^2$ is defined in \Cref{prop:mrbio_phi}, and $M' =\frac{\Phi(x_1)-\Phi^{\ast}}{\gamma d} +\left(\frac{2G^2(c_1^1+c_2^2)d^2}{\lambda \mu}+\frac{2C_Q^2d^2}{\eta_K^2}\right)\log(m+K)+ \frac{2G^2}{S \lambda \mu d \eta_0}$.

\begin{proof}
Firstly, we define a Lyapunov function,
\begin{align*}
\delta_k = \Phi(x_k) + \frac{\gamma}{\lambda \mu}\left(9 {L^{\prime}}^2\|y_k-y^{\ast}(x_k)\|^2+\frac{1}{\eta_{k-1}}\|\bar{\epsilon}_k\|^2+\frac{1}{\eta_{k-1}}\|\nabla_yg(x_k, y_k)-u_k\|^2\right).
\end{align*}
Then, we have
\begin{align*}
\delta_{k+1}&-\delta_k\\ 
=& \Phi(x_{k+1})-\Phi(x_k) + \frac{9 {L^{\prime}}^2 \gamma}{\lambda \mu} (\|y_{k+1}-y^{\ast}(x_k)\|^2 -\|y_k - y^{\ast}(x_k)\|^2) \\ 
&+ \frac{\gamma}{\lambda \mu} \bigg(\frac{1}{\eta_k}\|\bar{\epsilon}_{k+1}\|^2 - \frac{1}{\eta_{k-1}}\|\bar{\epsilon}_k\|^2 + \frac{1}{\eta_k}\|\nabla_yg(x_{k+1}, y_{k+1})-u_{k+1}\|^2 \\
&- \frac{1}{\eta_{k-1}}\|\nabla_yg(x_k, y_k)-u_k\|^2\bigg) \\
 \overset{(i)}\leq& -\frac{\eta_k}{2 \gamma} \|\widetilde{x}_{k+1}-x_k\|^2+2\eta_k \gamma {L^{\prime}}^2 \|y_k - y^{\ast}(x_k) \|^2 + 2\eta_k \gamma \|\bar{\epsilon}_k\|^2 + 2\eta_k \gamma C_Q^2 \\
&+ \frac{9 {L^{\prime}}^2 \gamma}{\lambda \mu} \bigg(-\frac{\eta_k \mu \lambda}{4} \|y_k-y^{\ast}(x_k)\|^2-\frac{3\eta_k}{4}  \|\widetilde{y}_{k+1}-y_k\|^2+\frac{25\eta_k \lambda}{6\mu} \|\nabla_yg(x_k, y_k)-u_k\|^2\\
&+\frac{25\kappa_y^2 \eta_k}{6\lambda \mu} \|x_k-\widetilde{x}_{k+1}\|^2\bigg)
+\frac{\gamma}{\lambda \mu}\bigg(-\frac{9\lambda \mu \eta_k}{4}\|\bar{\epsilon}_k\|^2 + 2L_Q^2\eta_k(\|\widetilde{x}_{k+1}-x_k\|^2 + \|\widetilde{y}_{k+1}-y_k\|^2) \\&+ \frac{2\alpha_{k+1}^2 G^2}{\eta_k}\bigg)+\frac{\gamma}{\lambda \mu}\bigg(-\frac{75 {L^{\prime}}^2 \lambda}{2\mu} \eta_k\|\nabla_yg(x_k, y_k)-u_k\|^2+2L^2\eta_k(\|\widetilde{x}_{k+1}-x_k\|^2\\&+\|\widetilde{y}_{k+1}-y_k\|^2)+\frac{2\beta_{k+1}^2G^2}{\eta_k}\bigg) 
\\ \overset{(ii)}\leq& -\frac{{L^{\prime}}^2\eta_k \gamma}{4} \|y^{\ast}(x_k) - y_k\|^2 -\frac{\gamma \eta_k}{4} \|\bar{\epsilon}_k\|^2 - \frac{\eta_k}{4\gamma} \|\widetilde{x}_{k+1} - x_k\|^2 + \frac{2\alpha_{k+1}^2 G^2 \gamma}{\lambda \mu \eta_k} + \frac{2\beta_{k+1}^2 G^2 \gamma}{\lambda \mu \eta_k},
\end{align*} 
where $(i)$ follows from Lemmas \ref{lem:episdiff} and \ref{lem:ystar}, $(ii)$ follows because $L'\geq L_Q$ and $ 0\leq \gamma \leq \frac{\lambda \mu}{\sqrt{150{L^{\prime}}^2L^2/\mu^2+8\lambda\mu(L_Q^2+L^2)}}$.
Rearranging the terms in above inequality, we obtain
\begin{align} \label{eq:trackingerror}
\frac{{L^{\prime}}^2 \eta_k}{4} \|y^{\ast}(x_k) - y_k\|^2 + \frac{\eta_k}{4} \|\bar{\epsilon}_k\|^2 + \frac{\eta_k}{4\gamma^2} \|\widetilde{x}_{k+1}-x_k\|^2 \leq& \frac{\delta_k-\delta_{k+1}}{\gamma} + \frac{2(\alpha_{k+1}^2+\beta_{k+1}^2)G^2}{\lambda \mu \eta_k} \nonumber \\
    +& 2\eta_k C_Q^2.
\end{align}
Note that we set $y_1 = y^{\ast}(x_1)$ and obtain
\begin{align*}
\delta_1 = \Phi(x_1) + \frac{\gamma}{\lambda \mu}\left(9 {L^{\prime}}^2 \|y_1-y^{\ast}(x_1)\|^2+\frac{1}{\eta_0}\|\bar{\epsilon}_1\|^2+\frac{1}{\eta_0}\|\nabla_yg(x_1,y_1)-u_1\|^2\right).
\end{align*}
Then, telescoping \cref{eq:trackingerror} over $k$ from $1$ to $K$ yields
\begin{align*}
\frac{1}{K}\sum_{k=1}^{K}& \left(\frac{{L^{\prime}}^2}{4}\|y^{\ast}(x_k)-y_k\|^2+\frac{1}{4}\|\bar{\epsilon}_k\|^2+\frac{1}{4\gamma^2}\|\widetilde{x}_{k+1}-x_k\|^2\right) \\
 \overset{(i)}\leq& \frac{1}{K \eta_k \gamma}\left(\Phi(x_1)+\frac{2\gamma G^2}{S\lambda \mu \eta_0}-\Phi^{\ast}\right) + \frac{1}{K \eta_k} \sum_{k=1}^{K} \left(\frac{2\alpha_{k+1}^2 G^2}{\lambda \mu \eta_k}+\frac{2\beta_{k+1}^2G^2}{\lambda \mu \eta_k}+2\eta_k C_Q^2\right) \\
\overset{(ii)}\leq& \frac{1}{K\eta_k \gamma}(\Phi(x_1)-\Phi^{\ast}) + \frac{2G^2}{K\eta_K S \lambda \mu \eta_0} + \frac{(2c_1^2G^2+2c_2G^2)d^3}{K\eta_K \lambda \mu} \log(m+K)\\ &+\frac{2C_Q^2d^3}{K\eta_K^3} \log(m+K) \\
\leq& \frac{\Phi(x_1)-\Phi^{\ast}}{\gamma d} \frac{(m+K)^{1/3}}{K}+\frac{2G^2}{dS\lambda \mu \eta_0} \frac{(m+K)^{1/3}}{K} \\
&+\bigg(\frac{(2c_1^2G^2+2c_2^2G^2)d^2}{\lambda \mu}+ \frac{2C_Q^2d^2}{\eta_K^2}\bigg) \frac{(m+K)^{1/3}}{K} \log(m+K)
\end{align*}
where $(i)$ follows from \cref{eq:trackingerror}, $(ii)$ follows because $\sum_{k=1}^K\eta_k^3\leq \int_1^K \frac{d^3}{m+k} \leq d^3\log(m+K)$. 

We further apply $\|\widetilde{x}_{k+1}-x_k\|=\eta\|x_{k+1}-x_k\|$ to the above inequality and obtain

\begin{align}
\label{eq:mrbothm}
\frac{1}{K} \sum_{k=1}^{K} \left(\frac{{L^{\prime}}^2}{4} \|y^{\ast}(x_k)-y_k\|^2+\frac{1}{4} \| \bar{\epsilon}_k\|^2 + \frac{1}{4\gamma^2\eta_k^2} \|x_{k+1}-x_k\|^2\right) \leq \frac{M'}{K} (m+K)^{1/3}.
\end{align}
where $M^{'} =\frac{\Phi(x_1)-\Phi^{\ast}}{\gamma d} + \frac{2G^2}{S \lambda \mu d \eta_0}+\left(\frac{2G^2(c_1^1+c_2^2)d^2}{\lambda \mu}+\frac{2C_Q^2d^2}{\eta_K^2}\right)\log(m+K)$. Then, the proof is complete.

\end{proof}
\end{theorem}
\subsection{Proof of \Cref{coro:mrbio}}
\begin{coro} [{\bf Restatement of \Cref{coro:mrbio}}]
    \label{coro:mrbio_restate}
Under the same conditions of Theorem \ref{thm:mrbio} and choosing  $K=\mathcal{O}(\epsilon^{-1.5}), Q=\mathcal{O}(\log(\frac{1}{\epsilon}))$, 
MRBO in \Cref{alg:mrbio} finds an $\epsilon$-stationary point with the gradient complexity of $\mathcal{O}(\epsilon^{-1.5})$ and the (Jacobian-) Hessian-vector complexity of $\mathcal{\widetilde O}(\epsilon^{-1.5})$. 
\end{coro}

\begin{proof}

We choose $Q=\mathcal{O}(\log(\frac{1}{\epsilon}))$, $K=\mathcal{O}(\epsilon^{-1.5})$ and $S=\mathcal{O}(1)$, and then have $\mathcal{O}(C_Q)=\mathcal{O}(\epsilon^{-1})$, $M'=\mathcal{O}(1)$,  and $m=\mathcal{O}(1)$. Hence, $\mathcal{O}(\frac{M'}{K} (m+K)^{1/3}) \leq \mathcal{O}(\frac{M' m^{1/3}}{K}+\frac{M'}{K^{2/3}})= \mathcal{O}(\frac{1}{K^{2/3}})=\mathcal{O}(\epsilon)$, which guarantees the target accuracy. The gradient complexity and Jacobian-vector complexity are given by $KS=\mathcal{O}(\epsilon^{-1.5})$, and the Hessian-vector complexity is given by $KSQ=\mathcal{\widetilde O}(\epsilon^{-1.5})$.
\end{proof}

\section{Proof of \Cref{thm:vrbio}} \label{sec:proofvrbo}

\subsection{Proofs of Supporting Lemmas (Propositions \ref{prop:vrbo_est} and \ref{prop:vrbo_phi} Correspond to Lemmas \ref{lem:deltadef} and \ref{lem:vk})}
For notation simplification, we define the following:
\begin{align*}
    V_{Q\xi} = \eta \sum_{q=-1}^{Q-1} \prod_{j=Q-q}^{Q} (I-\eta \nabla_y^2G(x_k, y_k; \zeta_j))\nabla_yF(x_k,y_k;\xi),
\end{align*}
which is a single-sample form of $V_{Qk}$ defined in \cref{eq:Vqk}. We note that $\|\mathbb{E}[V_{Q\xi}]\|^2=\|\mathbb{E}[V_{Qk}]\|^2\leq \eta^2 M^2(Q+1)^2$, where the inequality follows from \cref{eq:evqk}.


Firstly, we characterize the variance of the hypergradients between different iterations.
\begin{lemma} \label{lem:spider}
Consider \Cref{alg:vrbio}. Suppose Assumptions \ref{ass:lip} and \ref{ass:var} hold. Then, we have
\begin{align}
    \mathbb{E}[\|\widetilde{v}_{k,t}-\overline{\nabla}\Phi(\widetilde{x}_{k,t})\|^2] \leq& \mathbb{E}[\|\widetilde{v}_{k,t-1}-\overline{\nabla}\Phi(\widetilde{x}_{k,t-1})\|^2]+\frac{L_Q^2}{S_2}\mathbb{E}[\|\widetilde{x}_{k,t}-\widetilde{x}_{k,t-1}\|^2\nonumber\\
    &+\|\widetilde{y}_{k,t}-\widetilde{y}_{k,t-1}\|^2],
\end{align}
where $\overline{\nabla}\Phi(\widetilde{x}_{k,t})$ is defined in \cref{eq:hypergradientest} and $L_Q$ is defined in \Cref{lem:LQ}.
    \begin{proof}
    In \Cref{alg:vrbio}, the hypergradient estimator $\widetilde{v}_{k,t}$ updates as the following form:
        \begin{align*}
            \widetilde{v}_{k,t} = \widehat{\nabla}\Phi(\widetilde{x}_{k,t},\widetilde{y}_{k,t};\mathcal{S}_2)-\widehat{\nabla}\Phi(\widetilde{x}_{k,t-1},\widetilde{y}_{k,t-1};\mathcal{S}_2)+\widetilde{v}_{k,t-1}.
        \end{align*}
          Note that 
           \begin{align*}
               \mathbb{E} [\widehat{\nabla}\Phi(\widetilde{x}_{k,t},\widetilde{y}_{k,t};\mathcal{S}_2)-\widehat{\nabla}\Phi(\widetilde{x}_{k,t-1},\widetilde{y}_{k,t-1};\mathcal{S}_2)|\widetilde{x}_{k,0:t},\widetilde{y}_{k,0:t}] = \mathbb{E}[\|\widetilde{v}_{k,t}-\widetilde{v}_{k,t-1}\|].
           \end{align*}
           Based on Lemma 1 in \cite{fang2018spider},
           \begin{align*}
               \mathbb{E} \|\widetilde{v}_{k,t}&-\widetilde{v}_{k,t-1}-(\widehat{\nabla}\Phi(\widetilde{x}_{k,t},\widetilde{y}_{k,t};\mathcal{S}_2)-\widehat{\nabla}\Phi(\widetilde{x}_{k,t-1},\widetilde{y}_{k,t-1};\mathcal{S}_2))\|^2 \\
               \leq& \frac{1}{S_2} \mathbb{E} [\|(\widehat{\nabla}\Phi(\widetilde{x}_{k,t},\widetilde{y}_{k,t};\xi)-\widehat{\nabla}\Phi(\widetilde{x}_{k,t-1},\widetilde{y}_{k,t-1};\xi)) \|^2].
           \end{align*}
            Furthermore, since $\widehat{\nabla}\Phi(x_k,y_k;\xi)$ is $L_Q$- Lipschitz continuous which is proved in \Cref{lem:LQ}, we have
            \begin{align*}
                \mathbb{E}[\|\widetilde{v}_{k,t}-\overline{\nabla}\Phi(\widetilde{x}_{k,t})\|^2] \leq& \mathbb{E}[\|\widetilde{v}_{k,t-1}-\overline{\nabla}\Phi(\widetilde{x}_{k,t-1})\|^2]+\frac{L_Q^2}{S_2}\mathbb{E}[\|\widetilde{x}_{k,t}-\widetilde{x}_{k,t-1}\|^2 \\
                &+\|\widetilde{y}_{k,t}-\widetilde{y}_{k,t-1}\|^2].
            \end{align*}
            Then, the proof is complete.
    \end{proof}
\end{lemma}

\begin{lemma} \label{lem:delta_ti}
Suppose Assumptions \ref{ass:lip} and \ref{ass:var} hold. Let 
$\Delta_k=\mathbb{E}\|\overline{\nabla}\Phi(x_k)-v_k\|^2+\mathbb{E}\|\nabla_yg(x_k,y_k)-u_k\|^2,$ and $\widetilde{\Delta}_{k,t}=\mathbb{E}\|\overline{\nabla}\Phi(\widetilde{x}_{k,t})-\widetilde{v}_{k,t}\|^2+\mathbb{E}\|\nabla_yg(\widetilde{x}_{k,t},\widetilde{y}_{k,t})-\widetilde{u}_{k,t}\|^2$. Then, we have
	\begin{equation*}
	\widetilde{\Delta}_{k,0} \leq \Delta_k + \frac{2L_Q^{2}}{S_2} \mathbb{E} (\|x_{k+1}-x_k\|^2),
	\end{equation*}
	where $L_Q$ is defined in \Cref{lem:CQ}.
	\begin{proof} 
Based on the form of $\widetilde{\Delta}_{k,0}$, we have
		\begin{align*}
		\widetilde{\Delta}_{k,0} &= \mathbb{E}(\|\widetilde{v}_{k,0}-\overline{\nabla}\Phi(\widetilde{x}_{k,0}, \widetilde{y}_{k,0})\|^2+\|\widetilde{u}_{k,0}-\nabla_yg(\widetilde{x}_{k,0},\widetilde{y}_{k,0})\|^2) \\
		 &\overset{(i)}\leq \mathbb{E}(\|v_k-\overline{\nabla}\Phi(x_k,y_k)\|^2+\|u_k-\nabla_yg(x_k,y_k)\|^2+ \frac{2L_Q^{2}}{S_2} \mathbb{E}(\|\widetilde{x}_{k,0}-x_k\|^2+\|\widetilde{y}_{k,0}-y_k\|^2) \\
		& \overset{(ii)}= \Delta_k + \frac{2L_Q^{2}}{S_2} \mathbb{E}(\|x_{k+1}-x_k\|^2+\|y_k-y_k\|^2) \\
		& = \Delta_k + \frac{2L_Q^{2}}{S_2} \mathbb{E} (\|x_{k+1}-x_k\|^2),
		\end{align*}
	where $(i)$ follows from \Cref{lem:spider}, $(ii)$ follows because $\widetilde{y}_{k,0}=y_k$. Then, the proof is complete.
	\end{proof}
\end{lemma}

\begin{lemma}
Suppose Assumptions \ref{ass:lip} and \ref{ass:var} hold. Then, we have
\begin{align*}
    \Delta_k \leq  \widetilde{\Delta}_{k-1,0} + \frac{2L_Q^{2}}{S_2} \beta^2 \sum_{t=0}^{m}\|\widetilde{u}_{k-1,t}\|^2,
\end{align*}
	where $\widetilde{\Delta}_{k-1,0}, \Delta_k$ are defined in \Cref{lem:delta_ti} and $L_Q$ is defined in \Cref{lem:CQ}.
\begin{proof}
Based on the forms of $\Delta_k$, we have
    \begin{align*}
	\Delta_k =& \mathbb{E}(\|v_k-\overline{\nabla}\Phi(x_k,y_k)\|^2+\|u_k-\nabla_yg(x_k,y_k)\|^2) \\
	 \overset{(i)}=& \widetilde{\Delta}_{k-1,m+1} \\
	 \overset{(ii)}\leq& \mathbb{E}(\|\widetilde{v}_{k-1,0}-\overline{\nabla}\Phi(\widetilde{x}_{k-1,0}, \widetilde{y}_{k-1,0})\|^2+\|\widetilde{u}_{k-1,0}-\nabla_yg(\widetilde{x}_{k-1,0},\widetilde{y}_{k-1,0}\|^2)\\
	&+\frac{2L_Q^{2}}{S_2} \sum_{t=0}^{m} (\|\widetilde{x}_{k-1,t+1}-\widetilde{x}_{k-1,t}\|^2+\|\widetilde{y}_{k-1,t+1}-\widetilde{y}_{k-1,t}\|^2) \\
	 \overset{(iii)}=& \widetilde{\Delta}_{k-1,0} + \frac{2L_Q^{2}}{S_2} \beta^2 \sum_{t=0}^{m}\|\widetilde{u}_{k-1,t}\|^2,
	\end{align*}
	where $(i)$ follows because $u_k=\widetilde{u}_{k-1,m+1}$, and $v_k=\widetilde{v}_{k-1,m+1}$, $(ii)$ follows from \Cref{lem:spider}, and $(iii)$ follows from the fact that $\widetilde{x}_{k-1,t+1}=\widetilde{x}_{k-1,t}$. Then, the proof is complete.
\end{proof}
\end{lemma}

Furtheremore, we characterize the relationship between $\widetilde{u}_{k,t}$ in different iterations.
\begin{lemma}\label{lem:u_kt}
Suppose Assumptions $\ref{ass:conv}$, $\ref{ass:lip}$ and $\ref{ass:var}$ hold. We let $\beta=\frac{2}{13L_Q}$, and $S_2\geq 2(\frac{L}{\mu}+1)L\beta$. Then, we have
\begin{align}
    \mathbb{E}[\|\widetilde{u}_{k,t}\|^2|\mathcal{F}_{k,t}] \leq a\|\widetilde{u}_{k,t-1}\|^2,
\end{align}
where $a=\left(1-\frac{\beta \mu L}{\mu + L}\right)$, $\widetilde{u}_{k,t}$ is defined in \Cref{alg:vrbio}, and $\mathcal{F}_{k,t}$ denotes all information of $ {\{\widetilde{y}_{k,j}\}}_{j=0}^t$ and $\{\widetilde{u}_{k,j}\}_{j=0}^{t-1}$.
\begin{proof}
Based on the definition of $\widetilde{u}_{k,t}$, we have
    \begin{align*}
	 \mathbb{E}[\|\widetilde{u}_{k,t}&\|^2|\mathcal{F}_{k,t}]\\ =&\|\widetilde{u}_{k,t-1}\|^2+2\mathbb{E}[\langle\widetilde{u}_{k,t-1}, \nabla_yG(\widetilde{y}_{k,t})-\nabla_yG(\widetilde{y}_{k,t-1})\rangle|\mathcal{F}_{k,t}]\\
	&+\mathbb{E}[\|\nabla_yG(\widetilde{y}_{k,t})-\nabla_yG(\widetilde{y}_{k,t-1})\|^2|\mathcal{F}_{k,t}] \\
	 =& \|\widetilde{u}_{k,t-1}\|^2-\frac{\beta}{2}\mathbb{E}[\langle\widetilde{y}_{k,t}-\widetilde{y}_{k,t-1}, \nabla_yg(\widetilde{y}_{k,t})-\nabla_yg(\widetilde{y}_{k,t-1})\rangle] \\
	& + [\|\nabla_yG(\widetilde{y}_{k,t})-\nabla_yG(\widetilde{y}_{k,t-1})\|^2|\mathcal{F}_{k,t}]\\
	 \overset{(i)}\leq& \|\widetilde{u}_{k,t-1}\|^2 - \frac{2}{\beta}\left(\frac{\mu L}{\mu+L}\|\widetilde{y}_{k,t}-\widetilde{y}_{k,t-1}\|^2+\frac{1}{\mu+L}\|\nabla_yg(\widetilde{y}_{k,t}-\nabla_yg(\widetilde{y}_{k,t-1})\|^2\right) \\
	&+\mathbb{E}[\|\nabla_yG(\widetilde{y}_{k,t})-\nabla_yG(\widetilde{y}_{k,t-1})\|^2|\mathcal{F}_{k,t}] \\
	 \leq& \left(1-\frac{2\beta \mu L}{\mu +L}\right)\|\widetilde{u}_{k,t-1}\|^2 - \left(\frac{2}{\beta(\mu+L)}-2\right)\|\nabla_yg(\widetilde{y}_{k,t})-\nabla_yg(\widetilde{y}_{k,t-1})\|^2] \\
	& + 2\mathbb{E}[\|\nabla_yG(\widetilde{y}_{k,t})-\nabla_yG(\widetilde{y}_{k,t-1})-[\nabla_yg(\widetilde{y}_{k,t})-\nabla_yg(\widetilde{y}_{k,t-1})]\|^2|\mathcal{F}_{k,t}] \\
	\overset{(ii)}\leq& \left(1-\frac{2\beta \mu L}{\mu +L}\right)\|\widetilde{u}_{k,t-1}\|^2+2\mathbb{E}[\|\nabla_yG(\widetilde{y}_{k,t})-\nabla_yG(\widetilde{y}_{k,t-1})-(\nabla_yg(\widetilde{y}_{k,t})-\nabla_yg(\widetilde{y}_{k,t-1}))\|^2] \\
	 \overset{(iii)}\leq & \left(1-\frac{2\beta \mu L}{\mu + L}\right)\|\widetilde{u}_{k,t-1}\|^2+\frac{2L^2}{S_2}\|\widetilde{y}_{k,t}-\widetilde{y}_{k,t-1}\|^2 \\
	 =& \left(1-\frac{2\beta \mu L}{\mu + L}+\frac{2L^2\beta^2}{S_2}\right)\|\widetilde{u}_{k,t-1}\|^2 \\
	 \overset{(iv)}\leq &\left(1-\frac{\beta \mu L}{\mu+L}\right)\|\widetilde{u}_{k,t-1}\|^2
	\end{align*}
	where $(i)$ follows from Assumptions \ref{ass:lip} and \ref{ass:var}, $(ii)$ follows from the fact that $\beta \leq \frac{1}{2L}$, $(iii)$ follows from \Cref{lem:spider}, and $(iv)$ follows because $S_2\geq 2(\frac{L}{\mu}+1)L\beta$. Then, the proof is complete.
\end{proof}
\end{lemma}

Furthermore, we characterize the relationship among $u_k$, $\delta_k$ and $\Delta_k$.
\begin{lemma}\label{lem:delta_k}
Suppose Assumptions \ref{ass:lip} and \ref{ass:var} hold. Let $\delta_k=\mathbb{E}\|\nabla_yg(x_k,y_k)\|^2$, $\beta=\frac{2}{13L_Q}$, and $S_2\geq 2(\frac{L}{\mu}+1)L\beta$. Then, we have
	\begin{align*}
	\Delta_k \leq \widetilde{\Delta}_{k-1,0} + \frac{2L_Q^{2}\beta^2}{S_2(1-a)}\mathbb{E}\|\widetilde{u}_{k-1,0}\|^2,
	\end{align*}
	where $\widetilde{\Delta}_{k-1,0}$ and $\Delta_k$ are defined in \Cref{lem:delta_ti}, and $a$ is defined in \Cref{lem:u_kt}, and
	\begin{align*}
	\mathbb{E}\|\widetilde{u}_{k,0}\|^2 \leq 3(\widetilde{\Delta}_{k,0}+\mathbb{E}\|\nabla_yg(x_{k+1}, y_k)-\nabla_yg(x_k, y_k)\|^2+\delta_k).
 	\end{align*}
 	\begin{proof}
 	Based on eq.(23) and eq.(24) in \cite{luo2020stochastic} first version, we obtain:
 	\begin{align*}
 	    \delta_k \leq \widetilde{\delta}_{k-1,0} + \frac{l^2\lambda^2}{S_2(1-\alpha)}\|\widetilde{u}_{k-1,0}\|^2
 	\end{align*}
 	and
 	\begin{align*}
 	    \mathbb{E}\|\widetilde{u}_{k,0}\|^2 \leq 3(\widetilde{\Delta}_{k,0}+\mathbb{E}\|\nabla_yf(x_{k+1},y_k)-\nabla_yf(x_k,y_k)\|^2+\delta_k).
 	\end{align*}
 	The proof is finished by replacing $l$ with $L_Q$ and replacing $f(x,y)$ with $g(x,y)$.
 	\end{proof}

\end{lemma}



Next, we characterize the recursive updates of $\delta_{k}$ and $\Delta_{k}$, respectively. 
\begin{lemma} \label{lem:smalldelta_k}
Suppose Assumptions \ref{ass:conv}, \ref{ass:lip} and \ref{ass:var} hold. Let $\beta=\frac{2}{13L_Q}$, and $S_2\geq 2(\frac{L}{\mu}+1)L\beta$. Then, we have
\begin{align*}
    \delta_{k+1}\leq \frac{4}{\mu \beta (m+1)}(\mathbb{E}\|\nabla_yg(x_{k+1},y_k)-\nabla_yg(x_k,y_k)\|^2+\delta_k)+\frac{L\beta}{2-L\beta}\mathbb{E}\|\widetilde{u}_{k,0}\|^2+\widetilde{\Delta}_{k,0}.
\end{align*}
where $\delta_k$ is defined in \Cref{lem:delta_k} and $\widetilde{\Delta}_{k,0}$ is defined in \Cref{lem:delta_ti}.
\begin{proof}
Following from Lemma 12 in \cite{luo2020stochastic}, we have
	\begin{align*}
	\mathbb{E}\|\nabla_yg(x_{k+1}, \widetilde{y}_{k,t+1})\|^2 \leq& \frac{2}{\mu \beta (m+1)}\|\nabla_yg(x_{k+1}, \widetilde{y}_{k,0})\|^2+\frac{L\beta}{2-L\beta}\mathbb{E}\|\widetilde{u}_{k,0}\|^2 \\
	&+\mathbb{E}\|\nabla_yg(x_{k+1}, \widetilde{y}_{k,0})-\widetilde{u}_{k,0}\|^2.
	\end{align*}
Since $\widetilde{y}_{k,0}=y_k, x_{k+1}=\widetilde{x}_{k,0}, \widetilde{y}_{k,m+1}=y_{k+1}$, $\delta_{k+1}=\mathbb{E}\|\nabla_yg(x_{k+1}, y_{k+1})\|^2$, we have
	\begin{align*}
		\delta_{k+1}  \leq& \frac{2}{\mu \beta (m+1)} \|\nabla_yg(x_{k+1}, y_k)\|^2+\frac{L\beta}{2-L\beta}\mathbb{E}\|\widetilde{u}_{k,0}\|^2+\mathbb{E}\|\nabla_yg(x_{k+1}, y_k)-\widetilde{u}_{k,0}\|^2 \\
    \leq& \frac{2}{\mu \beta (m+1)}\mathbb{E}\|\nabla_yg(x_{k+1}, y_k)-\nabla_yg(x_k, y_k)+\nabla_yg(x_k, y_k)\|^2+\frac{L\beta}{2-L\beta}\mathbb{E}\|\widetilde{u}_{k,0}\|^2 \\
	&+\mathbb{E}\|\nabla_yg(\widetilde{x}_{k,0},\widetilde{y}_{k,0})-\widetilde{u}_{k,0}\|^2 \\
	 \leq& \frac{4}{\mu \beta (m+1)}(\mathbb{E}\|\nabla_yg(x_{k+1},y_k)-\nabla_yg(x_k,y_k)\|^2+\delta_k)+\frac{L\beta}{2-L\beta}\mathbb{E}\|\widetilde{u}_{k,0}\|^2+\widetilde{\Delta}_{k,0}.
		\end{align*}
		Then, the proof is complete.
	\end{proof}
\end{lemma}

\begin{lemma}\label{lem:Delta_tele} 
Suppose Assumptions \ref{ass:lip} and \ref{ass:var} hold. Let $\beta=\frac{2}{13L_Q}$, and $S_2\geq 2(\frac{L}{\mu}+1)L\beta$. Then, we have
\begin{align*}
    \Delta_k \leq & \frac{\alpha^2 L_Q^{2}}{S_2}\left(2+\frac{12L_Q^{2}\beta^2}{S_2(1-a)}+\frac{6L^2\beta^2}{S_2(1-a)}\|v_{k-1}\|^2+\frac{6L_Q^{2}\beta^2}{S_2(1-a)}\delta_{k-1}\right) \\
    &+\left(1+\frac{6L_Q^{2}\beta^2}{S_2(1-a)} \right) \Delta_{k-1},
\end{align*}
where $\delta_{k-1}$ is defined in \Cref{lem:delta_k}, $\Delta_k$ is defined in \Cref{lem:delta_ti} and $L_Q$ is defined in \Cref{lem:LQ}.
\begin{proof}
Based on the bounds on $\Delta_k$ in \Cref{lem:delta_k}, we have
	\begin{align*}
	\Delta_k \leq& \widetilde{\Delta}_{k-1,0}+\frac{2L_Q^{2}\beta^2}{S_2(1-a)}\mathbb{E}\|\widetilde{u}_{k-1,0}\|^2 \\
	\overset{(i)}\leq& \widetilde{\Delta}_{k-1,0}+\frac{6L_Q^{2}\beta^2}{S_2(1-a)}(\widetilde{\Delta}_{k-1,0}+\|\nabla_yg(x_k, y_{k-1})-\nabla_yg(x_{k-1}, y_{k-1})\|^2+\delta_{k-1}) \\
	\overset{(ii)}\leq& \left (1+\frac{6L_Q^{2}\beta^2}{S_2(1-a)}\right)\widetilde{\Delta}_{k-1,0} + \frac{6L_Q^{2}\beta^2}{S_2(1-a)}(L^2\alpha^2\|v_{k-1}\|^2+\delta_{k-1}) \\
	\overset{(iii)}\leq& \left(1+\frac{6L_Q^{2}\beta^2}{S_2(1-a)}\right) \left(\Delta_{k-1}+\frac{2L_Q^{2}\alpha^2}{S_2}\|v_{k-1}\|^2\right)+\frac{6L_Q^{2}\beta^2}{S_2(1-a)}(L^2\alpha^2\|v_{k-1}\|^2+\delta_{k-1}) \\
	 \leq& \frac{\alpha^2 L_Q^{2}}{S_2}\left(2+\frac{12L_Q^{2}\beta^2}{S_2(1-a)}+\frac{6L^2\beta^2}{S_2(1-a)}\|v_{k-1}\|^2+\frac{6L_Q^{2}\beta^2}{S_2(1-a)}\delta_{k-1}\right)\\
	&+\left(1+\frac{6L_Q^{2}\beta^2}{S_2(1-a)}\right) \Delta_{k-1},
	\end{align*}
	where $(i)$ follows from \Cref{lem:delta_k}, $(ii)$ follows from \Cref{ass:lip}, and $(iii)$ follows from \Cref{lem:delta_ti}. Then, the proof is complete.
\end{proof}
\end{lemma}

\begin{lemma}\label{lem:delta_tele}
Suppose Assumptions \ref{ass:conv}, \ref{ass:lip} and \ref{ass:var} hold. Let $\beta=\frac{2}{13L_Q}$, and $S_2\geq 2(\frac{L}{\mu}+1)L\beta$. Then, we have
\begin{align*}
	\delta_k \leq& \bigg(\frac{4L^2\alpha^2}{\mu \beta (m+1)}+\frac{3L^3\beta \alpha^2}{2-L\beta}+\frac{6LL_Q^{2}\alpha^2\beta}{2-L\beta}
	+2L_Q^{2}\alpha^2\bigg)\mathbb{E}\|v_{k-1}\|^2\\
	&+ \frac{2+2L\beta}{2-L\beta}\Delta_{k-1}+\left(\frac{4}{\mu \beta(m+1)}+\frac{3L\beta}{2-L\beta}\right)\delta_{k-1},
\end{align*}
where $\delta_{k}$ is defined in \Cref{lem:delta_k}, $\Delta_{k-1}$ is defined in \Cref{lem:delta_ti} and $L_Q$ is defined in \Cref{lem:LQ}.
\begin{proof}
Based on the bounds of $\delta_k$ in \Cref{lem:smalldelta_k}, we have
    \begin{align*}
	\delta_k \leq& \frac{4}{\mu \beta(m+1)}(L^2\alpha^2\|v_{k-1}\|^2+\delta_{k-1})+\frac{L\beta}{2-L\beta}\mathbb{E}\|\widetilde{u}_{k-1,0}\|^2+\widetilde{\Delta}_{k-1,0} \\
	 \overset{(i)}\leq& \frac{4}{\mu \beta (m+1)}(L^2\alpha^2 \|v_{k-1}\|^2+\delta_{k-1}) + \frac{3L\beta}{2-L\beta}(\widetilde{\Delta}_{k-1,0}+L^2\alpha^2\|v_{k-1}\|^2+\delta_{k-1})+\widetilde{\Delta}_{k-1,0} \\
	 =& \left(\frac{4}{\mu \beta(m+1)}+\frac{3L\beta}{2-L\beta}\right)\delta_{k-1} +\left(\frac{4L^2\alpha^2}{\mu \beta (m+1)}+\frac{3L^2\beta \alpha^2}{2-L\beta}\right)\|v_{k-1}\|^2 \\
	 & + \left (1+\frac{3L\beta}{2-L\beta}\right)\widetilde{\Delta}_{k-1,0}\\
	 \overset{(ii)}\leq& \left(\frac{4}{\mu \beta (m+1)}+\frac{3L\beta}{2-L\beta}\right)\delta_{k-1}+\left(1+\frac{3L\beta}{2-L\beta}\right)\left(\Delta_{k-1}+\frac{2L_Q^{2}\alpha^2}{S_2}\|v_{k-1}\|^2\right) \\
	&+ \left(\frac{4L^2\alpha^2}{\mu \beta (m+1)}+\frac{3L^3\beta \alpha^2}{2-L\beta}\right)\|v_{k-1}\|^2
	 \\
	 \leq & \bigg(\frac{4L^2\alpha^2}{\mu \beta (m+1)}+\frac{3L^3\beta \alpha^2}{2-L\beta}+\frac{6LL_Q^{2}\alpha^2\beta}{2-L\beta}
	+2L_Q^{2}\alpha^2\bigg)\mathbb{E}\|v_{k-1}\|^2\\
	&+ \frac{2+2L\beta}{2-L\beta}\Delta_{k-1}+\left(\frac{4}{\mu \beta(m+1)}+\frac{3L\beta}{2-L\beta}\right)\delta_{k-1},
	 \end{align*}
	 where $(i)$ follows from \Cref{lem:delta_k}, and $(ii)$ follows from \Cref{lem:delta_ti}. Then, the proof is complete.
\end{proof}
\end{lemma}

\begin{lemma}[{\bf Restatement of \Cref{prop:vrbo_est}}]\label{lem:deltadef} Suppose Assumptions \ref{ass:conv}, \ref{ass:lip} and \ref{ass:var} hold. Let $\eta < \frac{1}{L}$. Then, we have
\begin{align}\label{eq:sigmaprime}
    \mathbb{E}\|\widehat{\nabla}\Phi(x_k,y_k;\xi)-\overline{\nabla}\Phi(x_k)\|^2 \leq \sigma'^2,
\end{align}
where $\sigma'^2=2M^2+28(Q+1)^2L^2M^2\eta^2$, $\widehat{\nabla}\Phi(x_k,y_k;\xi)$ is defined in \cref{eq:xgest_vrbio} with single sample $\xi$, and $\overline{\nabla}\Phi(x_k)$ is defined in \cref{eq:hypergradientest}. Furthermore, let $\beta=\frac{2}{13L_Q}$, $q=(1-a)S_2$, and $m=\frac{16}{\mu\beta}-1$. Then, we have
\begin{align}\label{eq:Delta_tele}
    \sum_{k=0}^{K-1}\Delta_k \leq \frac{4\sigma'^2K}{S_1} + 22\alpha^2L_Q^{2}\sum_{k=0}^{K-2}\mathbb{E}\|v_k\|^2 + \frac{4}{3} \delta_0,
\end{align}
where $\delta_0$ is defined in \Cref{lem:delta_k}.

\begin{proof}
We first prove \cref{eq:sigmaprime}. Based on the forms of $\widehat{\nabla}\Phi(x_k,y_k;\xi)$ and $\overline{\nabla}\Phi(x_k)$, we have
\begin{align*}
     \mathbb{E}\|\widehat{\nabla}&\Phi(x_k,y_k;\xi)-\overline{\nabla}\Phi(x_k)\|^2 \\
     \leq & \mathbb{E}\|\nabla_xF(x_k,y_k;\xi)-\nabla_xf(x_k,y_k) - (\nabla_x\nabla_yG(x_k,y_k;\zeta)V_{Q\xi} -\nabla_x\nabla_yg(x_k,y_k)\mathbb{E}[V_{Q\xi}])\|^2 \\
     \overset{(i)}\leq & 2M^2+4\mathbb{E}\|\nabla_x\nabla_yG(x_k,y_k;\zeta)V_{Q\xi}-(\nabla_x\nabla_yG(x_k,y_k;\zeta)\mathbb{E}[V_{Q\xi}])\|^2\\
     &+4\mathbb{E}\|(\nabla_x\nabla_yG(x_k,y_k;\zeta)\mathbb{E}[V_{Q\xi}])-\nabla_x\nabla_yg(x_k,y_k)\mathbb{E}[V_{Q\xi}])\|^2 \\
     \overset{(ii)}\leq & 2M^2 + 4L^2\mathbb{E}\|V_{Q\xi}-\mathbb{E}[V_{Q\xi}]\|^2+4\|\mathbb{E}[V_{Q\xi}]\|^2\mathbb{E}\|\nabla_x\nabla_yG(x_k,y_k;\zeta)-\nabla_x\nabla_yg(x_k,y_k)\|^2 \\
     \overset{(iii)}\leq  &\textstyle 2M^2+8L^2\mathbb{E}\|\eta \sum_{q=-1}^{Q-1} \prod_{j=Q-q}^{Q} (I-\eta \nabla_y^2G(x_k, y_k; \zeta_j))\nabla_yF(x_k,y_k;\xi)\\
     &\textstyle-\eta \sum_{q=-1}^{Q-1} \prod_{j=Q-q}^{Q} (I-\eta \nabla_y^2G(x_k, y_k; \zeta_j))\nabla_yf(x_k,y_k)\|^2\\
     &\textstyle +8L^2\mathbb{E}\|\eta \sum_{q=-1}^{Q-1} \prod_{j=Q-q}^{Q} (I-\eta \nabla_y^2G(x_k, y_k; \zeta_j))\nabla_yf(x_k,y_k)\\
     &\textstyle-\eta \sum_{q=-1}^{Q-1} \prod_{j=Q-q}^{Q} (I-\eta \nabla_y^2g(x_k, y_k))\nabla_yf(x_k,y_k) \|^2+4\eta^2M^2L^2(Q+1)^2 \\
     \overset{(iv)}\leq & 2M^2+8L^2\eta^2(Q+1)^2M^2+16L^2\eta^2M^2(Q+1)^2+4\eta^2M^2L^2(Q+1)^2\\
     = &2M^2+28(Q+1)^2L^2M^2\eta^2=\sigma'^2,
\end{align*}
where $(i)$ and $(ii)$ follows from \Cref{ass:lip}, $(iii)$ follows because $\|\mathbb{E}[V_{Q\xi}]\|^2=\|\mathbb{E}[V_{Qk}]\|^2\leq \eta^2 M^2(Q+1)^2$ in \cref{eq:evqk}, and $(iv)$ follows because $\|(I-\eta \nabla_y^2G(x_k,y_k;\zeta)) \|\leq 1$. 

Then, we present the proof of \cref{eq:Delta_tele}. Based on the bound on $\Delta_k$ in \Cref{lem:Delta_tele}, we have
	\begin{align*}
	\Delta_k \leq &\left(1+\frac{6L_Q^{2}\beta^2}{S_2(1-a)}\right)\Delta_{k-1} + \frac{\alpha^2L_Q^{2}}{S_2}\left(2+\frac{12L_Q^{2}\beta^2}{1-a}+\frac{6L^2\beta^2}{1-a}\right)\|v_{k-1}\|^2\\
	&+\frac{6L_Q^{2}\beta^2}{S_2(1-a)}\delta_{k-1} \\
	 \leq& \frac{\alpha^2L_Q^{2}}{S_2}\left(2+\frac{12\beta^2(L^2+L_Q^{2})}{1-a}\right)\sum_{p=k'}^{K-1}\left(1+\frac{6L_Q^{2}}{S_2(1-a)}\right)^{p-k'}\mathbb{E}\|v_{K-1+k'-p}\|^2 \\
	&+\frac{6L_Q^{2}\beta^2}{S_2(1-a)}\sum_{p=k'}^{K-1}\left(1+\frac{6L_Q^{2}\beta^2}{S_2(1-a)}\right)^{p-k'}\delta_{K-1+k'-p}+\left(1+\frac{6L_Q^{2}\beta^2}{S_2(1-a)}\right)^{k-k'}\Delta_{k'} \\
	\overset{(i)}\leq& \frac{3}{2} \Delta_{k'} + \frac{3\alpha^2L_Q^{2}}{S_2}\left(1+\frac{6\beta(L^2+L_Q^{2})}{1-a}\right)\sum_{p=k'}^{K-1}\mathbb{E}\|v_{K-1+k'-p}\|^2\\
	&+\frac{9L_Q^{2}\beta^2}{S_2(1-a)}\sum_{p=k'}^{K-1}\delta_{K-1+k'-p},
	\end{align*}
	where $(i)$ follows from the following bound:
	\begin{align*}
	\left(1+\frac{L_Q^{2}\beta^2}{S_1(1-a)}\right)^{p-k'} &\leq \left(1+\frac{6L_Q^{2}\beta^2}{S_2(1-a)}\right)^q \leq 1+ \frac{\frac{6L_Q^{2}\beta^2q}{S_2(1-a)}}{1-\frac{6L_Q^{2}\beta^2(q-1)}{S_2(1-a)}} \\
	&\leq 1+\frac{6L_Q^{2}\beta^2}{1-\frac{6L_Q^{2}\beta^2q}{S_2(1-a)}} < \frac{3}{2}
	\end{align*}
	where $\beta = \frac{2}{13L_Q}$, and $q=(1-a)S_2$.
	Then telescoping $\Delta_k$ over $k$ from $(n_k-1)q$ to $K-1$, we have
	\begin{align*}
	\sum_{k=(n_k-1)q}^{K-1} \Delta_k \leq& \frac{3\alpha L_Q^{2}}{S_2}\left(1+\frac{6\beta^2(L^2+L_Q^{2})}{1-a}\right) \sum_{k=(n_k-1)q}^{K-1}\sum_{p=k'}^{k} \mathbb{E}\|v_{K-1+k'-p}\|^2 \\
	&+\frac{9L_Q^{2}\beta^2}{S_2(1-a)}\sum_{k=(n_k-1)q}^{K-1}\sum_{p=k'}^{k-1}\delta_{K-1+k'-p}+\frac{3}{2}(K-(n_k-1)q)\Delta_{(n_k-1)q}.
	\end{align*}
	Since 
	\begin{align*}
	    \sum_{k=(n_k-1)q}^{K-1}\sum_{p=k'}^{k} \mathbb{E}\|v_{K-1+k'-p}\|^2 \leq q\sum_{k=(n_k-1)q}^{K-2}\mathbb{E}\|v_k\|^2,
	\end{align*}
	and 
	\begin{align*}
	    \sum_{k=(n_k-1)q}^{K-1}\sum_{p=k'}^{k-1}\delta_{K-1+k'-p} \leq q \sum_{k=(n_k-1)q}^{K-2}\delta_k,
	\end{align*}
we have
	\begin{align*}
	\sum_{k=(n_k-1)q}^{K-1}\Delta_k \leq& \frac{3}{2}(K-(n_k-1)q)\Delta_{(n_k-1)q} + \frac{3\alpha^2 L_Q^{2}q}{S_2}\left(1+\frac{6\beta^2(L^2+L_Q^{2})}{1-a}\right)\sum_{k=(n_k-1)q}^{K-2}\mathbb{E}\|v_k\|^2\\ 
	&+ \frac{9L_Q^{2}\beta^2q}{S_2(1-a)} \sum_{k=(n_k-1)q}^{K-2}\delta_k.
	\end{align*}
	Futhermore, we assume that $\sigma^\prime \geq\sigma$ and derive the following bound on the initial update in each epoch:
	\begin{align*}
	\sum_{k=(n_K-n_k)q}^{(n_K-n_k+1)q-1}\Delta_k \leq& \frac{3 \sigma'^2q}{2S_1} + \frac{3\alpha^2L_Q^{2}q}{S_2}\left(1+\frac{6\beta^2(L^2+L_Q^{2})}{1-a}\right) \sum_{k=(n_K-n_k)q}^{(n_K-n_k+1)q-1}\mathbb{E}\|v_k\|^2 \\
	&+ \frac{9L_Q^{2}\beta^2q}{S_2(1-a)}\sum_{k=(n_K-n_k)q}^{(n_K-n_k+1)q-1}\delta_k.
	\end{align*}
	Based on the above inequality, we telescope $\Delta_k$ over $k$ from $0$ to $K-1$, and obtain
	\begin{align*}
	\sum_{k=0}^{K-1} \Delta_k &\leq \frac{3\sigma'^2K}{2S_1}+\frac{3\alpha^2L_Q^{2}q}{S_2}\left(1+\frac{6\beta^2(L^2+L_Q^{2})}{1-a}\right)\sum_{k=0}^{K-2}\mathbb{E}\|v_k\|^2+\frac{9L_Q^{2}\beta^2q}{S_2(1-a)}\sum_{k=0}^{K-2}\delta_k \\
	& \overset{(i)}\leq \frac{3\sigma'^2K}{2S_1}+6\alpha^2L_Q^{2}\sum_{k=0}^{K-2}\mathbb{E}\|v_k\|^2 + \frac{1}{4}\sum_{k=0}^{K-2}\delta_k,
	\end{align*}
	where $(i)$ follows because $\beta=\frac{2}{13L_Q}$, and $q=(1-a)S_2$. We further derive the following bound on $\delta_k$:
	
	\begin{align*}
	\delta_k \leq& \left(\frac{4L^2\alpha^2}{\mu \beta(m+1)}+\frac{3L^3\beta \alpha^2}{2-L\beta}+\frac{6LL_Q^{2}\alpha^2\beta}{2-L\beta}+2L_Q^{2}\alpha^2\right)\mathbb{E}\|v_{k-1}\|^2\\
	&+\frac{2+2L\beta}{2-L\beta}\Delta_{k-1}+\left(\frac{4}{\mu \beta(m+1)}+\frac{3L\beta}{2-L\beta}\right)\delta_{k-1}\\
	& \overset{(ii)}\leq \frac{1}{2} \delta_{k-1} + \frac{13}{4}L_Q^{2}\alpha^2\mathbb{E}\|v_{k-1}\|^2+\frac{5}{4}\Delta_{k-1},
	\end{align*}
	where $(ii)$ follows because $\beta=\frac{2}{13L_Q},$ $q=(1-a)S_2$, and $m=\frac{16}{\mu \beta}-1$. Then, we telescope $\delta_k$ and $\Delta_k$ over $k=0$ to $K-1$, and have
	
	\begin{align}\label{eq:delta_temp}
	\sum_{k=0}^{K-1}\delta_k \leq 2\delta_0 + \frac{13}{2}L_Q^{2}\alpha^2 \sum_{k=0}^{K-2}\mathbb{E}\|v_k\|^2+\frac{5}{2}\sum_{k=0}^{K-2}\Delta_k,
	\end{align}
	and
	\begin{align*}
	\sum_{k=0}^{K-1}\Delta_k &\leq \frac{3\sigma'^2K}{2S_1} + 6\alpha^2L_Q^{2}\sum_{k=0}^{K-2}\mathbb{E}\|v_k\|^2+\frac{1}{2}\delta_0 + \frac{13}{8}L_Q^{2}\alpha^2 \sum_{k=0}^{K-3}\mathbb{E}\|v_k\|^2+\frac{5}{8} \sum_{k=0}^{K-2}\Delta_k \\
	& \leq \frac{3\sigma'^2K}{2S_1} + 8\alpha^2L_Q^{2}\sum_{k=0}^{K-2}\mathbb{E}\|v_k\|^2+\frac{1}{2}\delta_0 + \frac{5}{8} \sum_{k=0}^{K-2}\Delta_k.
	\end{align*}
	Finally, we rearrange the terms in the above bound and obtain
	\begin{align*}
	\sum_{k=0}^{K-1}\Delta_k \leq \frac{4\sigma'^2K}{S_1} + 22\alpha^2L_Q^{2}\sum_{k=0}^{K-2}\mathbb{E}\|v_k\|^2 + \frac{4}{3} \delta_0,
	\end{align*}
 Then, the proof is complete.
\end{proof}
\end{lemma}

\begin{lemma}\label{lem:doubledelta}
Suppose Assumptions \ref{ass:conv}, \ref{ass:lip} and \ref{ass:var} hold. Let $\beta=\frac{2}{13L_Q}$, $q=(1-a)S_2$, and $m=\frac{16}{\mu\beta}-1$. Then, we have
\begin{align*}
    	\sum_{k=0}^{K-1}\delta_k \leq \frac{10\sigma'^2K}{S_1}+6\delta_0 + 62\alpha^2L_Q^{2}\sum_{k=0}^{K-2}\mathbb{E}\|v_k\|^2,
\end{align*}
where $L_Q$ is defined in \Cref{lem:LQ} and $\sigma'$ is defined in \Cref{lem:deltadef}.

\begin{proof}
Based on the inequalities in \cref{eq:delta_temp} and \cref{eq:Delta_tele}, we have
	\begin{align*}
	\sum_{k=0}^{K-1}\delta_k &\leq \frac{10\sigma'^2K}{S_1} + 55\alpha^2L_Q^{2}\sum_{k=0}^{K-2}\mathbb{E}\|v_k\|^2+\frac{10}{3}\delta_0 + 2\delta_0+\frac{13}{2}L_Q^{2}\alpha^2\sum_{k=0}^{K-2}\mathbb{E}\|v_k\|^2 \\
	& \leq \frac{10\sigma'^2K}{S_1}+6\delta_0 + 62\alpha^2L_Q^{2}\sum_{k=0}^{K-2}\mathbb{E}\|v_k\|^2.
	\end{align*}
	Then, the proof is complete.
	\end{proof}
\end{lemma}

\begin{lemma}[{\bf Restatement of \Cref{prop:vrbo_phi}}] \label{lem:vk}
Suppose Assumptions \ref{ass:conv},\ref{ass:lip} and \ref{ass:var} hold. 
Then, we have
		\begin{align}\label{eq:phik}
		\mathbb{E}[\Phi(x_{k+1})] \leq& \mathbb{E}[\Phi(x_k)] + \frac{\alpha{L'}^2}{\mu^2}\mathbb{E}\|\nabla_yg(x_k,y_k)\|^2 + \alpha \mathbb{E} \|\widetilde{\nabla} \Phi(x_k)-v_k\|^2 \nonumber \\ 
		 &- \left(\frac{\alpha}{2}-\frac{\alpha^2}{2}L_\Phi\right)\mathbb{E}\|v_k\|^2.
		\end{align}
	where $L'=L+\frac{L^2}{\mu}+\frac{M\tau}{\mu}+\frac{LM\rho}{\mu^2}$, and $\widetilde{\nabla}\Phi(x_k)$ is defined in \cref{eq:tildephi}.
	\begin{proof}
	Based on the smoothness of the function $\Phi(x)$, we have
	\begin{align*}
		\Phi(x_{k+1})
		\overset{(i)}\leq& \Phi(x_k) + \langle\nabla \Phi(x_k), x_{k+1}-x_k\rangle + \frac{L_\Phi}{2} \|x_{k+1}-x_k\|^2 \\
		 \leq&\Phi(x_k)-\alpha \langle \nabla\Phi(x_k), v_k\rangle+\frac{\alpha^2}{2} L_\Phi \|v_k\|^2 \\
		 \leq&\Phi(x_k) - \alpha\langle \nabla\Phi(x_k)-v_k,v_k\rangle-\alpha \|v_k\|^2+\frac{\alpha^2}{2}L_\Phi \|v_k\|^2\\
		\leq& \Phi(x_k) + \frac{\alpha}{2} \|\nabla \Phi(x_k)-v_k\|^2-\left(\frac{\alpha}{2}-\frac{\alpha^2}{2}L_\Phi\right)\|v_k\|^2 \\
		 \leq& \Phi(x_k) + \alpha \|\nabla\Phi(x_k)-\widetilde{\nabla}\Phi(x_k)\|^2+\alpha\|\widetilde{\nabla}\Phi(x_k)-v_k\|^2-\left(\frac{\alpha}{2}-\frac{\alpha^2}{2}L_\Phi\right)\|v_k\|^2 \\
		\overset{(ii)}\leq& \Phi(x_k) + \frac{\alpha {L'}^2}{\mu^2} \|\nabla_yg(x_k,y_k)-\nabla_yg(x_k, y^{\ast}(x_k))\|^2+\alpha \|\widetilde{\nabla} \Phi(x_k)-v_k\|^2\\
		& - \left(\frac{\alpha}{2}-\frac{\alpha^2}{2}L_\Phi\right) \|v_k\|^2 \\
		\overset{(iii)}\leq& \Phi(x_k)+\frac{\alpha {L'}^2}{\mu^2} \|\nabla_yg(x_k,y_k)\|^2+\alpha\|\widetilde{\nabla}\Phi(x_k)-v_k\|^2 -\left(\frac{\alpha}{2}-\frac{\alpha^2}{2}L_\Phi\right)\|v_k\|^2,
		\end{align*}
		 where $(i)$ follows from Assumptions \ref{ass:lip} and \ref{ass:var}, $(ii)$ follows from Lemma 7 in~\cite{ji2021bilevel} and the $\mu$-strong convexity of $g(x,y)$ w.r.t.\ $y$, and $(iii)$ follows because $\nabla_yg(x_k,y^{\ast}(x_k))=0$.
		 
		Taking the expectation on both sides, we obtain
		\begin{align*}
		\mathbb{E}[\Phi(x_{k+1})] \leq& \mathbb{E}[\Phi(x_k)] + \frac{\alpha{L'}^2}{\mu^2}\mathbb{E}\|\nabla_yg(x_k,y_k)\|^2 + \alpha \mathbb{E} \|\widetilde{\nabla} \Phi(x_k)-v_k\|^2 \nonumber \\ 
		 &- \left(\frac{\alpha}{2}-\frac{\alpha^2}{2}L_\Phi\right)\mathbb{E}\|v_k\|^2.
		\end{align*}
		Then, the proof is complete.
	\end{proof}
\end{lemma}

\begin{lemma} \label{lem:vktele}
Suppose Assumptions \ref{ass:conv}, \ref{ass:lip} and \ref{ass:var} hold, $\beta=\frac{2}{13L_Q}$, $q=(1-a)S_2$, $m=\frac{16}{\mu\beta}-1$, and $\alpha=\frac{1}{20L_m^3}$ where $L_m=\max\{L_Q,L_\Phi\}$. Then, we have
\begin{align*}
\sum_{k=0}^{K-1} \mathbb{E}\|v_k\|^2 \leq & L^{\prime\prime}(\Phi(x_0)-\Phi^{\ast})+\frac{9L'^2\alpha \delta_0 L^{\prime\prime}}{\mu^2} + \frac{18L'^2\sigma'^2K\alpha L^{\prime\prime}}{\mu^2 S_1} \\
&+ 2\alpha L^{\prime\prime}\sum_{k=0}^{K-1} \|\widetilde{\nabla}\Phi(x_k)-\overline{\nabla}\Phi(x_k)\|^2,
\end{align*}
where $\frac{1}{L^{\prime\prime}}=\frac{\alpha}{2}-\frac{L_\Phi \alpha^2}{2}-\frac{62 \alpha^3{L'}^2L_Q^2}{\mu^2}-44\alpha^3L_Q^2,$ $\sigma'$ is defined in \Cref{lem:deltadef}, $\widetilde{\nabla}\Phi(x_k)$ is defined in \cref{eq:tildephi}, $\overline{\nabla}\Phi(x_k)$ is defined in \cref{eq:hypergradientest}, and $L'$ is defined in \Cref{lem:deltadef}.

\begin{proof}
	Telescoping \cref{eq:phik} over $k$ from $0$ to $K-1$, we have
		\begin{align*}
		\bigg(\frac{\alpha}{2}&-\frac{L_\Phi\alpha^2}{2}\bigg) \sum_{k=0}^{K-1}\mathbb{E}\|v_k\|^2 \\
		\leq & \Phi(x_0)-\mathbb{E}[\Phi(x_K)]+\frac{\alpha {L'}^2}{\mu^2} \sum_{k=0}^{K-1} \delta_k + 2\alpha \sum_{k=0}^{K-1} \Delta_k + 2\alpha \sum_{k=0}^{K-1} \|\widetilde{\nabla}\Phi(x_k)-\overline{\nabla}\Phi(x_k)\|^2\\
\overset{(i)}\leq& \Phi(x_0)-\mathbb{E}[\Phi(x_K)]+\frac{\alpha L'^2}{\mu^2}\left(\frac{10\sigma'^2K}{S_1}+6\delta_0+62\alpha^2L_Q^{2}\sum_{k=0}^{K-2}\mathbb{E}\|v_k\|^2\right)\\
&+2\alpha\left(\frac{4\sigma'^2K}{S_1}+22\alpha^2L_Q^{2}\sum_{k=0}^{K-2}\mathbb{E}\|v_k\|^2+\frac{4}{3}\delta_0\right)+2\alpha \sum_{k=0}^{K-1}\|\widetilde{\nabla}\Phi(x_k)-\overline{\nabla}\Phi(x_k)\|^2 \\
\leq& \Phi(x_0) - \mathbb{E}[\Phi(x_K)] + \left(\frac{10L'^2}{\mu^2}+8\right)\frac{\sigma'^2K\alpha}{S_1}+\left(\frac{6L'^2}{\mu^2}+\frac{8}{3}\right)\alpha\delta_0 \\
	&+ \left(\frac{62\alpha^3L'^2L_Q^{2}}{\mu^2}+44\alpha^3L_Q^{2}\right)\sum_{k=0}^{K-2}\mathbb{E}\|v_k\|^2+2\alpha\sum_{k=0}^{K-1}\|\widetilde{\nabla}\Phi(x_k)-\overline{\nabla}\Phi(x_k)\|^2, 
\end{align*}
	where $(i)$ follow from Lemmas \ref{lem:deltadef} and \ref{lem:doubledelta}.

We let $\frac{1}{L^{\prime\prime}}=(\frac{\alpha}{2}-\frac{L_\Phi \alpha^2}{2}-\frac{62 \alpha^3{L'}^2L_Q^2}{\mu^2}-44\alpha^3L_Q^2)$, which is guaranteed to be positive due to the parameter settings given in the lemma, reorganize the terms in the above inequality, and obtain
\begin{align*}
\frac{1}{L^{\prime\prime}} \sum_{k=0}^{K-1} \mathbb{E}\|v_k\|^2 \leq& \Phi(x_0) -\mathbb{E}[\Phi(x_K)]+\left(\frac{10L'^2}{\mu^2}+8\right)\frac{\sigma'^2K\alpha}{S_1}\\
& + \left(\frac{6L'^2}{\mu^2}+\frac{8}{3}\right)\alpha\delta_0 + 2\alpha\sum_{k=0}^{K-1}\|\widetilde{\nabla}\Phi(x_K)-\overline{\nabla}\Phi(x_k)\|^2.
\end{align*}
Then, we have the bound on $\sum_{k=0}^{K-1}\mathbb{E}\|v_k\|^2$ as
\begin{align*}
\sum_{k=0}^{K-1} \mathbb{E}\|v_k\|^2 \leq & L^{\prime\prime}(\Phi(x_0)-\Phi^{\ast})+\frac{9L'^2\alpha \delta_0 L^{\prime\prime}}{\mu^2} + \frac{18L'^2\sigma'^2K\alpha L^{\prime\prime}}{\mu^2 S_1} \\
&+ 2\alpha L^{\prime\prime}\sum_{k=0}^{K-1} \|\widetilde{\nabla}\Phi(x_k)-\overline{\nabla}\Phi(x_k)\|^2.
\end{align*}
Then, the proof is complete.
\end{proof}
\end{lemma}

\subsection{Main Proof of \Cref{thm:vrbio}}
\begin{theorem}({\bf Formal Statement of \Cref{thm:vrbio}})\label{thm:vrbores}
Apply VRBO to solve the problem in \cref{eq:main}. Suppose Assumptions \ref{ass:conv}, \ref{ass:lip}, \ref{ass:var} hold. Let $\alpha=\frac{1}{20L_m^3}, \beta=\frac{2}{13L_Q}, S_2\geq 2(\frac{L}{\mu}+1)L\beta, m=\frac{16}{\mu \beta}-1, q=\frac{\mu L \beta S_2}{\mu+L},$ and $\eta < \frac{1}{L}$. Then, we have
\begin{align*}
    \sum_{k=0}^{K-1}\mathbb{E}\|\nabla\Phi(x_k)\|^2 \leq& \frac{56L'^2}{\mu^2}\frac{\sigma'^2K}{S_1} + \frac{30L'^2\delta_0}{\mu^2} + 340\alpha^2L_Q^{2}\frac{L'^2}{\mu^2}\bigg(L^{\prime\prime}(\Phi(x_0)-\Phi^{\ast})+\frac{9L'^2\alpha\delta_0 L^{\prime\prime}}{\mu^2}\\ 
	&+ \frac{18L'^2\sigma'^2K\alpha L^{\prime\prime}}{\mu^2 S_1} +2\alpha L^{\prime\prime}C_Q^2K\bigg) + 4KC_Q^2,
\end{align*}
where $\frac{1}{L^{\prime\prime}}=\frac{\alpha}{2}-\frac{L_\Phi \alpha^2}{2}-\frac{62 \alpha^3{L'}^2L_Q^2}{\mu^2}-44\alpha^3L_Q^2,$ $\sigma'$ is defined in \Cref{lem:deltadef}, $L_m$ is defined in \Cref{lem:vktele}, $\widetilde{\nabla}\Phi(x_k)$ is defined in \cref{eq:tildephi}, $\overline{\nabla}\Phi(x_k)$ is defined in \cref{eq:hypergradientest}, and $L'$ is defined in \Cref{lem:vk}.
\end{theorem}

\begin{proof}
Based on the form of $\nabla\Phi(x_k)$ in \cref{eq:hypergradient}, we have
\begin{align*}
	\sum_{k=0}^{K-1}&\mathbb{E}\|\nabla\Phi(x_k)\|^2 \\
	=& \sum_{k=0}^{K-1}\mathbb{E}[\|\nabla \Phi(x_k)-\widetilde{\nabla}\Phi(x_k)+\widetilde{\nabla}\Phi(x_k)-\overline{\nabla}\Phi(x_k)+\overline{\nabla}\Phi(x_k)-v_k+v_k\|^2] \\
	 \leq& 4\sum_{k=0}^{K-1}(\mathbb{E}\|\nabla\Phi(x_k)-\widetilde{\nabla}\Phi(x_k)\|^2+\mathbb{E}\|\widetilde{\nabla}\Phi(x_k)-\overline{\nabla}\Phi(x_k)\|^2 \\
	& + \mathbb{E}\|\overline{\nabla}\Phi(x_k)-v_k\|^2+\mathbb{E}\|v_k\|^2) \\
	 \leq& 4\sum_{k=0}^{K-1}(L'^2\|y_k-y^{\ast}(x_k)\|^2+C_Q^2+\Delta_k^2+\mathbb{E}\|v_k\|^2) \\
	 \overset{(i)}\leq &4 \sum_{k=0}^{K-1}\left(\frac{L'^2}{\mu^2}\|\nabla_yg(x_k,y_k)\|^2+C_Q^2+\Delta_k^2+\mathbb{E}\|v_k\|^2\right) \\
	 \leq& 4\sum_{k=0}^{K-1}\left(\frac{L'^2\delta_k}{\mu^2}+C_Q^2+\Delta_k^2+\mathbb{E}\|v_k\|^2\right)\\
	 \overset{(ii)}\leq& \frac{4L'^2}{\mu^2}\left(\frac{10\sigma'^2K}{S_1}+6\delta_0+62\alpha^2L_Q^{2}\sum_{k=0}^{K-2}\mathbb{E}\|v_k\|^2\right)+4KC_Q^2+4\sum_{k=0}^{K-1}\mathbb{E}\|v_k\|^2 \\
	&+4\left(\frac{4\sigma'^2K}{S_1}+22\alpha^2L_Q^{2}\sum_{k=0}^{K-2}\mathbb{E}\|v_k\|^2+\frac{4}{3}\delta_0\right) \\
	\leq& \left(\frac{40L'^2}{\mu^2}+16\right)\frac{\sigma'^2K}{S_1}+\left(\frac{24L'^2}{\mu^2}+\frac{16}{3}\right)\delta_0 + \left(247\alpha^2L_Q^{2}\frac{L'^2}{\mu}+88\alpha^2L_Q^{2}+4\right)\sum_{k=0}^{K-1}\mathbb{E}\|v_k\|^2 \\
	& + 4KC_Q^2\\
	 \overset{(iii)}\leq& \frac{56L'^2}{\mu^2}\frac{\sigma'^2K}{S_1} + \frac{30L'^2\delta_0}{\mu^2} + 340\alpha^2L_Q^{2}\frac{L'^2}{\mu^2}\bigg (L^{\prime\prime}(\Phi(x_0)-\Phi^{\ast})+\frac{9L'^2\alpha\delta_0 L^{\prime\prime}}{\mu^2}\\ 
	&+ \frac{18L'^2\sigma'^2K\alpha L^{\prime\prime}}{\mu^2 S_1} +2\alpha L^{\prime\prime}C_Q^2K \bigg) + 4KC_Q^2,
	\end{align*}
	where $(i)$ follows from \Cref{ass:conv}, $(ii)$ follows from \Cref{lem:doubledelta}, and $(iii)$ follows from \Cref{lem:vktele}. Taking the expectation on both sides, we have
	\begin{align*}
	   \frac{1}{K}\sum_{k=0}^{K-1}\mathbb{E}\|\nabla\Phi(x_k)\|^2 \leq &\frac{56L'^2}{\mu^2}\frac{\sigma'^2}{S_1} + \frac{30L'^2\delta_0}{\mu^2K} + 340\alpha^2L_Q^{2}\frac{L'^2}{\mu^2K}\bigg (L^{\prime\prime}(\Phi(x_0)-\Phi^{\ast})+\frac{9L'^2\alpha\delta_0 L^{\prime\prime}}{\mu^2}\\
	   &+\frac{18L'^2\sigma'^2K\alpha L^{\prime\prime}}{\mu^2 S_1} +2\alpha L^{\prime\prime}C_Q^2K \bigg) + 4C_Q^2.
	\end{align*}
	Since $C_Q=\mathcal{O}(1-\eta\mu)^Q, L_Q=\mathcal{O}(Q^2), \beta=\mathcal{O}(Q^{-2}), \sigma'^2=\mathcal{O}(Q^2)$, we obtain the following bound:
	\begin{align*}\label{eq:vrbores}
	      \frac{1}{K}\sum_{k=0}^{K-1}\mathbb{E}\|\nabla\Phi(x_k)\|^2 
    \leq \mathcal{O}\left(\frac{Q^4}{K}+\frac{Q^6}{S_1}+Q^4(1-\eta\mu)^{2Q}\right).
	\end{align*}
	Then, the proof is complete.
\end{proof}

\subsection{Proof of \Cref{coro:vrbio}}
\begin{coro}[{\bf Restatement of \Cref{coro:vrbio}}]
\label{coro:vrbio_restate}
    Under the same conditions of \Cref{thm:vrbio}, choose $S_1 =\mathcal{O}( \epsilon^{-1}), S_2=\mathcal{O}(\epsilon^{-0.5}), Q=\mathcal{O}(\log (\frac{1}{\epsilon^{0.5}})), K = \mathcal{O}(\epsilon^{-1})$. Then, VRBO finds an $\epsilon$-stationary point with the gradient complexity of $\mathcal{\widetilde O}(\epsilon^{-1.5})$ and Hessian-vector complexity of $\mathcal{\widetilde O}(\epsilon^{-1.5})$. 
\end{coro}

\begin{proof}
 Based on the setting in \Cref{coro:vrbio_restate}, we have $\mathcal{O}(\frac{Q^4}{K}+\frac{Q^6}{S_1}+Q^4(1-\eta\mu)^{2Q})=\mathcal{O}(\epsilon)$, which guarantees the target $\epsilon$-accuracy. Note that the period $q=(1-a)S_2=\mathcal{O}(\epsilon^{-0.5})$. Thus, the gradient and Jacobian complexities are given by $\mathcal{O}(KS_1/q+KS_2m)=\mathcal{\widetilde O}(\epsilon^{-1.5}+\epsilon^{-1.5})=\mathcal{\widetilde O}(\epsilon^{-1.5})$, and that Hessian-vector complexity is given by $\mathcal{O}(KQS_1/q+KS_2mQ)=\mathcal{\widetilde O}(\epsilon^{-1.5}+\epsilon^{-1.5})=\mathcal{ \widetilde O}(\epsilon^{-1.5})$. Then the proof is complete.
\end{proof}


\end{document}